\documentclass{article} % For LaTeX2e
\usepackage{iclr2027_conference,times}
\iclrfinalcopy
\usepackage{amsmath,amsfonts,bm}

\def\eqref#1{equation~\ref{#1}}
\def\1{\bm{1}}

\DeclareMathAlphabet{\mathsfit}{\encodingdefault}{\sfdefault}{m}{sl}
\SetMathAlphabet{\mathsfit}{bold}{\encodingdefault}{\sfdefault}{bx}{n}

\PassOptionsToPackage{table}{xcolor} 
\usepackage[T1]{fontenc}

\usepackage{amsmath,amssymb,amsthm}
\usepackage{booktabs}
\usepackage[table]{xcolor}   % [table] enables \cellcolor in tabulars
\usepackage{multirow}
\usepackage{array}
\usepackage{graphicx}
\usepackage{algorithm}
\usepackage{algpseudocode}
\usepackage{pifont}
\usepackage{url}
\usepackage{hyperref}

\newtheorem{proposition}{Proposition}
\newtheorem{corollary}{Corollary}
\newtheorem{lemma}{Lemma}
\newtheorem{assumption}{Assumption}

\newcommand{\Fcal}{\mathcal{F}}

\newcommand{\ind}{\mathbf{1}}
\usepackage{enumitem}   % dotted leaders in the appendix index
\newcommand{\betaeff}{\beta_{\mathrm{eff}}}

\title{Think More or Read More? A Token-Budget Theory of Reasoning with Shared Context}

\author{Manh Nguyen\thanks{Corresponding Author}\hspace{3em}Sunil Gupta\hspace{3em}Hung Le\\
Applied Artificial Intelligence Initiative, Deakin University, Australia\\
\texttt{\{manh.nguyen, sunil.gupta, thai.le\}@deakin.edu.au}
}

\begin{document}
\maketitle

\begin{abstract}

Reasoning systems increasingly improve answers by incorporating additional context into the model's input, from multi-agent debate, where agents read their peers' responses, to self-refinement, where a model revisits its own drafts. In these systems, communication consumes context as one agent's output becomes another's input.
We study the regime where token reading and generation draw on one fixed token window, so every token spent reading is a token that cannot be spent reasoning. 
We develop a \textbf{token-budget theory} that formalizes this trade-off, modeling accuracy as a function of the reasoning budget remaining after reading and aggregating votes with a Condorcet jury model that permits correlated errors.
For multi-agent debate, the theory gives a closed-form prediction for the optimal number of agents. This quantity is estimated from single-agent measurements, allowing the optimal group size to be determined without running debate, and the same accounting extends to a heterogeneous set of different models. 
We empirically verify our theory and show that, at the predicted group size, our method outperforms strong baselines under the tightest calibration spend, where debate-calibrated rules can afford only one or two questions.
Self-refinement obeys the same accounting, with its own draft in place of a peer message. 
The theory further predicts that communication pays only when the gain in the accuracy ceiling exceeds the reasoning room spent on reading, with this condition tightening as the window shrinks. 
Across five open models and four reasoning tasks, self-consistency matches or beats deployed debate at $93$\% of $980$ matched window-and-group-size settings, with the gap being largest under tight windows and nearly disappearing at a loose one, consistent with the theory. 
Under a shared window, how much room a model has left to \textit{think matters more than what it reads}.

\end{abstract}

\section{Introduction}
\label{sec:intro}
Large language models (LLMs) are increasingly used in reasoning pipelines that reuse model-generated text as input to subsequent computation. Multi-agent debate exchanges answers among agents~\citep{du2023debate,liang2024encouraging}, while self-refinement feeds a model its own previous drafts~\citep{madaan2023selfrefine}. Such reuse can improve reasoning by providing additional information, but under a fixed token budget, it also raises a fundamental question: not \emph{whether} added feedback helps, but \emph{how much} it is worth in terms of the reasoning budget it consumes.

\begin{figure}[t]
\centering
\includegraphics[width=\linewidth]{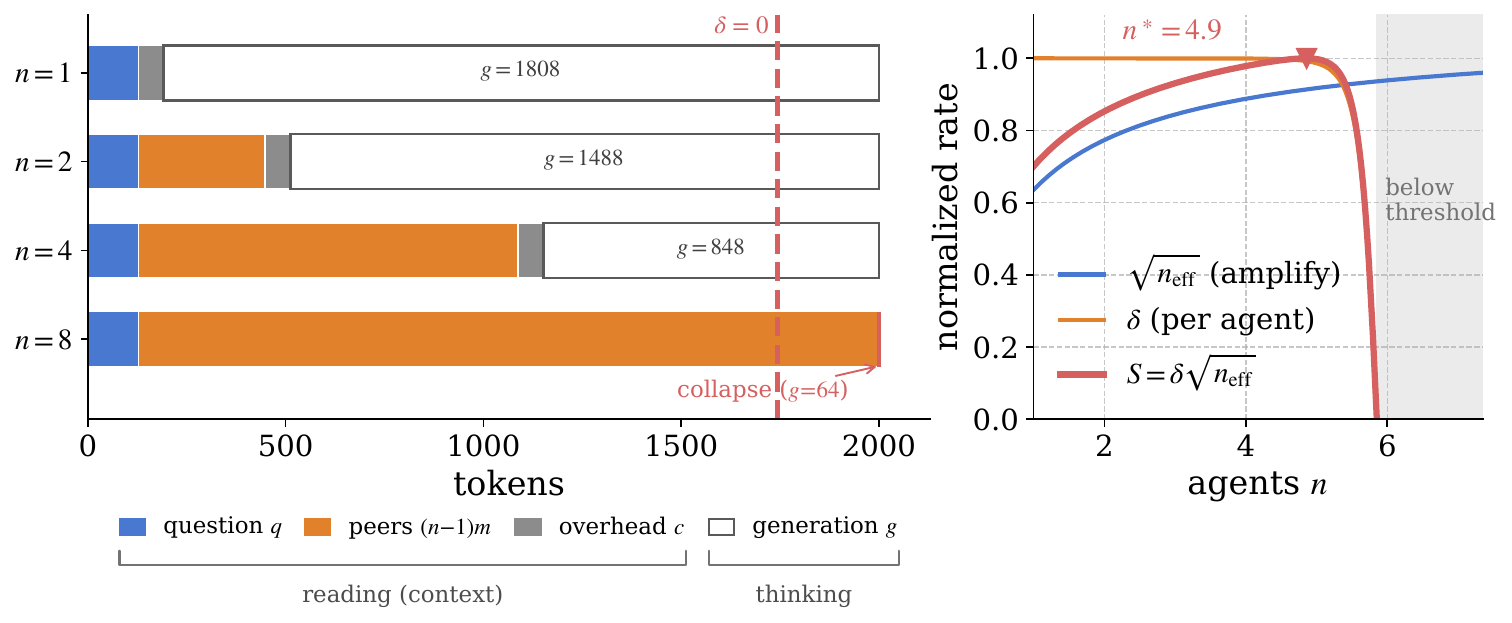}
\caption{Overview of our method (Qwen2.5-7B, GSM8K). One window $W$ pays for both
reading (question plus peer messages) and thinking (generation $g$), and each extra agent adds $m$ tokens of reading.
The per-agent edge $\delta$ is the gap between the probability of the correct answer and that of the strongest wrong answer: larger $\delta$ means a stronger correct-answer margin, while $\delta\leq0$ means that a wrong answer is at least as likely. \textbf{Left}: The dashed line marks the smallest thinking room at which $\delta>0$; a group that reads past it (here $n{=}8$) leaves too few tokens for reasoning and its accuracy collapses.
\textbf{Right}: The effective sample size $n_{\mathrm{eff}}$ measures the independent-vote equivalent of $n$ correlated agents, while $\delta$ falls as thinking room shrinks. Their product $S$, the synergy index, therefore captures the trade-off between per-agent accuracy and voting amplification and peaks at $n^\ast=4.9$. The shaded band marks $\delta{\leq}0$, where votes amplify a wrong majority and extra agents cannot help.}
\label{fig:overview}
\end{figure}

Majority-voting analyses predict that adding agents amplifies a correct majority until correlated errors cap the gain~\citep{ladha1992condorcet,liu2026phase}.
% on multi-hop reasoning tasks, 
However, on multi-hop reasoning tasks, a single long chain-of-thought~\citep{wei2022cot} often matches or outperforms a multi-agent system under an equal thinking budget~\citep{tran2026single}. Debate also does not consistently outperform self-consistency on an accuracy-per-token-cost frontier~\citep{smit2024mad, bertalanic2026consensus}, echoing the broader trade-off between token expenditure and accuracy in sampling and search~\citep{snell2024scaling, brown2024monkeys}.
The token budget deserves separate attention because it corresponds directly to the cost incurred in practice. At current API prices, inference can range from cents to tens of dollars per million tokens, depending on the model and provider~\citep{openai_pricing_2026, google_gemini_pricing_2026, anthropic_pricing_2026}. For multi-agent debate, the cost grows further with the number of agents $n$ (group size): adding agents increases not only the amount of reasoning generated, but also the amount of peer text that must be read and processed.
Existing work attacks this cost either by compressing agent communication or by fixing the reasoning budget, but treats communication and reasoning as separate costs, despite both consuming the same token budget ($\S$\ref{sec:related}).
Once the token budget is viewed this way, a trade-off appears: with $n$ small, more of the budget remains for thinking, while with $n$ large, most of it is spent on reading peer messages, leaving little room to reason.
The question is then how to estimate the best group size $n^\ast$ for multi-agent debate under this trade-off.
Existing approaches measure $n$ by sweeping or adaptively searching the design space, for instance, successive halving~\citep{jamieson2016halving}. Both require debate runs to identify the operating point, without explaining \emph{why} one group size is optimal or \emph{when} adding more agents stops improving accuracy.

We therefore introduce a \textbf{token-budget theory} that answers both questions in closed form.
Consider homogeneous debate in which each agent, working under a window of $W$ tokens, must first pay for what it reads, namely the question and its peers' messages, together $H$ tokens, plus a small fixed overhead $c$. The remainder, $g=W-H-c$, is the room it has left to reason.
Instead of assigning each agent a fixed chance of voting correctly, we introduce its \emph{ability}, which is the accuracy it reaches when its reasoning is capped at $g$ tokens. 
Ability grows with thinking room, too little room can prevent an agent from reaching an answer, while beyond a sufficient room, additional tokens yield little further gain.
We let it grow with $g$ (Eq.~\ref{eq:beff-gen}) via the fittable thinking-gain function $\phi$, whose functional form and parameters can be calibrated from held-out calibration data. We then combine the resulting votes with a Condorcet jury model that allows correlated errors, developing the per-answer margin $\delta$ (Eq.~\ref{eq:delta}), the group accuracy $A(n)$, the effective sample size $n_{\mathrm{eff}}$ and the synergy index $S(n)=\delta\sqrt{n_{\mathrm{eff}}}$ (Eq.~\ref{eq:synergy}) in Section~\ref{sec:method}.
Figure~\ref{fig:overview} provides an overview of our method.

Our model yields a closed-form prediction of an \textbf{interior optimal group size $n^\ast$} that \emph{moves with the budget}.
It is read off single-agent measurements: we measure the ability curve with a cap sweep and estimate one error-correlation parameter at the \emph{initial round}, then predict group size without running the debate. 
With a sufficiently large budget, our predictions approach the classical correlation-limited ensemble, where the main benefit comes from aggregating independent evidence.
The practical advantage of our approach is therefore greatest in the tight-budget regime, where choosing the group size correctly matters most and our method achieves the best accuracy-budget trade-off. At the predicted size, it incurs only a $1.4\%$ accuracy gap from a per-cell oracle, using between $26$ and $79$ calibration questions depending on the cell, where the debate-calibrated rules use only one or two at matched spend. The accounting extends to a heterogeneous model set, where the rule chooses which models to include as well as how many.

Our framework also clarifies the cost of communication relative to self-consistency~\citep{wang2023selfconsistency}. Self-consistency leaves the full window available for reasoning, whereas debate spends part of that budget processing peer responses, so communication must provide sufficient informational benefit to justify this cost. In our experiments, self-consistency matches or beats deployed debate at $93\%$ of the $980$ matched operating points, while a paired test finds debate significantly ahead at only $0.5\%$. The gap is largest under tight budgets and nearly disappears at a loose budget, where communication consumes a smaller share of the available reasoning budget.
Self-refinement, which rereads only its own last draft without voting, pays the same fixed reading cost every round with none of the amplification that voting provides, so its accuracy collapses faster than debate's as rounds accumulate.

Our contributions are as follows.
\textbf{First}, we introduce a \textbf{token-budget theory} that predicts the \textbf{optimal debate size} $n^\ast$ in closed form and extends to heterogeneous model sets, outperforming strong baselines under the tightest calibration spend without running a debate.
\textbf{Second}, the theory reveals when \textbf{debate's voting benefits outweigh its communication cost} relative to self-consistency under the same token budget.
\textbf{Third}, ablations show that \textbf{reading cost} drives collapse and the interior optimum, with robustness tests covering longer reasoning and practical token pricing.

\section{Related Work}
\label{sec:related}

\paragraph{Multi-agent debate in LLMs.} LLM instances critique and revise each other's answers to improve reasoning relative to a single pass~\citep{du2023debate, liang2024encouraging}, and voting over more agents can improve accuracy in practice~\citep{wang2023selfconsistency,li2024moreagents}. Each additional agent, however, also introduces peer messages that other agents must process, increasing the communication burden of the group. Existing work addresses this burden by reducing communication through sparse speaker selection, group partitioning, or retaining only conflicting messages~\citep{li2024sparse, liu2024groupdebate, zeng2025smad, nguyen2026dar}. A complementary line studies whether debate is worthwhile under a matched sampling or thinking budget, with several works finding that self-correction or a single long reasoning chain can match or outperform debate~\citep{huang2024selfcorrect, tran2026single, bertalanic2026consensus}.
A third line promotes diversity through heterogeneous model sets, arguing that homogeneous evaluation overvalues debate~\citep{zhang2025overvaluing}, attributing vanilla debate's weakness to lost candidate diversity~\citep{zhu2026demystifying}, and showing that mixed panels revise less harmfully than homogeneous ones~\citep{nilayam2026heterogeneous}. Beyond debate, this idea appears in mixtures of agents~\citep{wang2024moa}, model routing~\citep{ong2024routellm}, and cost-aware cascades~\citep{chen2023frugalgpt}. 
% These approaches reduce communication, promote diversity, or allocate computation, but do not explicitly model the reading and reasoning costs as competing uses of a shared token budget. We instead model these two costs jointly, derive the resulting optimal group size in closed form, and extend the analysis to heterogeneous model sets.

\paragraph{Estimating the operating point.} Existing approaches search the design space, e.g., via successive halving~\citep{jamieson2016halving}, or learn the operating point from data, but both require debate runs for calibration.
\citet{chen2024morecalls} make voting accuracy non-monotone in call count through query-difficulty diversity, but their optimum is refit separately at each budget, whereas ours transfers from a single-agent cap sweep. 
\citet{liu2026phase} give the closest theoretical treatment, deriving an amplification-collapse threshold in closed form from a power-law ability curve. Their ability curve is calibrated in simulation and remains concave, so it cannot capture the near-zero accuracy below the minimum length needed for a solution, the regime where our collapse occurs. Their accounting also holds the per-agent generation allowance fixed, so adding an agent costs no reasoning room, making their collapse a consequence of correlation alone rather than reading. We recover their setting as the $m=0$ case ($\S$~\ref{sec:ablations}).

\paragraph{Token Budgeting in LLMs.} A separate line studies how to allocate a fixed token budget across samples or reasoning length for a single agent, including token budgets for chain-of-thought, budget forcing at decoding time, budget-aligned tree search, and the non-monotone effect of reasoning length on task success~\citep{han2025tokenbudget,muennighoff2025s1,bgmcts2026,qi2026brief}. Controlled comparisons at matched sampling or search budgets find that voting, rather than talking, carries the gain~\citep{choi2025debatevote, kaesberg2025voting, wang2024bounds}, while \citet{wunderlich2026pareto} report the opposite ordering under per-agent rather than shared contexts. 
These methods budget within reasoning or sampling, leaving peer text outside the account.
% These methods allocate the budget within reasoning or sampling, but do not treat reading peer responses as a competing use of the same tokens. We instead formulate reading and reasoning jointly for debate and self-refinement, deriving the operating point from a single-agent cap sweep rather than recalibrating it at each budget.

\section{Method}
\label{sec:method}

We formalize the token-budget problem behind debate, define per-agent ability as a function of the room left to think, and turn it into a group-accuracy prediction. We call the resulting predictor
\textbf{TB-MAD}, for token-budget multi-agent debate: it reads a single-agent probe sweep and returns the group size, with no debate run.

\subsection{Problem Formulation}
\label{sec:conserved}

Consider $n$ LLM agents that debate a question from a task $Q$ over several rounds under one token budget $W$, the \emph{window}. Reading and thinking draw on the same $W$, where tokens spent reading peer messages are tokens no longer available for the agent's own chain of thought, and vice versa. Each peer message costs about $m$ tokens on average to read once. Under last-round visibility, which we use throughout, agent $i$ retains the question ($q$ tokens), a fixed overhead ($c$), and $n-1$ peer messages, so its retained history and remaining room to think are
\begin{align}
\label{eq:budget}
H \;=\; q + (n-1)\,m, \qquad\qquad g \;=\; W - H - c .
\end{align}
Retaining the full transcript would increase $H$ at every $n$ and can only decrease $g$, so the penalty we measure is a lower bound.
Self-refinement uses a different accounting. The group size is one, so there are no peer
messages, but the agent retains its own last draft. Its history is $H = q + m$ and its room is $g = W - q - m - c$. In contrast, Eq.~\ref{eq:budget} with $n=1$ describes a debate agent with no peer messages, corresponding to self-consistency.

Adding an agent raises $H$ and lowers $g$ by $m$ tokens, which cuts both ways. When $W$ is generous, $g$ stays past the point where extra tokens stop buying accuracy and a larger group is close to free.
When $W$ is tight, the agent pushes $g$ below the room a chain needs to reach an answer, and the lost generation budget outweighs the extra vote. We call this \emph{collapse}: under a shared budget, more agents can reduce accuracy. Given $W$, we want the group size that maximizes task accuracy,
\begin{align}
\label{eq:problem}
\max_{n} \; A_Q(n, W) \quad \text{s.t.} \quad g(n) = W - H(n) - c \ge 0 ,
\end{align}
where $A_Q(n,W)$ is the debate group's accuracy at size $n$ and window $W$. Since a
debate agent's vote depends on how much of the window it got to reason in, solving
Eq.~\ref{eq:problem} needs a model of per-agent ability as a function of $g$, which we give next.

\subsection{Per-Agent Ability and the Thinking Gain}
\label{sec:ability}

Rather than assuming fixed voting accuracy, we let that chance depend on what the window leaves an agent to think. The \emph{per-agent ability} $\betaeff(g)$ is the accuracy of a single agent whose generation is capped at $g$ tokens, measured directly with a cap sweep and no debate. We drop the $Q$ subscript below, since every quantity is on the same task.
\begin{align}
\label{eq:beff-gen}
\betaeff(g) \;=\; p_{\max}\,\underbrace{\phi(g)}_{\text{thinking gain}},
\qquad \phi(g) \;=\; \betaeff(g)/p_{\max} \;\in\; [0,1],
\end{align}
where $p_{\max}=\lim_{g\to\infty}\betaeff(g)$ is the ceiling ability with unlimited room, so the
\emph{thinking gain} $\phi$ is the fraction of its own ceiling an agent reaches at cap $g$.
\begin{assumption}[Admissible gain]
\label{assump:admissible}
$\phi:[0,\infty)\to[0,1]$ is continuously differentiable, strictly increasing ($\phi'>0$) on $(0,\infty)$, and saturating ($\phi(\infty){=}1$), and satisfies $\phi(0){=}0$.
\end{assumption}
\noindent More room cannot hurt, since a longer allowance can always reproduce a shorter chain, and it eventually stops helping. This is consistent with the diminishing return in test-time-compute scaling literature~\citep{snell2024scaling}.

\paragraph{Choice of $\phi$.} The cap sweeps are \emph{sigmoidal}, near zero below the length a solution needs, rising sharply around it, then saturating. Held-out error selects the sigmoid \emph{family} over the concave power law prior work adopts~\citep{liu2026phase}. Isotonic regression extrapolates competitively out of range but is far worse under leave-one-out. Within the sigmoid family, the \textsc{Hill} form has the lowest total error (Appendix Table~\ref{tab:forms_mean}). We adopt it,
\begin{align}
\label{eq:beff}
\betaeff(g) \;=\; p_{\max}\,\frac{g^{\kappa}}{g^{\kappa}+g_0^{\kappa}},
\end{align}
with scale $g_0$ and shape $\kappa>0$, fitted on a separate calibration set. Under this form, $\phi(g)$ could be interpreted as the probability that a solution fits within a budget of $g$ tokens, making $g_0$ a median solution length (Appendix~\ref{app:our-implement}).

\begin{table}[t]
\centering
\caption{Candidate forms for the thinking gain $\phi$, adopted form in bold. \emph{Right:} all five fitted to one measured cap sweep (Qwen2.5-7B/GSM8K).}
\label{tab:forms}
\begin{minipage}[c]{0.55\linewidth}
\setlength{\tabcolsep}{4pt}
\renewcommand{\arraystretch}{1.3}
\small
\begin{tabular}{@{}l l@{}}
\toprule
Form & $\phi(g)$ \\
\midrule
\textbf{Hill}~\citep{hill1910possible} & $\dfrac{g^\kappa}{g^\kappa+g_0^\kappa}$ \\
Burr XII~\citep{burr1942} & $1-\left(1+\left(\frac{g}{g_0}\right)^\kappa\right)^{-\lambda}$ \\
Weibull~\citep{weibull1951statistical} & $1-e^{-\left(\frac{g}{g_0}\right)^\kappa}$ \\
Power law~\citep{pareto1964cours} & $1-(1+\frac{g}{g_0})^{-\kappa}$ \\
Monotone~\citep{isotonicregression} & isotonic step \\
\bottomrule
\end{tabular}
\end{minipage}\hfill
\begin{minipage}[c]{0.44\linewidth}
\centering
\includegraphics[width=\linewidth]{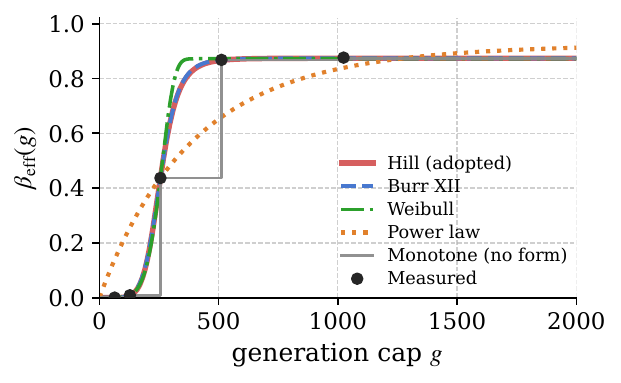}
\end{minipage}
\end{table}

\subsection{Group Accuracy and the Synergy Index}
\label{sec:synergy}
To turn per-agent ability into group accuracy, we use a correlation-aware Condorcet model~\citep{condorcet1785essai,boland1989majority,ladha1992condorcet} that captures how additional agents help until correlated errors diminish the gain. The following three simplifications make it tractable.
We read $\betaeff$ as a per-question probability where Eq.~\ref{eq:beff} defines a task average~\citep{chen2024morecalls}, and treat the agents as \emph{homogeneous} so all share $\betaeff(g(n))$. Moreover, we assume their errors are \emph{equicorrelated}, with every pair sharing a $\rho$ independent of $n$, estimated at round $0$ before any agent reads a peer and so measurable without running a debate (Appendix~\ref{app:rhoest}). $\rho$ plays the role that diversity measures play in the ensemble literature~\citep{KroghVedelsby1994, KunchevaWhitaker2003}, capturing a similar trade-off between individual accuracy and error independence.
% Using the conservative initial round (round-$0$) value instead of the true post-debate one costs under a tenth of an agent in predicted best size (Appendix~\ref{app:round1}).

For a given question, let the correct answer's margin over the strongest wrong answer be
\begin{align}
\label{eq:delta-def}
\delta(n) = \betaeff\!\bigl(g(n)\bigr) - q_{\max}(n),
%THINK MORE OR READ MORE? A TOKEN-BUDGET THEORY OF REASONING WITH SHARED CONTEXT
\qquad q_{\max} = (1-\betaeff)\,c_w,
\end{align}
where $q_{\max}$ is the highest-probability wrong answer and $c_w$ is the \emph{error concentration}, the probability of producing the strongest wrong answer conditional on being wrong.
On a $K$-way task, $c_w$ is bounded between $1/(K-1)$ and $1$, with the lower bound attained when errors are spread uniformly over the wrong options and the upper bound when every error picks the same wrong option.
% On a $K$-way choice this is one option's share of the wrong-answer mass while on open-ended tasks (e.g. GSM8K), $c_w$ is measured from the round-$0$ answer distribution as the fraction of a question's wrong-answer mass.
We hold $c_w$ fixed across $n$ and $g$, so all budget dependence enters through $\betaeff$. Thus
\begin{align}
\label{eq:delta}
\delta = (1+c_w)\betaeff-c_w,
\end{align}
which is positive exactly when $\betaeff>\tau\equiv c_w/(1+c_w)$.
Across answers, $c_w$ is analogous to $\rho$ across agents: $\rho$ captures dependence in whether agents are correct, while $c_w$ captures dependence in which wrong answer they choose when they are incorrect. Both are measured at round $0$ and held fixed and debiased for small samples (Appendix~\ref{app:cw}).

Score agent $j$ by $X_j=+1$ when it answers correctly, $X_j=-1$ when it gives the strongest wrong answer, and $X_j=0$ otherwise, so the group prefers the correct answer exactly when $\bar X_n>0$, with $\mathbb{E}[X_j]=\delta$ by homogeneity. Write $\sigma^2=\operatorname{Var}(X_j)$ for this per-agent variance, common to all $j$ by homogeneity (Appendix~\ref{app:sigma} gives its closed form in
$\betaeff$ and $c_w$). Equicorrelation with $\rho\in(0,1)$ gives $\operatorname{Cov}(X_i,X_j)=\rho\sigma^2$ and hence
\begin{align}
\label{eq:neff}
\operatorname{Var}(\bar X_n)
&= \frac{\sigma^2\bigl(1+\rho(n-1)\bigr)}{n}
= \frac{\sigma^2}{n_{\mathrm{eff}}(n)}, \\
n_{\mathrm{eff}}(n)
&= \frac{n}{1+\rho(n-1)}
\xrightarrow[n\to\infty]{}\frac{1}{\rho}.
\end{align}
Thus $n$ correlated agents provide the effective sample size $n_{\mathrm{eff}}$, which is bounded by $1/\rho$ and yields diminishing returns~\citep{kish1965survey}. Approximating $\bar X_n$ by a normal with the same moments,
\begin{align}
\label{eq:accuracy}
A(n) = \Pr[\bar X_n>0] \simeq \Phi\!\left(\frac{\delta(n)\sqrt{n_{\mathrm{eff}}(n)}}{\sigma}\right),
\end{align}
where $\Phi$ is the standard normal CDF and $\delta(n)$ carries $g(n)$ from Eq.~\ref{eq:budget}.
Eq.~\ref{eq:accuracy} is an \emph{aggregation head}: given a per-agent ability, it maps that ability to the accuracy of the resulting vote. For debate, feeding it the pre-revision ability $\betaeff(g(n))$ gives the initial-round vote accuracy $A(n)$. After revision, the same head applies to the agents' updated ability, so revision changes the ability entering the head instead of the aggregation mechanism itself. We therefore use the pre-revision head to determine the optimal operating point $n^\ast$, while treating revision as a separate accuracy lift whose effect is evaluated on the deployed debate.

Because $\Phi$ is monotone and $\sigma>0$, we can maximize the signed \emph{synergy index}
\begin{align}
\label{eq:synergy}
\arg\max_n A(n)
= \arg\max_n S(n),
\qquad
S(n) = \delta(n)\sqrt{n_{\mathrm{eff}}(n)}.
\end{align}
$\delta(n)$ is the per-agent edge the budget supplies; $\sqrt{n_{\mathrm{eff}}(n)}$ the amplification voting supplies. Keeping the sign of $\delta$ matters for the objective's behaviour just outside $\Fcal$. Votes amplify a negative edge as readily as a positive one, so $S$ falls once $\delta$ turns negative and the rule returns $n=1$, where an unsigned measure would keep climbing. Since $n^\ast$ is read off $S(n)$ rather than accuracy itself, it does not depend on $\Phi$ being Gaussian: any strictly increasing map from margin to accuracy yields the maximizer. As $n$ grows the amplification rises while $\delta(n)$ falls, because each agent costs $m$ tokens of room and so lowers $\betaeff$, producing the interior optimum.
Eq.~\ref{eq:synergy} treats $\sigma$ as a constant fit once per cell; once the vote score is fixed, $\sigma$ is instead a function of $n$ through $\betaeff(n)$, and substituting that derived $\sigma(n)$ moves $\hat n^\ast$ at only $7$ of $140$ swept operating points (Appendix~\ref{app:sigma}).
We summarize the procedure in Appendix Alg.~\ref{alg:predict}.

\paragraph{Self-refinement as the degenerate case.} Self-refinement reads its own previous draft, corresponding to $n{=}1$ while still paying the reading cost. Hence $g_{\mathrm{ref}}(W)=W-q-m-c$, reducing $\delta$ through the generation-floor effect, while $\sqrt{n_{\mathrm{eff}}}=1$ provides no amplification. Repeating rounds therefore does not improve $S$ unless revision increases the ability of subsequent generations.

\subsection{Theory: the collapse regime}
\label{sec:theory}

We now analyze what happens when ability is tied to the window, using a single model in which reading consumes reasoning budget and yields the following results.
% : (1) the optimal group size is unique and interior, grows sub-linearly with the window, and eventually stops binding (Prop.~\ref{prop:single-peak}, Appendix Cor.~\ref{cor:nstar}); (2) at matched aggregation, reading peers cannot beat the initial-round vote unless peer text raises the accuracy ceiling $p_{\max}$ (Prop.~\ref{prop:illusion}); and (3) self-refine is the $n{=}1$
% case, so a single untouched agent is the ceiling unless re-reading raises $p_{\max}$.

\paragraph{The interior best size.}
\label{sec:interior}

Throughout, $n$ is continuous, $g(n)=W-q-c-(n-1)m$, and we restrict to the feasible region $\Fcal=\{n:\delta(n)\ge0\}$. The constraint is set by $\delta$, not $g$: $\delta(n)$ reaches zero strictly before $g(n)$ does, so $\Fcal={n:\delta(n)\ge0}\subset{n:g(n)\ge0}$, while $\phi(g)\to0$ continuously as $g\to0$.

\begin{proposition}[Single-peaked accuracy]
\label{prop:single-peak}
Let $\phi$ satisfy Assumption~\ref{assump:admissible}, with $m>0$, $p_{\max}\in(0,1]$, $\rho\in[0,1)$, $\delta(1)>0$, and suppose some larger group helps ($S(n)>S(1)$ for some $n\in\mathrm{int}(\Fcal)$). Then $\delta$ is strictly decreasing on $\Fcal=[1,n_0]$, where $n_0>1$ is the unique group size at which the average agent's edge over the strongest distractor vanishes ($\delta(n_0){=}0$), and $S(n)^2$ attains an interior maximum $n^\ast\in(1,n_0)$.
\end{proposition}
\begin{proof}
See Appendix~\ref{app:single-peak}.
\end{proof}
\noindent Every form in Table~\ref{tab:forms} has $\phi(0){=}0$, so $\delta<0$ once the reasoning budget is exhausted at $n_{\mathrm{exh}}=1+(W-q-c)/m$ (distinct from the deployment cap $n_{\max}$ used elsewhere); hence $n_0$ exists and $\Fcal$ lies strictly inside the domain of $\phi$.
The mechanism is a race between two rates, one shrinking like $1/n^{2}$ as an extra agent \textit{adds} amplification, the other growing without bound as it \textit{costs} per-agent edge with $\delta\to0^{+}$ near $n_0$. 
Continuity on a compact interval and a vanishing edge at $n_0$ give existence; uniqueness follows from log-concavity of the shifted edge $\delta$(Appendix~\ref{app:single-peak}).
The $1/n^2$ rate limits how fast $n^\ast(W)$ can grow. Each extra agent costs $m$ tokens of room while its amplification gain shrinks, keeping $n^\ast$ below the feasibility cap $n_{\max}$ (Appendix Cor.~\ref{cor:nstar}). As $W \to \infty$, thinking room ceases to bind and the classical correlation ceiling returns, so an interior optimum arises only when the window is scarce (Appendix~\ref{app:largew}).

\paragraph{Reading cannot beat voting at matched aggregation.}
\label{sec:illusion}

Self-consistency draws $n$ chains from one agent without exchanging messages, so each reads only the question
($g_{\mathrm{SC}}{=}W-q-c$), while a debate agent also retains $n-1$ peer messages and keeps only
$g(n)\le g_{\mathrm{SC}}$. Thus, debate's ``synergy'' is a \textit{voting} effect unless peer text raises
$\Delta p_{\max}$, which we measure in \S\ref{sec:main}.

% Debate's ``synergy'' is therefore a voting effect unless peer text raises
% $\Delta p_{\max}$. We do not measure $\Delta p_{\max}$ directly, so Prop.~2 bounds the
% reading cost of debate rather than settling its net value, and our empirical comparison
% in \S\ref{sec:main-results} measures the two together. We return to this in Appendix~E.

\begin{proposition}[Self-consistency dominates the initial-round vote]
\label{prop:illusion}
Let $A^{0}_{\mathrm{deb}}(n,W)$ be initial-round debate vote accuracy of Eq.~\ref{eq:accuracy}, before any revision. Under Assumption~\ref{assump:admissible}, at matched aggregation (same $n$, $\rho$, $\sigma$), $A_{\mathrm{SC}}(n,W)\ge A^{0}_{\mathrm{deb}}(n,W)$ for every $n\ge1$ and $W$, with equality at $n{=}1$. If peer text raises the ceiling by $\Delta p_{\max}$, then debate can exceed self-consistency at $(n,W)$ only when
\begin{align}
\label{eq:info-vs-cost}
\Delta p_{\max}\,\phi\bigl(g(n)\bigr) \;>\; p_{\max}\bigl(\phi(g_{\mathrm{SC}})-\phi(g(n))\bigr),
\end{align}
and the dominance result above is the special case $\Delta p_{\max}\le0$.
\end{proposition}
\begin{proof}
See Appendix~\ref{app:illusion}.
\end{proof}

\paragraph{Self-Refinement.} For self-refinement ($n{=}1$ case), Prop.~\ref{prop:illusion} becomes $A_{\mathrm{ref}}(W)\le p_{\max}\phi(g_{\mathrm{SC}})$,  a single untouched agent is the ceiling unless re-reading raises $p_{\max}$. 
\S\ref{sec:main} measures it.
\section{Experiments}
\label{sec:exp}

\subsection{Setup}
\label{sec:setup}

\paragraph{Models, Data And Metrics.} We run five open instruction models across four families from $3$--$12$B: \textbf{Llama-3.2-3B}~\citep{llama3},
\textbf{Qwen2.5-7B}~\citep{qwen2025}, \textbf{Falcon3-7B}~\citep{falcon3}, \textbf{Gemma-2-9B}~\citep{gemma2} and \textbf{Gemma-3-12B}~\citep{gemma3}. We use four reasoning benchmarks, two difficulty levels in math and two non-math domains: \textbf{GSM8K}~\citep{cobbe2021gsm8k},
\textbf{MATH500}~\citep{lightman2024verify}, ten-way \textbf{MMLU-Pro}~\citep{wang2024mmlupro} and five-way \textbf{CSQA}~\citep{talmor2019commonsenseqa} (Appendix Table~\ref{tab:datasets}). For all tasks, we report accuracy in percentage.

\paragraph{Baselines.} We compare our closed-form prediction against eight alternatives grouped by calibration cost.
\emph{Debate-calibrated}, using outcomes from actual debate runs that our method does not require: \textbf{Empirical Argmax} selects the best accuracy among tested $(n,W)$ configurations, \textbf{Successive Halving}~\citep{jamieson2016halving} spends the same budget adaptively, and a \textbf{Quadratic Surface} in $(n,W)$. 
\emph{No calibration:} \textbf{naive $n{=}1$} and \textbf{naive $n{=}n_{\max}$}.
\emph{Probe-calibrated, as ours is:} \textbf{Learned} $\betaeff$ replaces the functional form with gradient-boosted trees; \textbf{Star Exponent}~\citep{liu2026phase} substitutes the power-law ability curve into our aggregator and budget model, which isolates the effect of the functional form; and \textbf{Difficulty Mixture}~\citep{chen2024morecalls} estimates per-question accuracy from probe sweeps and simulates voting, with the per-call allowance adjusted under our window accounting as $n$ varies.
More details are in Appendix~\ref{app:baseline-implement}.

\paragraph{Implementation Details.} 
We report \textit{regret} as the percentage-point accuracy gap between the selected group size and an oracle that chooses the highest-accuracy configuration for each model–dataset pair.
We sweep the grid of Table~\ref{tab:protocol}, extending to $W \in \{16k, 32k\}$ and $n$ up to $32$ to test the large-$W$ limit. For all baselines, the first $40\%$ of test samples are used for calibration and the remaining $60\%$ for evaluation, with calibration costs accounted for as described in Appendix~\ref{app:calibration}. 
All runs use vLLM~\citep{kwon2023vllm} at temperature $1.0$ across three seeds, pooling samples across seeds before fitting rather than fitting each seed separately and averaging the estimates. This choice has negligible effect on the results (Appendix~\ref{app:our-implement}). More details are in Appendix~\ref{app:protocol}.

% ------

\begin{table}[t]
\centering
\caption{Held-out regret against the per-cell oracle (percentage points) when each method chooses $n$ at an equal \emph{token} spend. \textbf{Bold} is the lowest regret in a column and \underline{underline} the second lowest, the \textbf{Avg.}\ column included. Per-cell values for all calibration spend in Appendix Tables~\ref{tab:matchedcostpercell}--\ref{tab:matchedcostpercell2}.}
\label{tab:matchedcost}
\setlength{\tabcolsep}{3.5pt}
\small
\begin{tabular}{@{}lcccccccccc@{\hspace{11pt}}c@{}}
\toprule
\multirow{2}{*}[-3pt]{\textbf{Method}} & \multicolumn{5}{c}{\textbf{1.0M budget}} & \multicolumn{5}{c}{\textbf{2.0M budget}} & \multirow{2}{*}[-3pt]{\textbf{Avg.}} \\
\cmidrule(lr){2-6}\cmidrule(lr){7-11}
 & L-3B & Q-7B & F-7B & G2-9B & G3-12B & L-3B & Q-7B & F-7B & G2-9B & G3-12B &  \\
\midrule
\multicolumn{12}{@{}l}{\emph{Debate-calibrated}} \\
Successive Halving & 7.13 & 12.05 & 9.18 & 5.63 & 11.39 & 5.52 & 4.93 & 6.39 & 3.38 & 8.84 & 7.44 \\
Empirical Argmax & 4.88 & 3.17 & 3.37 & 3.19 & 1.41 & 4.48 & 2.38 & 2.77 & 3.05 & 1.62 & 3.03 \\
Quadratic Surface & 4.20 & 2.46 & 2.43 & 2.98 & 2.27 & 4.25 & 2.24 & 2.95 & 2.07 & \textbf{0.91} & 2.68 \\
\midrule
\multicolumn{12}{@{}l}{\emph{No calibration}} \\
Naive $n{=}n_{\max}$ & 17.86 & 35.46 & 27.28 & 16.74 & 36.70 & 17.86 & 35.46 & 27.28 & 16.74 & 36.70 & 26.81 \\
Naive $n{=}1$ & 7.88 & 3.22 & 3.87 & 4.32 & \textbf{1.03} & 7.88 & 3.22 & 3.87 & 4.32 & \underline{1.03} & 4.07 \\
\midrule
\multicolumn{12}{@{}l}{\emph{Probe-calibrated}} \\
Difficulty Mixture & 5.81 & 1.32 & 2.11 & 3.54 & 1.30 & 5.37 & 1.43 & 2.16 & 3.86 & 1.38 & 2.83 \\
Star Exponent & 3.82 & \underline{1.16} & 1.47 & 2.06 & \underline{1.22} & 4.04 & 1.18 & 1.36 & 1.97 & 1.25 & 1.95 \\
Learned $\betaeff$ & \underline{2.72} & \textbf{1.14} & \underline{1.07} & \textbf{1.27} & 2.20 & \underline{2.39} & \underline{1.05} & \underline{1.05} & \textbf{1.10} & 1.84 & \underline{1.58} \\
\rowcolor{green!15} \textbf{TB-MAD} (Ours) & \textbf{2.29} & 1.29 & \textbf{0.93} & \underline{1.40} & 1.90 & \textbf{1.88} & \textbf{0.83} & \textbf{0.84} & \underline{1.15} & 1.47 & \textbf{1.40} \\
\bottomrule
\end{tabular}

\end{table}

\subsection{Main Results}
\label{sec:main}

\paragraph{Fitting And Predicting Performance.} $\betaeff(g)$ is measured by capping one agent's generation at $g$, then fitting $(p_{\max},g_0,\kappa)$, with the \textsc{hill} form chosen on held-out error from five candidate forms (Appendix Table~\ref{tab:forms_mean}). Using this fit, TB-MAD achieves the lowest regret, averaging $1.40\%$ and staying below $2\%$ across all models and eight comparisons at equal calibration spends (Table~\ref{tab:matchedcost}). The difference is largely driven by calibration cost: debate-calibrated rules afford far fewer questions, since debate expands input and output tokens by up to $49\times$ and $11\times$ (Appendix Table~\ref{tab:calibtokens}). Among probe-based methods, Star Exponent uses a power law that has higher held-out RMSE than the selected \textsc{hill} form (Appendix Table~\ref{tab:forms_mean}), while Difficulty Mixture estimates a per-question difficulty that is more data-hungry; Learned $\betaeff$ is the closest but trails while sacrificing interpretability. Actual $n^*$ predictions at $1$M spend are in Appendix Figure~\ref{fig:predict}. 
Broken out by window, TB-MAD is within $1\%$ of the best rule at six of the seven windows, trailing by $1.16\%$ at $W{=}2500$ (Appendix Table~\ref{tab:bywindow}). Its performance is therefore not driven by a single region of the sweep.

\paragraph{TB-MAD Scaling Behavior.}

TB-MAD provides the best regret-cost trade-off, with regret decreasing as the calibration set grows (Figure~\ref{fig:kscaling}). Its plain regret is already among the lowest at $10$ calibration questions and stays roughly flat beyond $60$, since one probe-fitted ability curve gains little from more data. Debate-calibrated baselines start higher and keep dropping, closing part of that gap while their cost climbs (left panel). At equal spend, TB-MAD achieves the lowest regret and saturates once the maximum number of calibration questions is reached, whereas debate-based baselines keep improving as more tokens are spent, and Successive Halving overtakes TB-MAD at the largest budget.

\paragraph{Robustness to Heterogeneous Settings.}
We extend the evaluation to a heterogeneous setting with $K$ unique models, each appearing with the same model count, and isolate the allocation problem (Appendix~\ref{app:hetero}). 
The probe-calibrated methods are again the strongest, and Learned $\beta_\text{eff}$ edges ahead of TB-MAD. Our rule trades a small amount of regret for a closed form and an explicit model-selection gate (Appendix Tables~\ref{tab:heterobaselines} and~\ref{tab:heterobk}).

\paragraph{Generalization To Self-Refinement.} Revision inherits the same generation-floor effect as debate, making self-refinement similarly token-inefficient: tight windows can reduce accuracy before any refinement benefit emerges. Useful refinement also requires both headroom and self-correction ability, so gains remain modest relative to the added token cost (Appendix~\ref{app:selfrefine}).

\paragraph{When More Agents Do Not Pay.} 
Reading peers costs generation, which resolves two questions often studied separately. First, communication does not pay as majority voting with \emph{no} communication matches or beats debate at $93\%$ of the $980$ matched $(n,W)$ operating points with $n\ge2$ (Appendix Table~\ref{tab:selfcons}). 
A paired bootstrap on the same cells is more favourable still. Self-consistency is significantly better in $57\%$ of cases, debate in just $0.5\%$, and the remaining $42.5\%$ are indistinguishable. The benefit of voting over debate has also been reported at matched sampling budgets~\citep{smit2024mad, choi2025debatevote}.
More importantly, Eq.~\ref{eq:info-vs-cost} says when the ordering should reverse, and we test that prediction across seven windows rather than at a single operating point.
At $W{=}1200$, the margin is $+42.0\%$, or $+13.7\%$ restricting to cells where $g$ is not clipped to the floor. It falls below $1$ point by $W{=}6\mathrm{k}$ and is indistinguishable from zero through $32\mathrm{k}$ on Qwen2.5-7B/GSM8K (Appendix Table~\ref{tab:largewsc}), in line with Appendix Cor.~\ref{cor:largew}'s correlation ceiling. 
The measured debate accuracy includes one revision round, so the observed margin also includes a revision lift not covered by Prop.~\ref{prop:illusion}, estimated at $0.0$–$6.7\%$ (Table~\ref{tab:selfrefine}, Appendix~\ref{app:selfrefine}).
Second, a group is not always worth forming: a single agent matches or beats the debate group at $47\%$ of the $980$ cells (Appendix Table~\ref{tab:selfcons}, single$\ge$group column), with $\delta>0$ determining when grouping pays. Single-agent probes recover the sign of $\delta$ correctly in $81\%$ of cases, compared with $56\%$ for always grouping. The gap $A(n,W)-A(1,W)$ is interior in $n$ and collapses under tight windows across all five models (Appendix Figure~\ref{fig:surface}), a generation-floor effect since raising $g_{\min}$ recovers the worst cell (Appendix Figure~\ref{fig:gmin}). Directly measuring $\Delta p_{\max}$ per agent at fixed $g$, gives a small, sign-mixed estimate across four model-task cells (Appendix Table~\ref{tab:deltapmax}), so we find no reliable ceiling gain from reading a peer, consistent with debate's apparent gains coming from the vote and not from $\Delta p_{\max}$.  
Further details are in Appendix~\ref{app:scgap}.

\begin{figure}[t]
\centering
\includegraphics[width=\textwidth]{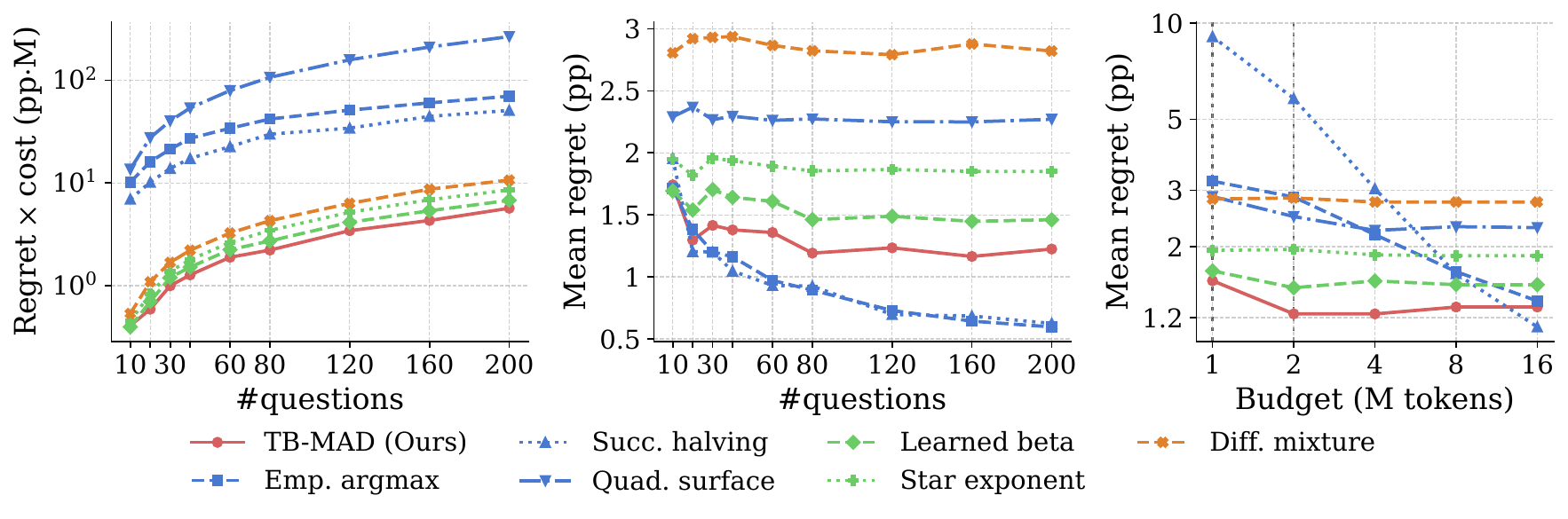}
\caption{Two ways to fix a calibration budget, equal questions (left and middle) and equal spend (right, dotted verticals at the two budgets of Table~\ref{tab:matchedcost}), and two readings of the equal-questions sweep: cost-adjusted (left) and plain regret (middle).}
% \caption{The two ways to fix a calibration spend: equal questions (left, cost-adjusted) and equal spend (right, dotted verticals at the two spends of Table~\ref{tab:matchedcost}).}
\label{fig:kscaling}
\end{figure}
\subsection{Ablations}
\label{sec:ablations}

\paragraph{The reading cost $m$.} $m$ is the token cost of reading one peer message (Eq.~\ref{eq:budget}), determining how quickly added agents reduce the group's shared thinking room. Both the collapse and interior optimum arise from $m>0$. Setting $m{=}0$ holds $g$ fixed regardless of group size, as in debate where each agent has its own window. The theory then predicts the largest feasible group at every window, with no interior optimum and no collapse (Appendix Figure~\ref{fig:ablationm}). 
More details are in Appendix~\ref{app:m}.

\paragraph{The price of a token.} We ablate the output-to-input price ratio $r{=}c_\text{out}/c_\text{in}$ to reflect practical cost models, where reading and generating tokens can have different prices. Making input more expensive reduces their performance accordingly and TB-MAD keeps the lowest regret at every ratio tried on these spends (Appendix Table~\ref{tab:rate}), while cheaper input buys a larger cost-optimal group at the same spend (Appendix Figure~\ref{fig:cost}).

\paragraph{Is the collapse a long-context artifact?}
A longer transcript could hurt through position \citep{liu2024lostmiddle,hsieh2024ruler} or length alone~\citep{levy2024sametask}, but these effects do not explain the collapse.
Moving the question to the end of the transcript, after the peer messages, at fixed $H$ and $g$ changes accuracy by less than $1\%$ on average (Appendix Table~\ref{tab:qburied}). 
% The largest collapse-cell shift (Falcon3-7B, GSM8K) is $8.8\%$, but both placements land at $9.0\%$ and $0.2\%$ against a grid-mean accuracy of $79.2\%$ for the same model and task. 
Question position moves the collapse cell without lifting it out of collapse.
\section{Conclusion}
\label{sec:conclusion}

We develop a token-budget theory of when additional context is worth its cost by placing reading and thinking in a shared window. The theory predicts optimal group size and identifies a critical window below which a single agent can outperform debate. Empirically, our method achieves the lowest regret under tight budgets, while self-consistency often matches or outperforms debate at lower cost. The theory offers a simple guide for \emph{when} to read more, think more, or add agents. 
Limitations and Future work are discussed in Appendix~\ref{app:limitations}.

\subsection*{AI use statement}

In this work, we used generative AI tools solely to support writing, including proofing, language polishing, and
improving readability. We did not use them to develop research ideas, conduct experiments, or generate the substantive content of the work. All AI-assisted
text was reviewed and verified by the authors, who take full responsibility for the final content.

\subsection*{Ethics statement}

TB-MAD lowers the token cost of deploying multi-agent debate by choosing a group size before any debate is run. Its predictions are calibrated per model and task, so a size chosen outside the regime the probe sweep covers can lose accuracy with no signal that it has. We use only public benchmarks and open-weight models, with no human subjects or personal data.

\subsection*{Reproducibility statement}

All model equations, proofs and implementation details are in Appendices~\ref{app:protocol} and~\ref{app:single-peak}--\ref{app:illusion}.
We will release the code, the banked debate transcripts and the table and figure generation scripts openly for research purposes upon publication.

\bibliographystyle{iclr2027_conference}
\bibliography{references}

\clearpage
\appendix
\section*{Appendix}
\label{app:index}

Each appendix, the main-text section it supports, and its subsections, one per line with its page.

\newcommand{\idxsub}[2]{\hspace*{1.6em}\ref{#1}~~#2 \dotfill \pageref{#1}\\[1pt]}
\begingroup
\setlength{\parindent}{0pt}
\medskip

\textbf{\ref{app:proofs}~~Proofs} \emph{(supports \S\ref{sec:theory})} \dotfill \pageref{app:proofs}\\[2pt]
\idxsub{app:single-peak}{Single-peakedness}
\idxsub{app:largew}{The large-window limit}
\idxsub{app:illusion}{Self-consistency weakly dominates debate}
\medskip

\textbf{\ref{app:protocol}~~Implementation Details} \emph{(supports \S\ref{sec:ability}, \S\ref{sec:setup})} \dotfill \pageref{app:protocol}\\[2pt]
\idxsub{app:our-implement}{Implementing TB-MAD}
\idxsub{app:baseline-implement}{Implementing baselines}
\medskip

\textbf{\ref{app:hetero}~~Generalization To Heterogeneous Settings} \emph{(supports~\S\ref{sec:main})}\dotfill \pageref{app:hetero}\\[2pt]
\idxsub{app:hetero-method}{Extended method}
\idxsub{app:hetero-theory}{Updated theory analysis}
\idxsub{app:hetero-results}{Results}

\textbf{\ref{app:extra}~~Additional Results} \emph{(supports \S\ref{sec:main}, \S\ref{sec:ablations})}\dotfill \pageref{app:extra}\\[2pt]
\idxsub{app:selfrefine}{Self-refine, full results}
\idxsub{app:scgap}{Voting against talking, and the single agent}
\idxsub{app:api-rate}{Repricing at API rates}
\idxsub{app:cw}{Is the error concentration $c_w$ constant?}
\idxsub{app:sigma}{Does $\sigma$ depend on $n$?}
\idxsub{app:generation-floor}{Generation-floor recovery}
\medskip

\textbf{\ref{app:scope}~~Limitations And Future Work} \emph{(supports \S\ref{sec:conclusion})} \dotfill \pageref{app:scope}\\[2pt]
\endgroup

\section{Proofs}
\label{app:proofs}

\subsection{Single-peakedness (Prop.~\ref{prop:single-peak})}

\label{app:single-peak}

Treat $n$ as continuous, so
\[
g(n)=W-q-c-(n-1)m, \qquad
\delta(n)=(1+c_w)\,p_{\max}\,\phi\bigl(g(n)\bigr)-c_w .
\]
Since $\phi$ is defined on $[0,\infty)$, we work on $[1,n_{\mathrm{exh}}]$ with $n_{\mathrm{exh}}=1+(W-q-c)/m$,
the (real-valued, not the deployment cap $n_{\max}$ used elsewhere) group size at which the reasoning
budget is exhausted ($g(n_{\mathrm{exh}})=0$).

\paragraph{Step 1: $\delta$ is strictly decreasing.} $g$ is strictly decreasing in $n$ ($g'(n)=-m<0$) and
$\phi$ is strictly increasing (Assumption~\ref{assump:admissible}), so $\phi(g(n))$ is strictly
decreasing and hence $\delta(n)$ is also strictly decreasing. 
Every form in Table~\ref{tab:forms} has $\phi(0)=0$, so $\delta(n_{\mathrm{exh}})=-c_w<0$.
With $\delta(1)>0$, $\delta$ continuous and strictly decreasing, there is a unique root $n_0\in(1,n_{\mathrm{exh}})$, and
$\Fcal=[1,n_0]$ lies strictly inside the domain of $\phi$. (This step uses only monotonicity of $\phi$, not the differentiability also assumed in Assumption~\ref{assump:admissible}.)

\paragraph{Step 2: the log-derivative decomposition.} Write $M(n)=S(n)^2=n_{\mathrm{eff}}(n)\delta(n)^2$
for the squared synergy index of \S\ref{sec:synergy}. Since $\phi\in C^{1}$ and $g$ is affine, $\delta\in C^{1}$;
on $\mathrm{int}(\Fcal)$ we have $\delta>0$, so $\log M(n)=\log n_{\mathrm{eff}}(n)+2\log\delta(n)$ is differentiable and
\[
\frac{\rm d}{{\rm d}n}\log M=\underbrace{\frac{n_{\mathrm{eff}}'}{n_{\mathrm{eff}}}}_{\alpha(n)}
-\underbrace{\Bigl(-2\frac{\delta'}{\delta}\Bigr)}_{\mathcal{C}(n)} .
\]
With $D(n)=1+\rho(n-1)$ and $n_{\mathrm{eff}}=n/D$,
\begin{align*}
\frac{n_{\mathrm{eff}}'}{n_{\mathrm{eff}}}
&=\frac{\rm d}{{\rm d}n}\bigl(\log n-\log D\bigr)
=\frac1n-\frac{\rho}{D}\\
&=\frac{D-\rho n}{nD}
=\frac{1-\rho}{\,n\bigl(1+\rho(n-1)\bigr)\,},
\end{align*}
is the amplification term $\alpha(n)$. It is strictly positive for $\rho\in[0,1)$, strictly decreasing
in $n$ because the denominator $nD$ is increasing. Moreover, its decay rate depends on $\rho$, with
$\alpha(n)=\Theta(n^{-2})$ as $n\to\infty$ for $\rho\in(0,1)$ and $\alpha(n)=1/n$ at $\rho=0$.
For the cost term, $\delta'(n)=-(1+c_w)p_{\max}m\,\phi'\bigl(g(n)\bigr)$, and $\phi'$ is continuous and strictly positive on the compact interval $[g(n_0),g(1)]$, so $|\delta'|$ is bounded below by some $\varepsilon>0$ on $\Fcal$. Since $\delta>0$ and $\delta'<0$ on the interior, $\mathcal{C}(n)>0$ there, and as $n\to n_0^-$ we have $\delta\to0^+$ with $|\delta'|\ge\varepsilon$,
so $\mathcal{C}(n)\to+\infty$.

\paragraph{Step 3: existence of an interior maximizer.} $M=S^2$ is continuous on the compact interval $\Fcal=[1,n_0]$, so it attains a maximum there. At the right endpoint $\delta(n_0)=0$ gives $M(n_0)=0<M(1)$, and by hypothesis $M(\tilde n)>M(1)$ for some $\tilde n\in\mathrm{int}(\Fcal)$. The maximum therefore exceeds both endpoint values and is attained at some $n^\ast\in(1,n_0)$. At any such interior maximizer $\log M$ is differentiable (Step 2), so $\alpha(n^\ast)=\mathcal{C}(n^\ast)$.

\paragraph{Step 4: uniqueness under log-concavity.} 
$\mathcal{C}(n)=-2(\log\delta)'(n)$ is non-decreasing iff $(\log\delta)'$ is non-increasing iff $\delta$ is log-concave on $\Fcal$. If
$\mathcal{C}$ is non-decreasing then $\alpha-\mathcal{C}$ is strictly decreasing (strictly decreasing $\alpha$ minus non-decreasing $\mathcal{C}$), hence has exactly one zero. Therefore $\log M$ is strictly quasi-concave on $\mathrm{int}(\Fcal)$. The interior maximizer $n^*$ of Step 3 is a stationary point; by uniqueness, it is the only one, so $n^*$ is the unique interior maximum.

\begin{lemma}[Uniqueness of the peak]
\label{lem:logconc}
If $\delta$ is log-concave on $\Fcal$, the stationary point of $\log M$ is unique, so $n^\ast$ is the only interior maximum. 
\end{lemma}

The ability $\beta(n)=p_{\max}\phi(g)$ is log-concave in $g$, and $g$ is affine in $n$, so $\beta$ is log-concave in $n$. Since $\log\beta=\log p_{\max}+\log\phi$ and for the adopted
\textsc{hill} gain $\phi(g)=g^{\kappa}/(g^{\kappa}+g_0^{\kappa})$,
\begin{equation}
\label{eq:logphi-hill}
(\log\phi)'(g) \;=\; \frac{\kappa\,g_0^{\kappa}}{g\bigl(g^{\kappa}+g_0^{\kappa}\bigr)},
\end{equation}
which is decreasing in $g$, so $\log\phi$ is concave in $g$ and hence in $n$ because $g$ is affine in $n$. The same holds for the power law, where $(\log\phi)' = \kappa/\bigl(g_0 u(u^\kappa-1)\bigr)$ with $u = 1+g/g_0$ is decreasing in $g$. For the bounded exponential, $\log(1-e^{-u})$ has second derivative $-1/\bigl(4\sinh^2(u/2)\bigr) < 0$, and $u = g/g_0$ is affine, so log-concavity in $g$ follows. The \textsc{Weibull} case needs the chain rule, since $u = (g/g_0)^\kappa$ is convex for the fitted $\kappa > 1$ (Table~\ref{tab:e1fit}).

The shift also preserves log-concavity, on the whole of $\Fcal$. Recall $\delta(n)=(1+c_w)\beta(n)-c_w$ from the definitions above, and drop the argument $n$. Since $c_w>0$ is a constant, $\delta'=(1+c_w)\beta'$ and $\delta''=(1+c_w)\beta''$ share $\beta''$'s sign. 
\begin{itemize}
\item If $\beta''(n)\le0$: on $\mathrm{int}(\Fcal)$ we have $\delta(n)>0$ (Step 1), so
$\delta\delta''\le0\le(\delta')^2$, giving $\delta\delta''-(\delta')^2\le0$ directly; log-concavity of
$\beta$ is not needed in this case.
\item If $\beta''(n)>0$: substituting the derivatives above,
\begin{align}
    \delta\delta''-(\delta')^2
&=(1+c_w)\Bigl[(1+c_w)\bigl(\beta\beta''-(\beta')^2\bigr)-c_w\beta''\Bigr]\\
&\le(1+c_w)\bigl[0-c_w\beta''\bigr]\le0,    
\end{align}
using $\beta\beta''-(\beta')^2\le0$ (log-concavity of $\beta$, shown above) and $c_w\beta''\ge0$.
\end{itemize}
The two cases exhaust every $n\in\mathrm{int}(\Fcal)$, so $\delta\delta''-(\delta')^2\le0$ throughout,
i.e.\ $\delta$ is log-concave on all of $\Fcal$. Lemma~\ref{lem:logconc} then applies unconditionally:
$\mathcal{C}$ is non-decreasing on $\Fcal$, so $\alpha-\mathcal{C}$ has exactly one zero and
$n^\ast$ is unique. The only inputs are log-concavity of $\beta$ (established above for every candidate in
Table~\ref{tab:forms}) and $c_w>0$.

\begin{corollary}[$n^\ast$ decreases in $\rho$]
\label{cor:nstar-rho}
Under Prop.~\ref{prop:single-peak} and Lemma~\ref{lem:logconc}, $n^\ast$ is strictly decreasing in $\rho$.
\end{corollary}
\begin{proof}
$n^\ast$ solves $\alpha(n)=\mathcal{C}(n)$, and $\mathcal{C}(n)=-2\delta'(n)/\delta(n)$ does not
depend on $\rho$. From Step 2,
\[
\frac{\partial\alpha}{\partial\rho}(n)
=\frac{\partial}{\partial\rho}\frac{1-\rho}{n(1+\rho(n-1))}
=\frac{-1}{\bigl(1+\rho(n-1)\bigr)^2}\;<\;0
\]
for every $n>1$, so $\alpha$ falls pointwise as $\rho$ rises while $\mathcal{C}$ is fixed.
$\alpha-\mathcal{C}$ is strictly decreasing in $n$ at the crossing (Step 4), so lowering the
$\alpha$ shifts the crossing left and $n^\ast$ strictly decreases in $\rho$.
\end{proof}
\noindent More correlated errors buy less from an extra vote at every group size, so the optimum sits
at a smaller group.

\subsection{The large-window limit}
\label{app:largew}

Write $\alpha(n) = n_{\mathrm{eff}}'/n_{\mathrm{eff}} = (1-\rho)/\bigl(n(1+\rho(n-1))\bigr)$
for the amplification term of A.1 Step 2. Since
$\delta'(n) = -(1+c_w)p_{\max} m\, \phi'\bigl(g(n)\bigr)$, the cost term depends on $n$ and
$W$ only through the room $g$,
\[
  C(n) \;=\; -2\frac{\delta'(n)}{\delta(n)}
        \;=\; \widehat C\bigl(g(n,W)\bigr),
  \qquad
  \widehat C(g) \;=\;
  \frac{2m(1+c_w)p_{\max}\,\phi'(g)}{(1+c_w)p_{\max}\phi(g)-c_w}.
\]

\begin{corollary}[Slow growth]
\label{cor:nstar}
Assume the conditions of Prop.~1 and Lemma~1, so that $C$ is non-decreasing in $n$ on $F$,
equivalently $\widehat C' \le 0$. Then on any window range where the maximizer stays
interior, $n^\ast(W)$ is differentiable, non-decreasing in $W$, and
\[
  0 \;\le\; \frac{dn^\ast}{dW}
    \;=\; \frac{1}{\,m + \alpha'(n^\ast)/\widehat C'(g)\,}
    \;\le\; \frac{1}{m} \;=\; \frac{dn_{\max}}{dW}.
\]
The optimal size therefore never grows faster than the feasibility cap, and the gap between
them widens as the window loosens.
\end{corollary}

\begin{proof}
The interior maximizer solves $F(n,W) = \alpha(n) - \widehat C\bigl(g(n,W)\bigr) = 0$ with
$g(n,W) = W - q - c - (n-1)m$, so $\partial g/\partial W = 1$ and $\partial g/\partial n =
-m$. Hence $\partial F/\partial W = -\widehat C'(g)$ and $\partial F/\partial n =
\alpha'(n) + m\widehat C'(g)$. By A.1 Step 2, $\alpha$ is strictly decreasing, so
$\alpha' < 0$; by hypothesis $\widehat C' \le 0$, so $\partial F/\partial n < 0$ and the
implicit function theorem applies. Then
\[
  \frac{dn^\ast}{dW}
  = -\frac{\partial F/\partial W}{\partial F/\partial n}
  = \frac{\widehat C'(g)}{\alpha'(n^\ast) + m\widehat C'(g)} .
\]
Numerator and denominator are both non-positive, giving $dn^\ast/dW \ge 0$. Dividing
through by $\widehat C'$ gives the stated form, and $\alpha'/\widehat C' \ge 0$ gives the
upper bound $1/m$. Since $n_{\max}(W) = 1 + (W-q-c)/m$, that bound is exactly
$dn_{\max}/dW$.
\end{proof}

\begin{corollary}[The large-window limit]
\label{cor:largew}
As $W \to \infty$, thinking room stops binding, $\delta(n) \to (1+c_w)p_{\max}-c_w$ is
constant in $n$, and $S(n)^2$ becomes monotone increasing. For $\rho > 0$ it saturates at
$\bigl((1+c_w)p_{\max}-c_w\bigr)^2/\rho$ as $n_{\mathrm{eff}} \to 1/\rho$, the classical
correlation ceiling. At $\rho = 0$ the errors are independent, $n_{\mathrm{eff}} = n$ is
unbounded, and no ceiling exists. An interior best size therefore exists only while the
window is scarce, and once it is not, the classical picture returns.
\end{corollary}

\begin{proof}
$g(n) \to \infty$ forces $\phi\bigl(g(n)\bigr) \to 1$, so $\delta' \to 0$ and
$\widehat C \to 0$. With $\alpha > 0$ everywhere, $\log M$ has no interior stationary point
and $M$ is monotone increasing on $F$, saturating at the stated ceiling as
$n_{\mathrm{eff}} \to 1/\rho$ when $\rho > 0$.
\end{proof}

\subsection{Self-consistency weakly dominates debate (Prop.~\ref{prop:illusion})}

\label{app:illusion}

$g_{\mathrm{SC}}=W-q-c\ge g(n)$ since $(n-1)m\ge0$; $\phi$ increasing gives
$\phi(g_{\mathrm{SC}})\ge\phi(g(n))$; multiplying by $p_{\max}$ gives the ability bound
$\betaeff(g_{\mathrm{SC}})\ge\betaeff(g(n))$, with equality iff $n{=}1$ (since $(n-1)m=0$ only at
$n{=}1$). The margin $\delta = (1+c_w)\beta_{\mathrm{eff}} - c_w$ is strictly increasing in
$\beta_{\mathrm{eff}}$ for $c_w \ge 0$, and $\Phi$ is increasing, so the head of Eq.~9 is
increasing in $\beta_{\mathrm{eff}}$ at fixed $(n,\rho,\sigma,c_w)$. The ability bound
therefore carries over to $A_{\mathrm{SC}}(n,W) \ge A^0_{\mathrm{deb}}(n,W)$. \hfill$\square$

% The bound is on the static, round-$0$ ability, so the $R$-round revision of
% \S\ref{sec:predict} can lift debate above it only where $g(n)\approx g_{\mathrm{SC}}$, which is exactly
% where the debate wins we observe are within seed noise (\S\ref{sec:main}). 
Note that $\Delta p_{\max}$ is a gain in the \emph{ceiling} from conditioning on peer content and not in generation length, so the
left side of Eq.~\ref{eq:info-vs-cost} must be bought outright; we report it in Appendix~\ref{app:scgap}.

\section{Implementation Details}
\label{app:protocol}

% ---

\begin{table}[t]
\centering
\caption{Benchmark details. All benchmarks use 500 samples, including the first 40\% used for calibration and the remaining 60\% for evaluation.}
\label{tab:datasets}
\small
\begin{tabular}{llrrl}
\toprule
\textbf{Benchmark} & \textbf{Task} & \textbf{Sample size} & \textbf{HuggingFace path} \\
\midrule
GSM8K & Grade-school math (open) & 500 & \texttt{openai/gsm8k} \\
MATH500 & Competition math (open) & 500 & \texttt{HuggingFaceH4/MATH-500} \\
MMLU-Pro & Multi-domain QA (10-way) & 500 & \texttt{TIGER-Lab/MMLU-Pro} \\
CSQA & Commonsense QA (5-way) & 500 & \texttt{tau/commonsense\_qa} \\
\bottomrule
\end{tabular}

\end{table}

% ---
\begin{table}[t]
\centering
\caption{Operating grid. Generation room is the leftover
$g=\mathrm{clip}(W-H-c,g_{\min},g_{\mathrm{cap}})$, so a larger group at fixed $W$ thinks less. The probe sweep runs to 4096. Deployment clips $g$ to $g_{\mathrm{cap}}= 2048$, and the out-of-range test in Table~\ref{tab:forms_mean} fits on caps up to $2048$ and predicts at $4096$. Table~\ref{tab:baseline}'s long-CoT comparison raises $g_{\mathrm{cap}}$ to $8192$ to give the single agent every advantage.}
\label{tab:protocol}
\setlength{\tabcolsep}{6pt}
\small
\begin{tabular}{ll}
\toprule
Parameter & Value \\
\midrule
Window $W$ (tokens)        & $\{1.2,1.6,2.0,2.5,3.2,4.5,6.5\}$k \\
Agents $n$                 & $\{1,\dots,8\}$ \\
Question length $q$ (mean, tokens) & GSM8K $128$, MATH500 $146$, MMLU-Pro $258$, CSQA $105$ \\
Per-message length $m$        & $320$ \\
Overhead $c$               & $64$ \\
Revision rounds $R$        & $1$ (so $R{+}1{=}2$ generation passes) \\
Generation room $g$        & $\mathrm{clip}(W-H-c,\,g_{\min},\,g_{\mathrm{cap}})$ \\
$g_{\min},\,g_{\mathrm{cap}}$      & $64,\ 2048$ \\
Probe caps $g$             & $\{64, 128, 256, 512, 1024, 2048, 4096\}$\\
Seeds                      & $\{42,43,44\}$ (pooled) \\
\bottomrule
\end{tabular}
\end{table}

We provide the configuration details for all hyperparameters in Table~\ref{tab:protocol}. All experiments are conducted in a single H100 (80G).

\subsection{Implementing TB-MAD}\label{app:our-implement}

\paragraph{Fitting $\betaeff$.} The three parameters $(p_{\max},g_0,\kappa)$ are fitted to the single-agent \emph{thinking} sweep by bound-constrained nonlinear least squares. Let ${(g_i,A_i,N_i)}$ denote the per-level mean accuracies $A_i$ at reasoning budget $g_i$, computed from $N_i$ samples. We minimize the sample-weighted sum of squared residuals:
\begin{align}
\label{eq:beff-fit}
(\hat p_{\max},\hat g_0,\hat\kappa)=\arg\min_{p_{\max},g_0,\kappa}\sum_i N_i\bigl(A_i-\betaeff(g_i)\bigr)^2,
\end{align}
where $\betaeff(g)=p_{\max}g^{\kappa}/(g^{\kappa}+g_0^{\kappa})$ (Eq.~\ref{eq:beff}), with $p_{\max}\in(0,1]$ and $g_0,\kappa>0$. Table~\ref{tab:e1fit} reports the fits with three goodness-of-fit measures across all settings. $R^2$ alone is uninformative, reaching $R^2\ge0.993$ in every cell; we therefore use RMSE to select between candidate forms. We report a 95\% CI on $\hat p_{\max}$ (Table~\ref{tab:e1fit}) from a percentile bootstrap: resample the sweep's per-level records with replacement, recompute pooled accuracy at each level, and refit, over $500$ refits per cell. Several sweep points sit numerically above the point estimate $\hat p_{\max}$, but none exceed the CI's upper bound in any of the twenty cells, so every such crossing is within sampling noise and no cell needs refitting.

\begin{table}[t]
\centering
\caption{Single-agent Hill fits on the cap sweep, all four tasks. $p_{\max}$, its 95\% bootstrap CI (resampling sweep records, $500$ refits), RMSE and LOO are on the accuracy scale, reported as percentages; LOO is the held-out RMSE, leaving out one swept cap at a time.}
\label{tab:e1fit}
\setlength{\tabcolsep}{6pt}
\small
\begin{tabular}{llccccccc}
\toprule
Model & Task & $p_{\max}$ & 95\% CI & $g_0$ & $\kappa$ & RMSE & LOO & $R^2$ \\
\midrule
\multirow{4}{*}{Qwen2.5-7B}
 & GSM8K & 87.4 & [86.8, 88.0] & 256 & 6.9 & 0.17 & 1.89 & 1.000 \\
 & MATH500 & 67.6 & [66.5, 68.9] & 433 & 3.7 & 0.21 & 0.27 & 1.000 \\
 & MMLU-Pro & 49.9 & [48.8, 51.1] & 437 & 3.7 & 0.37 & 1.88 & 1.000 \\
 & CSQA & 80.5 & [79.7, 81.4] & 246 & 9.8 & 0.09 & 8.32 & 1.000 \\
\midrule
\multirow{4}{*}{Falcon3-7B}
 & GSM8K & 81.0 & [80.4, 81.7] & 244 & 4.7 & 0.37 & 2.43 & 1.000 \\
 & MATH500 & 61.0 & [59.7, 62.4] & 475 & 3.5 & 0.49 & 1.68 & 1.000 \\
 & MMLU-Pro & 47.7 & [46.6, 49.0] & 414 & 2.8 & 0.65 & 1.17 & 0.999 \\
 & CSQA & 74.9 & [74.0, 75.8] & 168 & 3.0 & 1.76 & 4.03 & 0.996 \\
\midrule
\multirow{4}{*}{Llama-3.2-3B}
 & GSM8K & 69.0 & [68.2, 69.8] & 171 & 5.5 & 0.15 & 3.48 & 1.000 \\
 & MATH500 & 34.4 & [33.4, 35.6] & 222 & 3.5 & 0.14 & 0.34 & 1.000 \\
 & MMLU-Pro & 32.8 & [31.9, 33.8] & 325 & 3.4 & 0.80 & 2.10 & 0.997 \\
 & CSQA & 67.8 & [67.0, 68.7] & 201 & 6.2 & 0.17 & 4.53 & 1.000 \\
\midrule
\multirow{4}{*}{Gemma-2-9B}
 & GSM8K & 81.3 & [80.5, 82.1] & 144 & 5.0 & 0.36 & 3.77 & 1.000 \\
 & MATH500 & 40.2 & [39.1, 41.3] & 226 & 3.7 & 0.07 & 0.19 & 1.000 \\
 & MMLU-Pro & 50.4 & [49.2, 51.5] & 274 & 3.3 & 1.77 & 5.43 & 0.992 \\
 & CSQA & 77.9 & [77.0, 78.7] & 137 & 5.4 & 1.66 & 13.15 & 0.996 \\
\midrule
\multirow{4}{*}{Gemma-3-12B}
 & GSM8K & 92.9 & [92.2, 93.4] & 184 & 4.2 & 0.35 & 1.65 & 1.000 \\
 & MATH500 & 75.3 & [74.0, 77.1] & 459 & 2.4 & 0.41 & 0.87 & 1.000 \\
 & MMLU-Pro & 60.4 & [59.0, 61.9] & 478 & 2.9 & 1.02 & 2.55 & 0.998 \\
 & CSQA & 80.5 & [79.6, 81.4] & 176 & 8.6 & 0.74 & 12.41 & 1.000 \\
\bottomrule
\end{tabular}
\end{table}

\begin{table}[t]
\centering
\caption{Fitted Hill midpoint $\hat g_0$ against the measured median completion length, per model and task. Median lengths are read off the largest generation cap ($4096$ tokens), where under $1\%$ of samples are truncated. \emph{ratio} is $\hat g_0$ over the median.}
\label{tab:g0length}
\begin{tabular}{lcccccc}
\toprule
 & \multicolumn{3}{c}{GSM8K} & \multicolumn{3}{c}{MMLU-Pro} \\
\cmidrule(lr){2-4}\cmidrule(lr){5-7}
Model & $\hat g_0$ & median len. & ratio & $\hat g_0$ & median len. & ratio \\
\midrule
Qwen2.5-7B & 256 & 264 & 0.97 & 437 & 490 & 0.89 \\
Falcon3-7B & 244 & 254 & 0.96 & 414 & 484 & 0.85 \\
\midrule
\textbf{Mean ratio} & \multicolumn{3}{c}{0.96} & \multicolumn{3}{c}{0.87} \\
\bottomrule
\end{tabular}

\end{table}

\paragraph{Pooling the seeds.} Every $\betaeff$ is fitted on the pools from three different seeds (42, 43, 44), because $\phi$ is nonlinear in $(p_{\max},g_0,\kappa)$: averaging three per-seed parameter vectors gives a curve that is not the average of the three fitted curves. 
Refitting Eq.~\ref{eq:beff} separately for each seed changes $n^\ast$ by only 0.02 agents on average, while the regret of each seed-specific $n^\ast$ relative to the per-cell oracle is identical to two decimal places in 18 of 20 cells. The question set is fixed across seeds, so this is generation-sampling noise only.

% \paragraph{Paired bootstrap.} We resample \emph{questions} as clusters ($2000$ resamples), carrying all of a question's rows together, which propagates both the question-sampling noise and the seed-to-seed noise that a plain standard deviation over three seeds would report on its own, and we take a $95\%$ interval on the per-question difference.

\paragraph{Does $g_0$ correspond to a real length scale?} On Qwen2.5-7B and Falcon3-7B, the fitted Hill $g_0$ lands within $4\%$ of the median completion length on GSM8K and within $15\%$ on MMLU-Pro (Table~\ref{tab:g0length}).

% \paragraph{Windowing and banking.} $m{=}320$ is a budget-side design constant, one number held fixed % across models and tasks, not a length at which messages fit: pooled solution lengths are % $p_{50}{=}298$, $p_{75}{=}492$, and the fraction exceeding $m$ runs $14\%$ on GSM8K against $68\%$ on % MATH500 and $76\%$ on MMLU-Pro. The cap is enforced by truncation, and the budget is charged on the % \emph{measured} prompt rather than the idealized $q+(n-1)m$ of Eq.~\ref{eq:budget}, so $g=W-H-c$ % conserves the window whatever a message contains.

\begin{table}[t]
\centering
\caption{Seed-to-seed spread of the ability fit, of the resulting $n^\ast$, and of the regret that spread costs, as standard deviations over three per-seed refits. $\hat p_{\max}$ and regret are in percentage points, $\hat g_0$ in tokens, $\hat\kappa$ and $n^\ast$ dimensionless. Each per-seed fit uses that seed's full thinking sweep and its regret is scored against the per-cell oracle on the full bank. \textbf{Mean} regret is scored from a fit refit on a single held-out seed at a time, not the pooled (all-seed) fit the rest of the paper uses, so it is not directly comparable to the pooled-fit regret in Table~\ref{tab:cwsigma}; the near-zero \textbf{s.d.} shows that which seed is dropped does not matter, not that single-seed refitting matches the pooled result.}
% \caption{Seed-to-seed spread of the ability fit, of the resulting $n^\ast$, and of the regret that spread costs, as standard deviations over three per-seed refits. $\hat p_{\max}$ and regret are in percentage points, $\hat g_0$ in tokens, $\hat\kappa$ and $n^\ast$ dimensionless. Each per-seed fit uses that seed's full thinking sweep and its regret is scored against the per-cell oracle on the full bank.}
\label{tab:seedspread}\small
\resizebox{\textwidth}{!}{\begin{tabular}{ll ccc c cc}
\toprule
& & \multicolumn{3}{c}{s.d.\ over the three per-seed fits} & s.d.\ of $n^\ast$ & \multicolumn{2}{c}{regret (pp)} \\
\cmidrule(lr){3-5}\cmidrule(lr){6-6}\cmidrule(lr){7-8}
 \textbf{Model} & \textbf{Task} & \textbf{$\hat p_{\max}$} & \textbf{$\hat g_0$} & \textbf{$\hat\kappa$} & \textbf{mean} & \textbf{mean} & \textbf{s.d.} \\
\midrule
\multirow{4}{*}{Qwen2.5-7B} & GSM8K & 0.02 & 0.1 & 0.30 & 0.026 & 2.01 & 0.00 \\
 & MATH500 & 0.31 & 1.1 & 0.05 & 0.020 & 0.78 & 0.00 \\
 & MMLU-Pro & 0.22 & 3.7 & 0.05 & 0.023 & 1.09 & 0.00 \\
 & CSQA & 0.07 & 1.7 & 1.66 & 0.064 & 0.70 & 0.00 \\
\midrule
\multirow{4}{*}{Falcon3-7B} & GSM8K & 0.22 & 0.7 & 0.07 & 0.024 & 0.20 & 0.00 \\
 & MATH500 & 0.09 & 2.7 & 0.01 & 0.016 & 1.31 & 0.00 \\
 & MMLU-Pro & 0.56 & 5.4 & 0.04 & 0.004 & 1.20 & 0.00 \\
 & CSQA & 0.15 & 0.4 & 0.08 & 0.033 & 1.11 & 0.00 \\
\midrule
\multirow{4}{*}{Llama3.2-3B} & GSM8K & 0.22 & 1.1 & 0.14 & 0.016 & 0.11 & 0.00 \\
 & MATH500 & 0.19 & 1.8 & 0.14 & 0.064 & 2.08 & 0.75 \\
 & MMLU-Pro & 0.24 & 0.7 & 0.10 & 0.000 & 4.40 & 0.00 \\
 & CSQA & 0.10 & 0.9 & 0.09 & 0.008 & 0.52 & 0.00 \\
\midrule
\multirow{4}{*}{Gemma2-9B} & GSM8K & 0.05 & 0.7 & 0.07 & 0.008 & 0.20 & 0.00 \\
 & MATH500 & 0.12 & 2.6 & 0.06 & 0.044 & 0.81 & 0.00 \\
 & MMLU-Pro & 0.23 & 2.0 & 0.02 & 0.016 & 0.80 & 0.00 \\
 & CSQA & 0.13 & 0.2 & 0.17 & 0.016 & 0.51 & 0.00 \\
\midrule
\multirow{4}{*}{Gemma3-12B} & GSM8K & 0.12 & 0.6 & 0.04 & 0.020 & 0.30 & 0.01 \\
 & MATH500 & 0.04 & 2.3 & 0.02 & 0.008 & 0.33 & 0.00 \\
 & MMLU-Pro & 0.17 & 0.7 & 0.03 & 0.004 & 1.73 & 0.00 \\
 & CSQA & 0.03 & 1.1 & 0.15 & 0.004 & 6.34 & 0.00 \\
\midrule
\textbf{Mean} & & 0.16 & 1.5 & 0.16 & 0.021 & 1.33 & -- \\
\bottomrule
\end{tabular}
}
\end{table}

\paragraph{Estimating $\hat\rho$, and its invariance}\label{app:rhoest}
From the initial round's answers, we compute pairwise agreement of \emph{errors} across agents (same wrong answer) and map it to the equicorrelation $\rho$ in Eq.~\ref{eq:neff}. Eq.~\ref{eq:neff} assumes that $\rho$ does not vary with either $n$ or $W$. We verify this empirically (Table~\ref{tab:rhow}): (1) \emph{Constant in $n$:} re-estimating $\hat\rho$ at each group size changes it by at most $\mathrm{range}_n\le0.05$, so $\rho$ does not grow with the group; and (2) \emph{Constant in $W$:} re-estimating $\hat\rho$ at each window changes it by at most $\mathrm{range}_W\le0.05$ across all twenty model-task cells. Error correlation appears to be a property of the model and task, rather than of the budget on these cells.

\begin{table}[t]
\centering
\caption{Round-0 error correlation $\hat\rho$ re-estimated at each window. range$_W$ and range$_n$ are $\max\hat\rho-\min\hat\rho$ over the swept windows and over the group sizes.}
\label{tab:rhow}\small
\begin{tabular}{llccccccccc}
\toprule
\multicolumn{2}{c}{} & \multicolumn{7}{c}{$\hat\rho$ at $W$ ($\times10^3$ tokens)} & & \\
\cmidrule(lr){3-9}
Model & Task & 1.2 & 1.6 & 2.0 & 2.5 & 3.2 & 4.5 & 6.5 & range$_W$ & range$_n$ \\
\midrule
\multirow{4}{*}{Qwen2.5-7B} & GSM8K & 0.36 & 0.35 & 0.36 & 0.35 & 0.34 & 0.35 & 0.35 & 0.01 & 0.02 \\
 & MATH500 & 0.71 & 0.69 & 0.68 & 0.68 & 0.68 & 0.68 & 0.68 & 0.04 & 0.01 \\
 & MMLU-Pro & 0.60 & 0.57 & 0.57 & 0.56 & 0.56 & 0.56 & 0.56 & 0.04 & 0.02 \\
 & CSQA & 0.58 & 0.58 & 0.58 & 0.58 & 0.58 & 0.58 & 0.58 & 0.01 & 0.03 \\
\midrule
\multirow{4}{*}{Falcon3-7B} & GSM8K & 0.41 & 0.41 & 0.40 & 0.40 & 0.41 & 0.41 & 0.40 & 0.01 & 0.01 \\
 & MATH500 & 0.68 & 0.65 & 0.64 & 0.64 & 0.65 & 0.64 & 0.64 & 0.05 & 0.01 \\
 & MMLU-Pro & 0.59 & 0.56 & 0.57 & 0.56 & 0.56 & 0.56 & 0.55 & 0.04 & 0.01 \\
 & CSQA & 0.51 & 0.50 & 0.50 & 0.50 & 0.50 & 0.50 & 0.50 & 0.01 & 0.02 \\
\midrule
\multirow{4}{*}{Llama3.2-3B} & GSM8K & 0.41 & 0.41 & 0.40 & 0.41 & 0.42 & 0.42 & 0.42 & 0.01 & 0.02 \\
 & MATH500 & 0.52 & 0.52 & 0.52 & 0.53 & 0.52 & 0.52 & 0.52 & 0.01 & 0.02 \\
 & MMLU-Pro & 0.44 & 0.43 & 0.43 & 0.43 & 0.42 & 0.43 & 0.44 & 0.02 & 0.00 \\
 & CSQA & 0.42 & 0.42 & 0.42 & 0.42 & 0.42 & 0.42 & 0.43 & 0.01 & 0.02 \\
\midrule
\multirow{4}{*}{Gemma2-9B} & GSM8K & 0.57 & 0.58 & 0.58 & 0.58 & 0.57 & 0.57 & 0.57 & 0.01 & 0.01 \\
 & MATH500 & 0.64 & 0.64 & 0.65 & 0.64 & 0.64 & 0.64 & 0.65 & 0.01 & 0.02 \\
 & MMLU-Pro & 0.66 & 0.66 & 0.66 & 0.66 & 0.67 & 0.66 & 0.66 & 0.01 & 0.01 \\
 & CSQA & 0.57 & 0.59 & 0.58 & 0.58 & 0.57 & 0.57 & 0.58 & 0.02 & 0.02 \\
\midrule
\multirow{4}{*}{Gemma3-12B} & GSM8K & 0.65 & 0.65 & 0.66 & 0.66 & 0.65 & 0.66 & 0.65 & 0.01 & 0.04 \\
 & MATH500 & 0.77 & 0.80 & 0.78 & 0.77 & 0.77 & 0.78 & 0.77 & 0.02 & 0.02 \\
 & MMLU-Pro & 0.65 & 0.64 & 0.64 & 0.63 & 0.63 & 0.63 & 0.63 & 0.02 & 0.01 \\
 & CSQA & 0.55 & 0.54 & 0.55 & 0.55 & 0.55 & 0.55 & 0.54 & 0.01 & 0.04 \\
\bottomrule
\end{tabular}

\end{table}

\subsection{Implementing Baselines}
\label{app:baseline-implement}

The five algorithms are collected in Alg.~\ref{alg:predict}--\ref{alg:baselinemix}; $\mathcal{D}$ is
the calibration sample and $g(n,W)$ the room the window leaves.

\emph{Calibrated on measured debate cells:}
\begin{itemize}[label=$\bullet$]\itemsep2pt
\item \textbf{Empirical argmax} measures every $(n,W)$ cell on $\mathcal{D}$ and selects the size with
the highest $\hat A(n,W)$ (Alg.~\ref{alg:baselinegrid}).
\item \textbf{Successive halving}~\citep{jamieson2016halving} evaluates all sizes on a small sample,
keeps the better half, and doubles the sample until one size remains or the budget is exhausted
(Alg.~\ref{alg:baselinehalving}).
\item \textbf{Quadratic surface} fits a least-squares quadratic in $\{1,n,W,n^2,W^2,nW\}$ to the
measured cells and selects its best predicted size per window (Alg.~\ref{alg:baselinegrid}).
\end{itemize}
\emph{Calibrated on single-agent probes only, as ours is:}
\begin{itemize}[label=$\bullet$]\itemsep2pt
\item \textbf{Learned $\betaeff$} uses our aggregator with a gradient-boosted-tree estimate
$\hat\beta_{\mathrm{gbt}}$ fitted to the $|\mathcal{G}|{=}7$ probe caps, with no parametric shape
(Alg.~\ref{alg:baselinesa}).
\item \textbf{Star exponent}~\citep{liu2026phase} uses our aggregator, $\hat\rho$, and budget model,
replacing $\phi$ with $p_{\max}\bigl(1-(1+g/g_0)^{-\kappa}\bigr)$, the \textsc{pow} form of
Table~\ref{tab:forms} (Alg.~\ref{alg:baselinesa}).
\item \textbf{Difficulty mixture}~\citep{chen2024morecalls} estimates per-question accuracy $p_i(g)$
from the probe sweep and simulates votes. The original fixes the per-call allowance, so $p_i$ cannot
vary with $n$; we instead evaluate $p_i$ at the remaining room $g(n,W)$, interpolated in $\log g$, and
add our correlation $\hat\rho$ and tie penalty $c_w$ (Alg.~\ref{alg:baselinemix}).
\end{itemize}
\emph{No calibration:} \textbf{naive $n{=}1$} and \textbf{naive $n{=}n_{\max}$} always use one agent
or the largest group; neither reads data, so neither has an algorithm.

\emph{Ours:} \textbf{TB-MAD} fits $(p_{\max},g_0,\kappa)$ on the probe sweep and evaluates
$p_{\max}\phi(g;g_0,\kappa)$ at the remaining room for each size (Alg.~\ref{alg:predict}).

\subsection{Calibration Cost}
\label{app:calibration}

TB-MAD, Star exponent, and learned $\betaeff$ estimate $\rho$ and $c_w$ from data rather than treating them as free constants. Each calibration question requires at least three independent round-$0$ answers, since no peer history is available yet and three samples are the minimum needed for the correlation estimate. We charge each the token cost of measuring them, $0.1\times$ a single probe question (Table~\ref{tab:calibtokens}, so all baselines are matched to the same calibration spend throughout (Figure~\ref{fig:kscaling}, right panel; Appendix Table~\ref{tab:rate}).
Table~\ref{tab:calibtime} additionally reports the wall-clock cost of calibration, comparing our single-agent probe with the full $(n,W)$ grid used by the Empirical Argmax debate baseline.

\begin{table}[t]
\centering
\caption{Fitting performance of thinking gain $\phi$, over five models and four datasets, per cell in Table~\ref{tab:e1fit}, Appendix~\ref{app:fitting}. RMSE is on the accuracy scale in percent; leave-one-out holds out one swept cap
at a time and out-of-range predicts at $4096$ from caps up to $2048$. Adopted form in bold.}
\label{tab:forms_mean}
\begin{tabular}{lccccc}
\toprule
Selection criterion & \textbf{\textsc{hill}} & \textsc{weib} & 
% \textsc{be} & 
\textsc{pow} & \textsc{burr} & iso \\
\midrule
Leave-one-out RMSE & \textbf{3.61} & 5.21 
% & 11.41 
& 11.74 & 4.49 & 10.02 \\
Out-of-range RMSE & 0.38 & 0.75 & 
% 5.69 & 
6.85 & 0.33 & \textbf{0.31} \\
\midrule
\textbf{Total} & \textbf{3.99} & 5.96 & 
% 17.10 & 
18.59 & 4.82 & 10.33 \\
\bottomrule
\end{tabular}

\end{table}
\label{app:fitting}

\begin{table}[t]
\centering
\caption{Calibration wall-clock in minutes, $200$ questions. Probe is our single-agent sweep; debate (empirical argmax) is the
full $(n,W)$ grid. \emph{ratio} is debate over probe per dataset, then averaged over the four.}
\label{tab:calibtime}
\resizebox{\textwidth}{!}{%
\begin{tabular}{lrrrrrrrrr}
\toprule
& \multicolumn{2}{c}{GSM8K} & \multicolumn{2}{c}{MATH500} & \multicolumn{2}{c}{CSQA} & \multicolumn{2}{c}{MMLU-Pro} & \\
\cmidrule(lr){2-3}\cmidrule(lr){4-5}\cmidrule(lr){6-7}\cmidrule(lr){8-9}
Model & probe & debate & probe & debate & probe & debate & probe & debate & ratio \\
\midrule
Llama3.2-3B & 6 & 21 & 11 & 34 & 6 & 38 & 11 & 71 & 5$\times$ \\
Qwen2.5-7B & 8 & 25 & 13 & 43 & 7 & 55 & 13 & 93 & 5$\times$ \\
Falcon3-7B & 8 & 31 & 17 & 54 & 6 & 57 & 14 & 105 & 6$\times$ \\
Gemma2-9B & 7 & 29 & 13 & 62 & 6 & 68 & 13 & 147 & 8$\times$ \\
Gemma3-12B & 12 & 62 & 34 & 146 & 10 & 171 & 31 & 336 & 9$\times$ \\
\bottomrule
\end{tabular}}

\end{table}

\begin{table}[htbp]
\centering
\caption{Input and output tokens each calibration rule spends, as a multiple of the single-agent
probe sweep, summed over the twenty cells at the full $200$-question pool. The probe (except Difficulty Mixture) multiple is
$1\times$ by definition; debate re-reads every peer message, so its input multiple is the largest.}
\label{tab:calibtokens}
\begin{tabular}{lrr}
\toprule
\textbf{Calibration rule} & \textbf{Input} & \textbf{Output} \\
\midrule
Empirical Argmax & $49\times$ & $11\times$ \\
Quadratic Surface & $49\times$ & $11\times$ \\
Successive Halving & $33\times$ & $9\times$ \\
Difficulty Mixture & $0.9\times$ & $0.8\times$ \\
Round-0 probe ($\rho$, $c_w$) & $0.1\times$ & $0.1\times$ \\
\midrule
Probe sweep (TB-MAD, learned $\betaeff$, Star Exponent) & $1\times$ & $1\times$ \\
\bottomrule
\end{tabular}

\end{table}
% ---
\section{Generalization to heterogeneous settings}
\label{app:hetero}

We extend our method to heterogeneous model settings, where agents have different abilities $\betaeff$. To isolate the effect of model heterogeneity, we consider a \emph{uniform allocation} setting in which the $K$ models are assigned the same number of agents, while every agent keeps the same room $g(n)$ and each model $i$ has fixed parameters $(p_{\max,i},g_{0,i},\kappa_i)$. The goal is the optimal number of agents across the $K$ models under a total budget of $n$ agents, with $n/K$ agents assigned to each model. 

\subsection{Extended method}
\label{app:hetero-method}

\paragraph{Per-model edges and candidate subsets.}
Under uniform allocation, every agent receives
\begin{align}
g(n)=\mathrm{clip}\bigl(W-q-c-(n-1)m,\;g_{\min},\;g_{\mathrm{cap}}\bigr),
\end{align}
so heterogeneity enters through the model-specific edge
\begin{align}
\label{eq:delta-j-hetero}
\mathbb{E}[X_j]=\delta_j(n)=(1+c_w)\,p_{\max,j}\,\phi_j\bigl(g(n)\bigr)-c_w .
\end{align}
A negative edge indicates a model that systematically points toward the wrong answer, so mixing it
with positive-edge models can dilute the group edge. We therefore allow the $n$ agents to draw either from all $K$ models or from any positive-edge singleton when the fitted edges have mixed signs:
\begin{align}
\label{eq:bardelta}
\bar\delta_S(n)=\frac{1}{|S|}\sum_{i\in S}\delta_i(n),
\end{align}
with candidate set
\begin{align}
\label{eq:hetero-candidates}
\mathcal{S}(n)=
\begin{cases}
\{\{1,\dots,K\}\}\cup\{\{i\}:\delta_i(n)>0\},
& \min_i\delta_i(n)\le 0<\max_i\delta_i(n),\\[2pt]
\{\{1,\dots,K\}\},
& \text{otherwise}.
\end{cases}
\end{align}
Thus $|S|=1$ is simply a homogeneous group from model $i$, whereas the full set remains available at every $n$.

\paragraph{Correlation becomes block structure.}
Let $\hat\rho_w$ and $\hat\rho_b$ denote round-$0$ error correlations for agents backed by the same
and different models, respectively. For a subset $S$ with equal blocks of $n/|S|$ agents,
\begin{align}
\label{eq:neff-hetero}
n_{\mathrm{eff}}(n,|S|)
=\frac{n}{1+\hat\rho_w(n/|S|-1)+\hat\rho_b(n-n/|S|)}.
\end{align}
This reduces to the homogeneous $n/(1+\hat\rho_w(n-1))$ when $|S|=1$ and, as $n\to\infty$,
approaches $|S|/(\hat\rho_w+(|S|-1)\hat\rho_b)$. Hence mixing raises the asymptotic ceiling whenever
$\hat\rho_b<\hat\rho_w$, i.e., cross-model errors are less correlated.

\paragraph{Accuracy and operating point.}
With common variance $\sigma^2$, $\operatorname{Var}(\bar X_n)=\sigma^2/n_{\mathrm{eff}}(n,|S|)$,
giving
\begin{align}
\label{eq:accuracy-hetero}
A(n,S)
=\Phi\!\left(\bar\delta_S(n)\sqrt{n_{\mathrm{eff}}(n,|S|)}\big/\sigma\right).
\end{align}
Since $\Phi$ is monotone, maximizing accuracy is equivalent to
\begin{align}
\label{eq:hetero-argmax}
\mathcal{V}(n,S)
=\operatorname{sign}\bigl(\bar\delta_S(n)\bigr)\,n_{\mathrm{eff}}(n,|S|)\,\bar\delta_S(n)^2 .
\end{align}
We preserve the sign of the margin to avoid amplifying a negative-edge subset, consistent with the treatment of $S(n)$ in $\S$~\ref{sec:synergy}.
The edges are obtained from single-agent probes, while only $\hat\rho_w,\hat\rho_b$ require round-$0$ measurements (independent answers before any exchange, not the full multi-round debate); calibrating all $K$ models therefore costs roughly $K$ times a set-blind probe rule. As in the homogeneous setting (\S\ref{sec:main}), we charge this round-$0$ probe cost against the calibration budget of every rule that reads $\hat\rho_w,\hat\rho_b$ from the debate bank rather than treating them as free constants, summed once per distinct model in the model set; Tables~\ref{tab:heterobaselines}
and~\ref{tab:heterobk} already reflect it.

\paragraph{The selection gate.}
Algorithm~\ref{alg:predict-hetero} evaluates $\mathcal{V}(n,S)$ over the swept grid and
$\mathcal{S}(n)$, then selects
\begin{align}
\label{eq:hetero-gate}
(\hat n^\ast,\hat S)
\in
\arg\max_{n\in\{K,2K,\dots,n_{\max}\},\;S\in\mathcal{S}(n)}
\mathcal{V}(n,S).
\end{align}
The gate admits singletons only when fitted edges disagree in sign; otherwise it reduces to the
homogeneous rule. Allowing singletons at every $n$ would simply favor the strongest model whenever $\bar\delta$ dominates, with little corresponding change in $n_{\mathrm{eff}}$, which explains the
weaker performance of \emph{TB-MAD-hetero w/o gate} in Tables~\ref{tab:heterobaselines}
and~\ref{tab:heterobk}. 
% Note that calibrating $K$ models costs $K$ probe sweeps. We charge that against the same $1.0M$ and $2.0M$ budgets, so TB-MAD-hetero affords $K$ times fewer calibration questions than the homogeneous rule at the same spend.

\subsection{Updated theory analysis}
\label{app:hetero-theory}

\begin{proposition}[Heterogeneous aggregation head]
\label{prop:hetero-head}
Under uniform allocation into $|S|$ equal blocks, a common $\sigma$, within/between round-$0$
correlations $(\rho_w,\rho_b)$, and the $\pm1$ scoring of \S\ref{sec:synergy}, the group accuracy is
\begin{align*}
A(n) = \Phi\!\left(\bar\delta_S(n)\,\sqrt{n_{\mathrm{eff}}(n,|S|)}\big/\sigma\right),
\end{align*}
with $n_{\mathrm{eff}}$ as in Eq.~\ref{eq:neff-hetero}. Whenever $\bar\delta_S(n)>0$,
$\arg\max_n A(n)=\arg\max_n n_{\mathrm{eff}}(n,|S|)\,\bar\delta_S(n)^2$, and every step of
\S\ref{sec:theory} that uses $\delta$ only through $A(n)$ holds verbatim with $\delta\to\bar\delta_S$.
\end{proposition}
\begin{proof}[Proof sketch]
The covariance of $(X_1,\dots,X_n)$ has constant diagonal $\sigma^2$, constant within-block
off-diagonal $\rho_w\sigma^2$, and constant cross-block off-diagonal $\rho_b\sigma^2$. Summing gives
$\operatorname{Var}(\bar X_n)=n^{-2}\bigl[n\sigma^2+(n/|S|-1)n\rho_w\sigma^2+(n-n/|S|)n\rho_b\sigma^2\bigr]
=\sigma^2/n_{\mathrm{eff}}(n,|S|)$. The normal approximation of $\bar X_n$ with mean $\bar\delta_S$
and this variance gives $A(n)$; monotonicity of $\Phi$ and $\sigma>0$ give the $\arg\max$ identity.
\end{proof}

We do not have experiments isolating the two results below, so we only note whether each still holds
and under which conditions.

\paragraph{Prop.~\ref{prop:illusion} (self-consistency dominates the initial-round vote).} Still
holds on a model set when no member's ceiling is raised by peer text, $\Delta p_{\max,i}\le 0$ for
all $i$: $g_{\mathrm{SC}}\ge g(n)$ and each $\phi_i$ increasing make $\bar\delta$ larger under
self-consistency, so the dominance carries over with no new assumption. It can fail once some member
is weaker than a peer that sends it information its own model lacks, $\Delta p_{\max,i}>0$. The
crossing condition is then
\begin{align}
\label{eq:info-vs-cost-hetero}
\sum_i \Delta p_{\max,i}\,\phi_i\bigl(g(n)\bigr) \;>\; \sum_i p_{\max,i}\bigl(\phi_i(g_{\mathrm{SC}})-\phi_i(g(n))\bigr),
\end{align}
the homogeneous Eq.~\ref{eq:info-vs-cost} summed over members.

\paragraph{The interior optimum.}
The result extends under the stated restrictions. Existence (Step 3 of Appendix~\ref{app:single-peak}) requires only that $\bar\delta_S$ be continuous, strictly decreasing, and cross zero once, all of which continue to hold here. In particular, even when $g_{0,i}$ and $\kappa_i$ vary across members, $\bar\phi=|S|^{-1}\sum_{i\in S}\phi_i$ remains strictly increasing as a positive average of strictly increasing functions, so the feasibility boundary $\bar\delta_S{=}0$ has a unique crossing in $g$. Local uniqueness at this crossing also carries over without assuming log-concavity of $\bar\phi$. Whenever $\bar\delta_S\to0^{+}$ with $\bar\delta_S''\le0$, the boundary argument in Appendix~\ref{app:single-peak} applies unchanged, since it relies only on positivity of $\bar\delta_S$ and the sign of its second derivative. Away from the boundary, however, log-concavity of $\bar\phi$ need not hold for a heterogeneous average and is not assumed here.

\subsection{Results}
\label{app:hetero-results}
We evaluate on the model sets below, spanning members of similar and of differing capability, on GSM8K, MATH500, MMLU-Pro and CSQA as in the homogeneous setting. Agent $i$ speaks as member $i\bmod K$, so a set forms $K$ equal blocks only at $n$ divisible by $K$, which fixes each grid.
\begin{itemize}[label=$\bullet$]\itemsep2pt
    \item $K{=}2$, $n\in\{2,4,6,8\}$: Qwen2.5-7B+Gemma3-12B, Qwen2.5-7B+Llama3.2-3B, Falcon3-7B+Gemma2-9B, Gemma3-12B+Llama3.2-3B.
    \item $K{=}3$, $n\in\{3,6\}$: Qwen2.5-7B+Gemma3-12B+Llama3.2-3B, Qwen2.5-7B+Falcon3-7B+Gemma2-9B.
    \item $K{=}4$, $n\in\{4,8\}$: Qwen2.5-7B+Falcon3-7B+Gemma2-9B+Gemma3-12B, Qwen2.5-7B+Falcon3-7B+Gemma3-12B+Llama3.2-3B.
\end{itemize}
Per-member $(p_{\max,i},g_{0,i},\kappa_i)$ come from the single-agent sweeps and $\hat\rho_w,\hat\rho_b$ from round-$0$ answers, exactly as on the homogeneous cells.

Tables~\ref{tab:heterobaselines} and~\ref{tab:heterobk} report held-out regret against the per-cell oracle when each method picks $(n,S)$, against the eight baselines of Table~\ref{tab:matchedcost}.
Probe-calibrated methods are the strongest at every $K$. At $K{=}2$ our method is second best, just behind learned $\betaeff$. For $K>2$ the star exponent matches it: as $K$ grows the group ability averages $K$ member curves, so the curve shape that separates the two fits at $K{=}2$ washes out. Removing the gate is worse on average at both $K$ ($1.49\to1.54$ on the pairs, $1.66\to2.75$ on the larger sets), which confirms the design choice.

\begin{table}[t]
\centering
\caption{Held-out regret against the per-cell oracle (percentage points) when each method chooses $n$ on the heterogeneous model-set bank at a matched \emph{token} budget, the model-set analogue of
Table~\ref{tab:matchedcost} with the same marks. Per-cell values in
Tables~\ref{tab:heterobaselinespercell}--\ref{tab:heterobaselinespercell2}.}
\label{tab:heterobaselines}
\small
\resizebox{\textwidth}{!}{\begin{tabular}{@{}lcccccccc@{\hspace{11pt}}c@{}}
\toprule
\multirow{2}{*}[-3pt]{\textbf{Method}} & \multicolumn{4}{c}{\textbf{1.0M budget}} & \multicolumn{4}{c}{\textbf{2.0M budget}} & \multirow{2}{*}[-3pt]{\textbf{Avg.}} \\
\cmidrule(lr){2-5}\cmidrule(lr){6-9}
 & Q+G3 & Q+L & F+G2 & G3+L & Q+G3 & Q+L & F+G2 & G3+L &  \\
\midrule
\multicolumn{10}{@{}l}{\emph{Debate-calibrated}} \\
Successive Halving & 8.96 & 10.88 & 7.97 & 7.47 & 4.75 & 6.87 & 4.82 & 5.50 & 7.15 \\
Quadratic Surface & 2.61 & 7.01 & 2.68 & 2.63 & 4.60 & 6.62 & 2.79 & 4.05 & 4.13 \\
Empirical Argmax & 2.93 & 7.09 & 3.30 & 3.07 & 2.96 & 5.18 & 3.64 & 3.19 & 3.92 \\
\midrule
\multicolumn{10}{@{}l}{\emph{No calibration}} \\
Naive $n{=}n_{\max}$ & 30.82 & 27.81 & 17.81 & 22.85 & 30.82 & 27.81 & 17.81 & 22.85 & 24.82 \\
Naive $n{=}2$ & 3.66 & 6.89 & 3.70 & 3.40 & 3.66 & 6.89 & 3.70 & 3.40 & 4.41 \\
\midrule
\multicolumn{10}{@{}l}{\emph{Probe-calibrated}} \\
Difficulty Mixture & 2.45 & 3.77 & 2.60 & 2.19 & 2.81 & 4.34 & 2.67 & 2.09 & 2.86 \\
Star Exponent & 1.92 & 3.66 & \underline{1.95} & 1.60 & 1.31 & 3.30 & 1.96 & 1.50 & 2.15 \\
Learned $\betaeff$ & 1.56 & 2.68 & \textbf{1.66} & \textbf{1.30} & \underline{1.13} & 2.35 & \textbf{1.58} & \underline{1.31} & 1.70 \\
TB-MAD-hetero w/o gate & \underline{1.33} & \textbf{0.82} & 2.06 & 2.41 & 1.48 & \textbf{1.12} & \underline{1.76} & 1.37 & \underline{1.54} \\
\rowcolor{green!15} \textbf{TB-MAD-hetero} (Ours) & \textbf{0.97} & \underline{1.87} & 2.05 & \underline{1.34} & \textbf{0.81} & \underline{2.07} & 1.95 & \textbf{0.85} & \textbf{1.49} \\
\bottomrule
\end{tabular}
}
\end{table}

\begin{table}[t]
\centering
\caption{Held-out regret (percentage points) on the balanced-$K$ sets, two $K{=}3$ sets on grid
$\{3,6\}$ and two $K{=}4$ on $\{4,8\}$, at the budgets of Table~\ref{tab:matchedcost} with the same
marks. Each naive column is a fixed $n$, blank where a set's grid omits it; a row spanning one $K$
gets no average. Per-cell values in Tables~\ref{tab:heterobkpercell}--\ref{tab:heterobkpercell2}.}
\label{tab:heterobk}
\small
\resizebox{\textwidth}{!}{\begin{tabular}{@{}lcccccccc@{\hspace{11pt}}c@{}}
\toprule
\multirow{2}{*}[-3pt]{\textbf{Method}} & \multicolumn{4}{c}{\textbf{1.0M budget}} & \multicolumn{4}{c}{\textbf{2.0M budget}} & \multirow{2}{*}[-3pt]{\textbf{Avg.}} \\
\cmidrule(lr){2-5}\cmidrule(lr){6-9}
 & Q+G3+L & Q+F+G2 & Q+F+G2+G3 & Q+F+G3+L & Q+G3+L & Q+F+G2 & Q+F+G2+G3 & Q+F+G3+L &  \\
\midrule
\multicolumn{10}{@{}l}{\emph{Debate-calibrated}} \\
Successive Halving & 2.97 & 3.09 & 4.48 & 3.66 & 2.41 & 1.68 & 2.77 & 2.84 & 2.99 \\
Quadratic Surface & 2.58 & 2.41 & 1.81 & 1.64 & 1.90 & 1.81 & 1.91 & 1.71 & 1.97 \\
Empirical Argmax & 2.14 & 1.88 & 2.39 & 2.16 & 1.75 & 1.79 & 2.02 & 1.73 & 1.98 \\
\midrule
\multicolumn{10}{@{}l}{\emph{No calibration}} \\
Naive $n{=}8$ & -- & -- & 20.72 & 19.46 & -- & -- & 20.72 & 19.46 & -- \\
Naive $n{=}6$ & 14.65 & 14.05 & -- & -- & 14.65 & 14.05 & -- & -- & -- \\
Naive $n{=}3$ & 2.88 & 2.10 & -- & -- & 2.88 & 2.10 & -- & -- & -- \\
Naive $n{=}4$ & -- & -- & 2.04 & 1.88 & -- & -- & 2.04 & 1.88 & -- \\
\midrule
\multicolumn{10}{@{}l}{\emph{Probe-calibrated}} \\
Difficulty Mixture & 2.60 & 1.86 & 1.88 & 1.44 & 2.36 & 1.79 & 1.75 & 1.54 & 1.90 \\
Star Exponent & 1.64 & 1.60 & \underline{1.77} & \underline{1.35} & 1.64 & 1.60 & \underline{1.70} & \underline{1.34} & \underline{1.58} \\
Learned $\betaeff$ & \textbf{1.23} & \textbf{1.49} & \textbf{1.61} & \textbf{1.26} & \underline{1.22} & \textbf{1.34} & \textbf{1.54} & \textbf{1.21} & \textbf{1.36} \\
TB-MAD-hetero w/o gate & 3.77 & 2.33 & 3.53 & 3.01 & 1.91 & 2.47 & 2.88 & 2.10 & 2.75 \\
\rowcolor{green!15} \textbf{TB-MAD-hetero} (Ours) & \underline{1.26} & \underline{1.50} & 2.72 & 1.90 & \textbf{0.99} & \underline{1.54} & 2.05 & \underline{1.34} & 1.66 \\
\bottomrule
\end{tabular}
}
\end{table}

% ---

\section{Additional Results}
\label{app:extra}

\subsection{Self-refinement, full results}
\label{app:selfrefine}

\begin{figure*}[t]
\centering
\includegraphics[width=\textwidth]{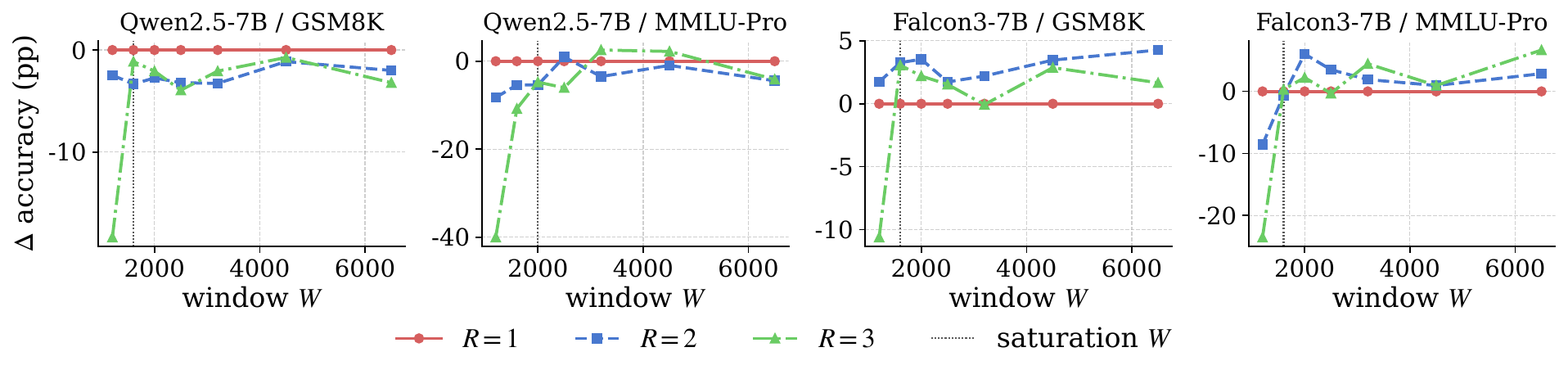}
\caption{Self-refinement accuracy against the first draft ($\Delta{=}A_R-A_1$, in points) as the window grows, $R\in\{1,2,3\}$; $R{=}1$ is the base generation without refinement. 
The dotted vertical line marks the saturation window: from there on, both $R{=}2$ and $R{=}3$ stay within $7\%$ of the first draft.}
\label{fig:generality}
\end{figure*}

Figure~\ref{fig:generality} shows the same generation-floor pattern across all cells: at the smallest window, increasing from $R{=}1$ to $R{=}3$ reduces accuracy by $1.3$--$44.3$ points, with the loss recovering as $W$ grows. More refinement steps require a larger window before they become beneficial. 

The weakest model (Llama3.2-3B) gains almost nothing ($\Delta\approx0$), consistent with limited self-correction ability without external feedback~\citep{huang2024selfcorrect}, while ceiling models have little headroom (Qwen/Gemma-3 on GSM8K). A genuine interior $R^\ast$ appears mainly for mid-capability models with both headroom and self-correction ability (Falcon3, Gemma-2, and most models on MATH500), but the resulting gains remain small relative to the additional token cost (Table~\ref{tab:selfrefine}).

\begin{table}[t]
\centering
\caption{Self-refine over $R\in\{1,2,3\}$ rounds across settings. $A_1$: first-draft accuracy; gain$^\ast$: oracle (grid-best) refine gain over $A_1$; $R^\ast$: the round most often best across the swept windows; drop: $A_1$ minus the $R{=}3$ accuracy at the smallest window.}
\label{tab:selfrefine}
\setlength{\tabcolsep}{3pt}\small
\begin{tabular}{@{}l cccc cccc cccc cccc@{}}
\toprule
 & \multicolumn{4}{c}{GSM8K} & \multicolumn{4}{c}{MATH500} & \multicolumn{4}{c}{MMLU-Pro} & \multicolumn{4}{c}{CSQA}\\
\cmidrule(lr){2-5}\cmidrule(lr){6-9}\cmidrule(lr){10-13}\cmidrule(lr){14-17}
Model & $A_1$ & gain$^\ast$ & $R^\ast$ & drop & $A_1$ & gain$^\ast$ & $R^\ast$ & drop & $A_1$ & gain$^\ast$ & $R^\ast$ & drop & $A_1$ & gain$^\ast$ & $R^\ast$ & drop\\
\midrule
Qwen2.5-7B & 87.6 & 0.0 & 1 & 18.3 & 64.7 & 1.3 & 1 & 44.3 & 47.0 & 2.5 & 1 & 40.0 & 81.4 & 1.1 & 1 & 4.7\\
Falcon3-7B & 81.3 & 4.3 & 2 & 10.6 & 55.1 & 3.5 & 1 & 31.3 & 48.3 & \textbf{6.7} & 3 & 23.5 & 74.1 & 0.7 & 2 & 1.3\\
Gemma-2-9B & 81.7 & 2.0 & 2 & 1.6 & 38.2 & 3.3 & 2 & 3.8 & 50.5 & 3.2 & 1 & 13.7 & 77.1 & 0.0 & 1 & 1.3\\
Gemma-3-12B & 92.7 & 0.9 & 1 & 7.5 & 65.7 & 2.3 & 1 & 27.1 & 54.3 & 2.9 & 2 & 33.3 & 80.5 & 0.2 & 1 & 2.3\\
Llama-3.2-3B & 69.9 & 0.0 & 1 & 28.8 & 33.7 & 0.0 & 1 & 15.3 & 27.3 & 0.0 & 1 & 20.3 & 66.7 & 0.0 & 1 & 14.1\\
\bottomrule
\end{tabular}
\end{table}

\subsection{Voting against talking, and the single agent}
\label{app:scgap}
\label{app:longcot}
\label{app:largewsc}

Two secondary consequences follow from the same comparison: self-consistency weakly dominates debate (Prop.~\ref{prop:illusion}), and forming a group is not always worthwhile ($\delta\le0$).

\paragraph{Where the margin is largest, and how firmly.}
Prop.~\ref{prop:illusion} implies that the self-consistency margin is set by the per-vote ability consumed by peer reading. It is therefore largest when the generation window binds and shrinks as the window loosens, eventually vanishing once both methods reach their best size at sufficiently large $W$ (Figures~\ref{fig:selfcons}--\ref{fig:selfcons-falcon}). The large gap at the tightest window is partly a clipping effect: many $(n,W{=}1200)$ cells collapse to the same clipped generation room, whereas restricting to unclipped cells reduces the margin from $+42.0$ to $+13.7\%$. A question-clustered paired bootstrap with $2000$ resamples confirms the cell-level effect; applying the same test per window gives the SC sig.\ $W$ column in Table~\ref{tab:selfcons}.

\paragraph{Where the closed form gets the sign wrong.} CSQA is the one task where the closed form recovers the $\delta$ sign poorly, most severely on Gemma3-12B (predicted $24\%$, the lowest cell in Table~\ref{tab:selfcons}). The per-cell fit there has $k{\approx}8.6$, an unusually steep Hill curve, fit to only seven calibration points; its leave-one-out RMSE is over an order of magnitude above its in-sample RMSE, the signature of overfitting rather than of a genuine near-step response. Combined with $c_w$ estimates that are the least stable of any task (Table~\ref{tab:cwsigma}), the fitted $\phi$ crosses the break-even $\beta$ almost as soon as the window opens, so the closed form predicts the group winning at every swept size while the group in fact never does. All five models show the same direction of error on CSQA, so this is a property of the task's fit stability, not an isolated bug in one cell.

\paragraph{Is $\Delta p_{\max}$ actually positive?} Prop.~\ref{prop:illusion}'s dominance result is the special case $\Delta p_{\max}\le0$. We test directly whether peer text raises the per-agent ceiling, instead of measuring it through majority vote. On Qwen2.5-7B and Falcon3-7B across GSM8K and MMLU-Pro, we fix $g{=}2048$ and compare a single agent's round-$1$ accuracy without peer access ($n{=}1$) with the round-$1$ accuracy of two agents after each reads the other's round-$0$ answer ($n{=}2$). This isolates the effect of peer text from voting. $\widehat{\Delta p_{\max}}$ is small and sign-mixed, at $2$--$4\%$ in magnitude, with the sign determined by model rather than task: positive for Falcon3-7B and negative for Qwen2.5-7B on both tasks (Table~\ref{tab:deltapmax}). Thus, peer reading does not reliably raise the ceiling in these cells, suggesting that debate gains arise primarily from voting rather than from $\Delta p_{\max}$.

\begin{table}[t]
\centering
\caption{Per-agent round-$1$ accuracy (\%) with no peer read (\textbf{No peer}, $n{=}1$) against with
one peer read (\textbf{With peer}, $n{=}2$), fixed $g{=}2048$ for both, averaged over $3$ seeds.
$\widehat{\Delta p_{\max}}$ (\%) is the difference.}
\label{tab:deltapmax}
\begin{tabular}{llccc}
\toprule
 & & \textbf{No peer} & \textbf{With peer} & \textbf{$\widehat{\Delta p_{\max}}$} \\
\textbf{Model} & \textbf{Task} & (round 1) & (round 1) & \\
\midrule
\multirow{2}{*}{Qwen2.5-7B}
& GSM8K & 86.7 & 84.7 & -1.97 \\
& MMLU-Pro & 49.8 & 46.0 & -3.78 \\
\midrule
\multirow{2}{*}{Falcon3-7B} 
& GSM8K & 80.6 & 82.9 & +2.33 \\
& MMLU-Pro & 46.1 & 50.2 & +4.11 \\
\bottomrule
\end{tabular}

\end{table}

\paragraph{Does the margin actually vanish at large $W$?} On Qwen2.5-7B/GSM8K, extending the same comparison past the main grid to $W\in\{6\text{k},\dots,32\text{k}\}$ confirms it does: each method read at its own best group size, the margin is already under $1$pp by $W{=}6$k and stays within seed noise through $32$k (Table~\ref{tab:largewsc}), consistent with Cor.~\ref{cor:largew}'s correlation ceiling. We report one representative case since reaching the ceiling requires double-digit $n$ and is not practical in realistic settings.

\begin{table}[t]
\centering
\caption{Qwen2.5-7B/GSM8K, self-consistency against debate at large $W$, each method read at its own best group size $\hat n$ over the extended grid $n \in \{1,2,4,8,13,21,32\}$ (columns 2--5). \textbf{margin} is debate minus self-consistency ($\pm$ seed s.d.\ over $3$ seeds), and \textbf{matched $n$} is the same gap when self-consistency is instead read at debate's $\hat n$. Every margin lies
inside its own seed s.d., so the two methods are indistinguishable from $W{=}6\mathrm{k}$ onward.}
\label{tab:largewsc}
\begin{tabular}{lcccccc}
\toprule
$W$ & $\hat n_{\mathrm{deb}}$ & debate & $\hat n_{\mathrm{SC}}$ & self-cons.\ & margin & matched $n$ \\
\midrule
6k & 13 & 93.67 & 13 & 94.07 & $-0.40 \pm 1.11$ & $-0.40$ \\
9k & 13 & 93.80 & 13 & 94.07 & $-0.27 \pm 0.46$ & $-0.27$ \\
12k & 13 & 94.00 & 21 & 94.07 & $-0.07 \pm 0.81$ & $+0.00$ \\
16k & 21 & 94.33 & 13 & 93.67 & $+0.67 \pm 1.22$ & $+0.73$ \\
24k & 13 & 94.07 & 21 & 94.00 & $+0.07 \pm 1.47$ & $+0.13$ \\
32k & 21 & 93.87 & 21 & 93.67 & $+0.20 \pm 0.53$ & $+0.20$ \\
\bottomrule
\end{tabular}

\end{table}

% \begin{table}[t]
% \centering
% \caption{Every cell behind Table~\ref{tab:largewsc}: debate and self-consistency accuracy (\%) at
% every swept $(n,W)$, and their gap (self-consistency minus debate). \textbf{Bold} marks, per
% column, the group size each method would pick (top two blocks) or the largest gap (bottom block).}
% \label{tab:largewscpercell}
% \resizebox{\textwidth}{!}{\input{tables/largew_percell_table}}
% \end{table}

\paragraph{Debate against single-agent long-CoT.} 
Contrary to the equal-budget finding of~\citet{tran2026single},
on our matched generated-token budgets, we find the debate group still leads a single agent whose room is grown to $g_{\mathrm{cap}}{=}8192$ ($W$ extends to $16$k). Its chain saturates well below the cap, at about $275$ tokens on GSM8K and $533$--$540$ on MMLU-Pro. At matched generated-token budgets, the group leads on every cell (Table~\ref{tab:baseline}). Self-consistency recovers the same advantage without message exchange, making the comparison one of breadth versus depth rather than communication.

\begin{table}[t]
\centering
\caption{Single-agent long-CoT ($n{=}1$, room grown to $g_{\mathrm{cap}}{=}8192$) against the debate group ($n{\ge}2$), matched at the tokens debate spends at its own best cell. Gen.\ tokens instead gives long-CoT's own best window, where more thinking stops improving its accuracy.}
\label{tab:baseline}
\begin{tabular}{llcccc}
\toprule
Model & Task & Long-CoT $A$ & gen.\ tokens & Debate $A$ & Debate $-$ long-CoT \\
\midrule
\multirow{2}{*}{Qwen2.5-7B} & GSM8K & 88.9 & 276 & 93.7 & $\mathbf{+4.9}$ \\
 & MMLU-Pro & 49.5 & 533 & 49.8 & $\mathbf{+0.2}$ \\
\midrule
\multirow{2}{*}{Falcon3-7B} & GSM8K & 81.5 & 274 & 91.4 & $\mathbf{+9.9}$ \\
 & MMLU-Pro & 46.5 & 540 & 55.3 & $\mathbf{+8.9}$ \\
\bottomrule
\end{tabular}

\end{table}

\begin{table}[t]
\centering
\caption{Accuracy (\%) over the swept windows with $n\ge2$. Barred columns average over the grid, unbarred read both methods at debate's size $n^\ast(W)$. \textbf{$\delta\le0$ share} compares each cell's debate accuracy against the single-agent accuracy pooled over all windows; \textbf{single$\ge$group} is the stricter, per-window version of the same comparison, the observed share of the $49$ cells with $A(1,W)\ge A(n,W)$ at matching $W$; \textbf{predicted} is the fraction of the $\delta\le0$-share signs the closed form recovers. \textbf{SC sig.\ $W$} is how many of the $7$ swept windows have SC significantly ahead (paired bootstrap) at a majority of the $7$ group sizes.}
\label{tab:selfcons}
\setlength{\tabcolsep}{5pt}
\small
\resizebox{\textwidth}{!}{\begin{tabular}{@{}l l c c c cc cc c c c cc@{}}
\toprule
& & \multicolumn{2}{c}{theory} & & \multicolumn{2}{c}{averaged over the grid} & \multicolumn{2}{c}{at the oracle $n^\ast(W)$} & \multicolumn{3}{c}{group vs.\ single agent} & \multicolumn{2}{c}{SC vs.\ debate} \\
\cmidrule(lr){3-4}\cmidrule(lr){6-7}\cmidrule(lr){8-9}\cmidrule(lr){10-12}\cmidrule(lr){13-14}
 \textbf{Model} & \textbf{Task} & \textbf{$p_{\max}$} & \textbf{$\bar\delta$} & \textbf{$\bar A_{1}$} & \textbf{$\bar A_{\text{SC}}$} & \textbf{$\bar A_{\text{deb}}$} & \textbf{$A_{\text{SC}}$} & \textbf{$A_{\text{deb}}$} & \textbf{$\delta\le0$ share} & \textbf{single$\ge$group} & \textbf{predicted} & \textbf{SC$\ge$deb} & \textbf{SC sig.\ $W$} \\
\midrule
\multirow{4}{*}{Llama3.2-3B} & GSM8K & 69.0 & +0.23 & 67.7 & \textbf{80.3} & 69.1 & \textbf{82.0} & 80.7 & 16/49 & 16/49 & 80 & 49/49 & 4/7 \\
 & MATH500 & 34.4 & -0.05 & 34.1 & \textbf{42.0} & 32.9 & \textbf{42.1} & 41.5 & 21/49 & 22/49 & 86 & 36/49 & 4/7 \\
 & MMLU-Pro & 32.8 & -0.23 & 31.8 & \textbf{37.1} & 27.9 & \textbf{37.0} & 35.7 & 23/49 & 22/49 & 47 & 42/49 & 4/7 \\
 & CSQA & 67.8 & +0.11 & 65.1 & \textbf{69.9} & 60.1 & \textbf{71.7} & 70.2 & 16/49 & 15/49 & 78 & 49/49 & 3/7 \\
\midrule
\multirow{4}{*}{Qwen2.5-7B} & GSM8K & 87.4 & +0.37 & 84.7 & \textbf{92.2} & 73.5 & \textbf{92.5} & 92.3 & 15/49 & 15/49 & 96 & 48/49 & 2/7 \\
 & MATH500 & 67.6 & +0.11 & 67.2 & \textbf{71.6} & 48.7 & \textbf{70.4} & 68.3 & 23/49 & 24/49 & 86 & 46/49 & 4/7 \\
 & MMLU-Pro & 49.9 & -0.15 & 47.9 & \textbf{53.2} & 32.4 & \textbf{52.6} & 46.0 & 36/49 & 40/49 & 69 & 49/49 & 7/7 \\
 & CSQA & 80.5 & +0.23 & 80.6 & \textbf{82.2} & 65.7 & \textbf{82.4} & 82.0 & 23/49 & 21/49 & 80 & 49/49 & 3/7 \\
\midrule
\multirow{4}{*}{Falcon3-7B} & GSM8K & 81.0 & +0.28 & 83.4 & \textbf{91.4} & 75.6 & \textbf{91.6} & 90.8 & 13/49 & 13/49 & 100 & 45/49 & 3/7 \\
 & MATH500 & 61.0 & +0.06 & 60.0 & \textbf{65.9} & 45.6 & \textbf{65.3} & 62.4 & 24/49 & 28/49 & 86 & 48/49 & 7/7 \\
 & MMLU-Pro & 47.7 & -0.15 & 48.6 & \textbf{53.4} & 38.7 & \textbf{52.4} & 50.8 & 24/49 & 25/49 & 96 & 48/49 & 5/7 \\
 & CSQA & 74.9 & +0.21 & 73.8 & \textbf{76.9} & 66.5 & \textbf{77.7} & 75.6 & 21/49 & 28/49 & 73 & 49/49 & 7/7 \\
\midrule
\multirow{4}{*}{Gemma2-9B} & GSM8K & 81.3 & +0.32 & 81.1 & \textbf{87.5} & 80.9 & \textbf{88.3} & 86.6 & 9/49 & 9/49 & 94 & 45/49 & 7/7 \\
 & MATH500 & 40.2 & -0.11 & 42.2 & \textbf{48.7} & 41.2 & 48.8 & \textbf{48.8} & 14/49 & 14/49 & 98 & 32/49 & 3/7 \\
 & MMLU-Pro & 50.4 & -0.14 & 49.0 & \textbf{51.1} & 40.0 & \textbf{50.8} & 50.6 & 23/49 & 23/49 & 88 & 40/49 & 3/7 \\
 & CSQA & 77.9 & +0.26 & 73.8 & \textbf{76.2} & 71.2 & \textbf{77.2} & 76.3 & 15/49 & 17/49 & 71 & 48/49 & 2/7 \\
\midrule
\multirow{4}{*}{Gemma3-12B} & GSM8K & 92.9 & +0.45 & 92.3 & \textbf{93.9} & 78.4 & \textbf{93.9} & 93.8 & 16/49 & 16/49 & 92 & 48/49 & 2/7 \\
 & MATH500 & 75.3 & +0.15 & 71.6 & \textbf{73.8} & 52.4 & \textbf{73.0} & 70.7 & 26/49 & 33/49 & 80 & 48/49 & 4/7 \\
 & MMLU-Pro & 60.4 & -0.10 & 57.7 & \textbf{60.6} & 40.2 & \textbf{60.2} & 56.6 & 26/49 & 32/49 & 90 & 47/49 & 4/7 \\
 & CSQA & 80.5 & +0.26 & 78.2 & \textbf{78.3} & 58.6 & \textbf{78.0} & 76.4 & 49/49 & 49/49 & 24 & 49/49 & 7/7 \\
\midrule
\textbf{All} & & -- & -- & 64.5 & \textbf{69.3} & 55.0 & \textbf{69.4} & 67.8 & $\mathbf{433/980}$ & $\mathbf{462/980}$ & $\mathbf{81}$ & $\mathbf{915/980}$ ($93$) & $\mathbf{85/140}$ \\
\bottomrule
\end{tabular}
}
\end{table}

\begin{figure*}[t]
\centering
\includegraphics[width=\textwidth]{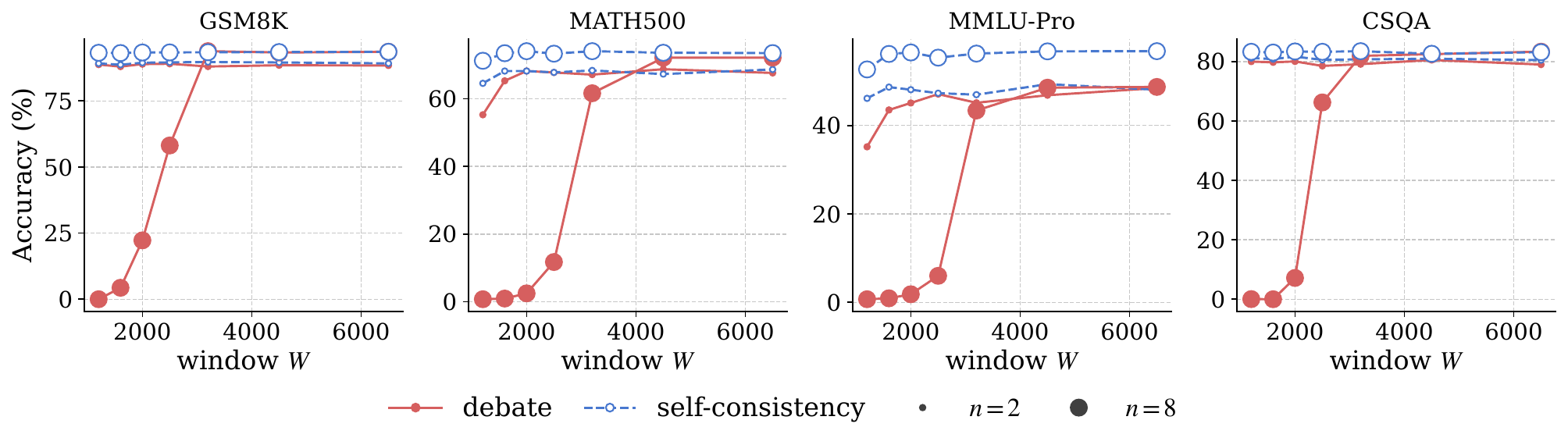}
\caption{Debate (red, solid) against self-consistency (blue, dashed) as the window grows, one panel per task (Qwen2.5-7B). Marker size is the group size $n\in\{2,8\}$. The gap closes as the window
grows on every panel (Prop.~\ref{prop:illusion}).}
\label{fig:selfcons}
\end{figure*}

\begin{figure*}[t]
\centering
\includegraphics[width=\textwidth]{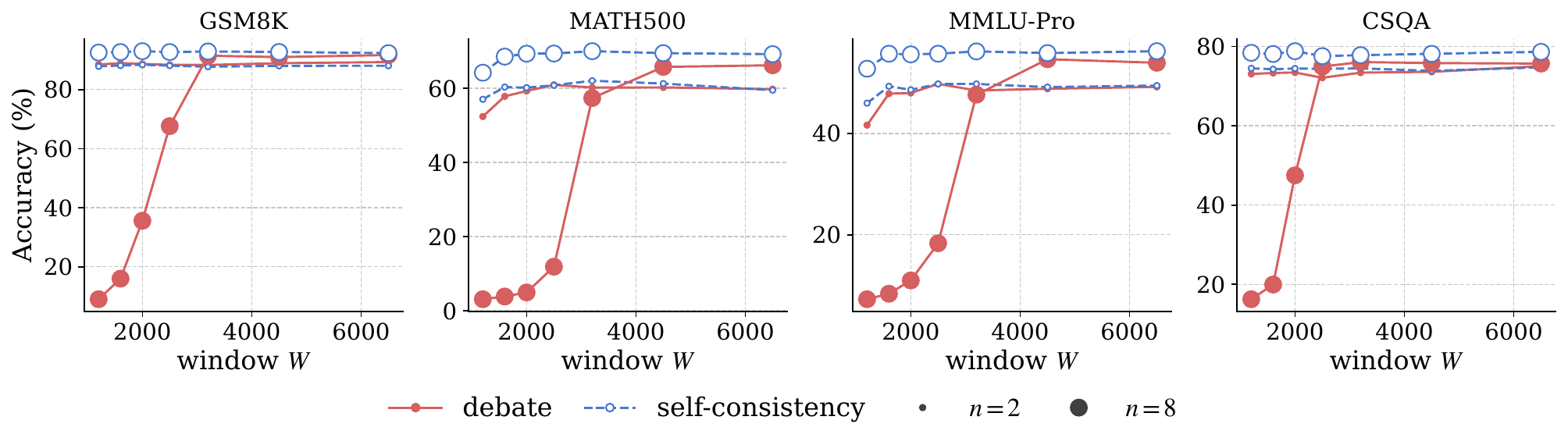}
\caption{Same as Figure~\ref{fig:selfcons}, on Falcon3-7B: the gap closes with the window on every
task there too.}
\label{fig:selfcons-falcon}
\end{figure*}

\subsection{Repricing at API rates}
\label{app:api-rate}

% This $n^\ast(B)$ is a cost-constrained optimum, not the accuracy-only $\hat n^\ast(W)$ of Figure~\ref{fig:predict}: there the oracle climbs to $n{=}8$ at loose $W$, but by only $0.7$-$1.6$pp, too small to clear the extra agents' token cost.

In practice, input tokens are often cheaper than output tokens. We define the cost of an operating point $(n,W)$ as $c_{\mathrm{in}}\cdot(\text{input tokens})+c_{\mathrm{out}}\cdot(\text{output tokens})$, normalize $c_{\mathrm{out}}=1$, and consider $r\in{1,2, 2.5,4}$. Figure~\ref{fig:cost} shows the accuracy-vs-cost Pareto frontier and the cost-optimal group size $n^\ast(B)$ (instead of $n^\ast(W)$), where $B$ is the average token per-question cost swept across operating points. The interior optimum persists under realistic pricing, with cheaper input shifting $n^\ast$ upward at matched budget. On Qwen2.5-7B/MMLU-Pro, however, the optimum remains pinned at $n=1$: all $n\ge2$ points collapse far below it due to the $g$-floor effect in \S\ref{app:scgap}, regardless of the price ratio.

The same pricing affects the coverage of a fixed calibration budget. Table~\ref{tab:rate} shows that discounted input costs slightly improve both coverage and performance for debate-calibrated methods, while our method retains the lowest regret at all budgets ($1$, $2$, and $4$M) across every price ratio.

\begin{figure*}[t]
\centering
\includegraphics[width=\textwidth]{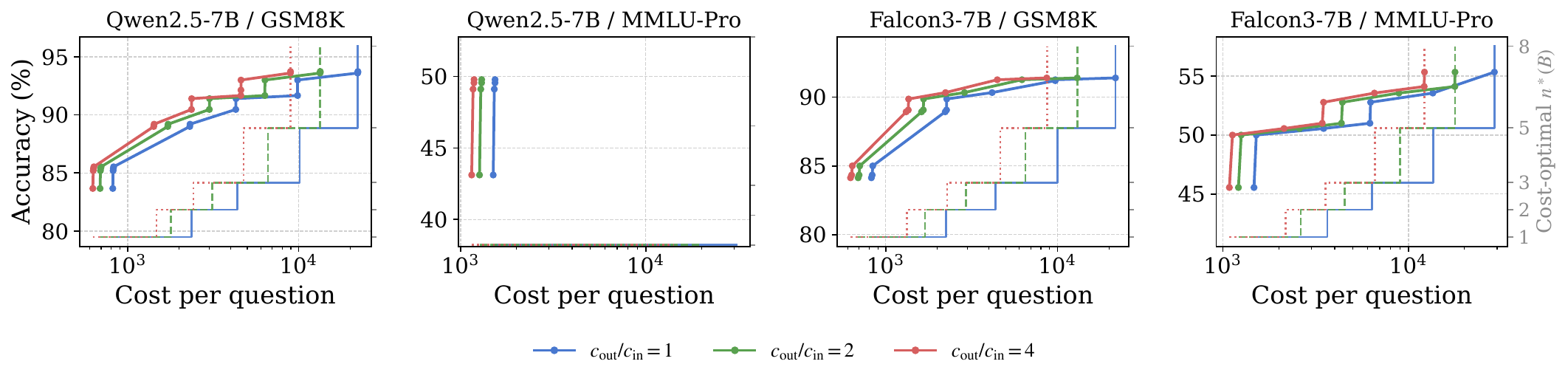}
\caption{Accuracy against cost per question. Lines are the Pareto frontier (left axis) and steps are the cost-optimal group $n^\ast(B)$ (right axis), under symmetric pricing and under cheaper input at API rates.}
\label{fig:cost}
\end{figure*}

\begin{table}[t]
\centering
\caption{Held-out regret (percentage points) at four input/output price ratios and every swept budget. \textbf{Bold} is the column minimum, \underline{underline} indicates an entry fixed at its large-budget value as the whole question is already used for calibration. First group: $c_{\mathrm{out}}/c_{\mathrm{in}}{=}1$, i.e.\ Table~\ref{tab:matchedcost}.}
\label{tab:rate}
\setlength{\tabcolsep}{5pt}
\resizebox{\textwidth}{!}{\small\begin{tabular}{@{}l r r r r r@{\hspace{10pt}} r r r r r@{\hspace{10pt}} r r r r r@{\hspace{10pt}} r r r r r@{\hspace{10pt}}@{}}
\toprule
\multirow{2}{*}[-3pt]{\textbf{Method}} & \multicolumn{5}{c}{\textbf{$c_{\mathrm{out}}/c_{\mathrm{in}}=1.0$}} & \multicolumn{5}{c}{\textbf{$c_{\mathrm{out}}/c_{\mathrm{in}}=2.0$}} & \multicolumn{5}{c}{\textbf{$c_{\mathrm{out}}/c_{\mathrm{in}}=2.5$}} & \multicolumn{5}{c}{\textbf{$c_{\mathrm{out}}/c_{\mathrm{in}}=4.0$}} \\
\cmidrule(lr){2-6}\cmidrule(lr){7-11}\cmidrule(lr){12-16}\cmidrule(lr){17-21}
\emph{Budget} (M) & 1 & 2 & 4 & 8 & 16 & 1 & 2 & 4 & 8 & 16 & 1 & 2 & 4 & 8 & 16 & 1 & 2 & 4 & 8 & 16 \\
\midrule
Empirical Argmax & 3.20 & 2.86 & 2.18 & 1.67 & 1.35 & 3.11 & 2.51 & 1.83 & 1.42 & 1.08 & 2.91 & 2.45 & 1.89 & 1.45 & 1.10 & 2.87 & 1.88 & 1.57 & 1.25 & 0.97 \\
Successive Halving & 9.08 & 5.81 & 3.04 & 1.65 & \textbf{1.12} & 6.54 & 4.56 & 1.83 & \textbf{1.26} & \textbf{1.00} & 6.31 & 4.45 & 1.73 & \textbf{1.24} & \textbf{0.92} & 5.55 & 2.01 & 1.39 & \textbf{1.10} & \textbf{0.95} \\
Quadratic Surface & 2.87 & 2.48 & 2.25 & 2.31 & 2.29 & 2.55 & 2.34 & 2.22 & 2.25 & 2.31 & 2.35 & 2.32 & 2.18 & 2.27 & 2.36 & 2.50 & 2.26 & 2.45 & 2.32 & 2.27 \\
Learned $\betaeff$ & 1.68 & 1.49 & 1.56 & \underline{1.52} & \underline{1.52} & 1.65 & 1.78 & 1.50 & \underline{1.52} & \underline{1.52} & 1.63 & 1.80 & 1.50 & \underline{1.52} & \underline{1.52} & 1.76 & 1.70 & 1.52 & \underline{1.52} & \underline{1.52} \\
Star Exponent & 1.95 & 1.96 & 1.88 & \underline{1.87} & \underline{1.87} & 2.05 & 1.93 & 1.82 & \underline{1.87} & \underline{1.87} & 1.95 & 2.08 & 1.89 & \underline{1.87} & \underline{1.87} & 1.93 & 2.03 & 1.87 & \underline{1.87} & \underline{1.87} \\
Difficulty Mixture & 2.82 & 2.84 & 2.76 & \underline{2.76} & \underline{2.76} & 2.69 & 2.83 & 2.76 & \underline{2.76} & \underline{2.76} & 2.73 & 2.82 & 2.76 & \underline{2.76} & \underline{2.76} & 3.05 & 2.78 & \underline{2.76} & \underline{2.76} & \underline{2.76} \\
\rowcolor{green!15} \textbf{TB-MAD} (Ours) & \textbf{1.56} & \textbf{1.23} & \textbf{1.23} & \underline{\textbf{1.29}} & \underline{1.29} & \textbf{1.43} & \textbf{1.50} & \textbf{1.23} & \underline{1.29} & \underline{1.29} & \textbf{1.29} & \textbf{1.42} & \textbf{1.24} & \underline{1.29} & \underline{1.29} & \textbf{1.50} & \textbf{1.40} & \textbf{1.26} & \underline{1.29} & \underline{1.29} \\
\bottomrule
\end{tabular}
}
\end{table}

% \paragraph{Predictor details}
% \label{app:predictor-detail}
% Agreement with the empirical argmax is in Figure~\ref{fig:predict}: the predicted
% $n^\ast(W)$ tracks the empirical one on every model. 

\begin{figure*}[t]
\centering
\includegraphics[width=\textwidth]{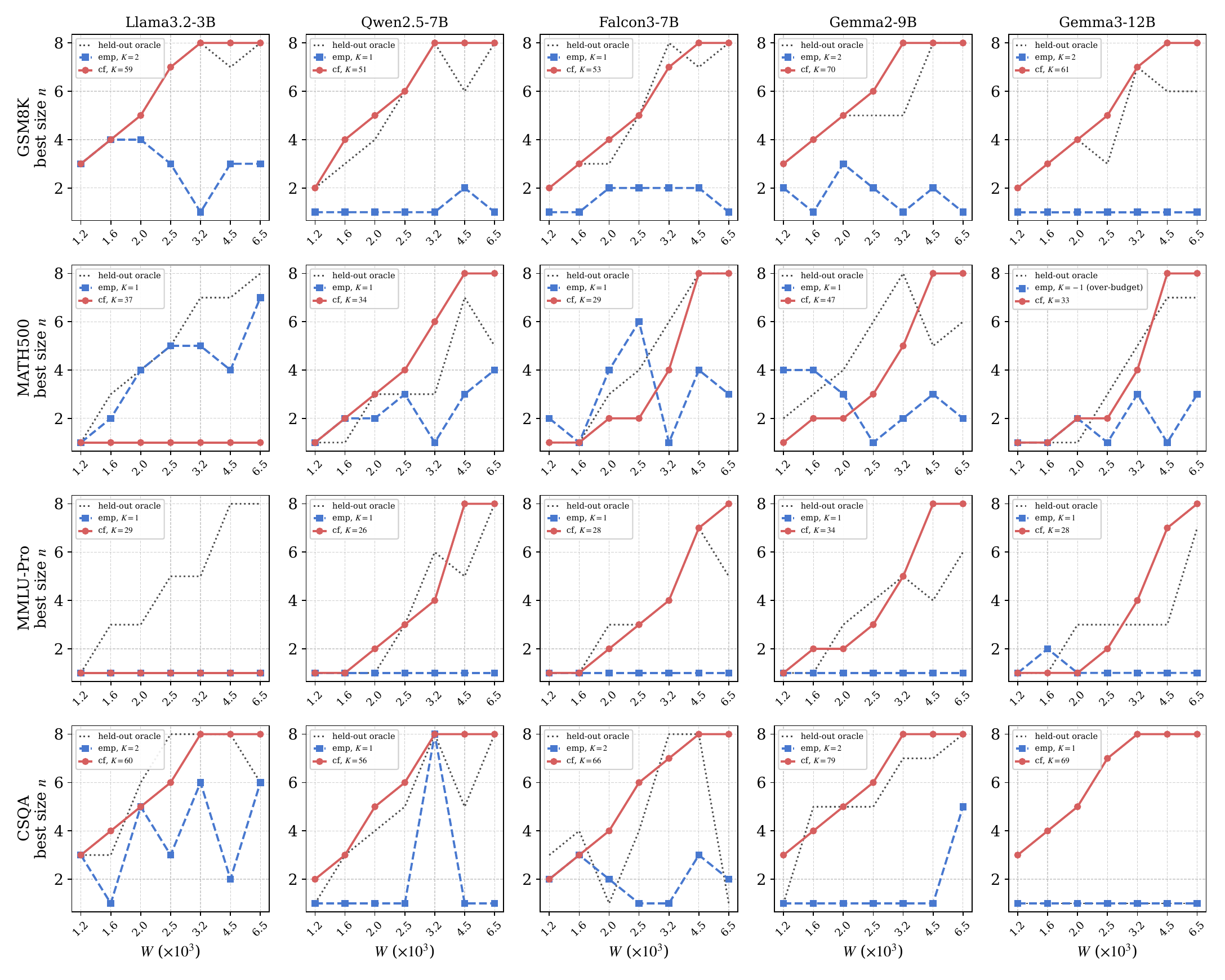}
\caption{Predicted $n^\ast(W)$ (red) against empirical $\hat n^\ast(W)$ (green) at 1M calibration tokens.}
\label{fig:predict}
\end{figure*}

% ---
\begin{table}[t]
\centering
\caption{Held-out regret (percentage points) by calibration window $W$, each method's own chosen $n$
averaged over all five models and four tasks. Complements Table~\ref{tab:matchedcost}, which
aggregates over $W$ instead. \textbf{Bold} is the lowest regret in a row.}
\label{tab:bywindow}
\begin{tabular}{rrrrrrrrr>{\columncolor{green!15}}r}
\toprule
 & \multicolumn{3}{c}{\emph{Debate-calibrated}} & \multicolumn{2}{c}{\emph{No calibration}} & \multicolumn{4}{c}{\emph{Probe-calibrated}} \\
\cmidrule(lr){2-4}\cmidrule(lr){5-6}\cmidrule(lr){7-10}
 \textbf{$W$} & \textbf{Halv.} & \textbf{Emp.} & \textbf{Surf.} & \textbf{$n_{\max}$} & \textbf{$n{=}1$} & \textbf{Mix} & \textbf{Star} & \textbf{Learn} & \textbf{TB-MAD} \\
\midrule
1200 & 1.03 & \textbf{0.82} & 2.38 & 58.82 & 1.56 & 1.65 & 2.16 & 1.91 & 1.12 \\
1600 & 1.16 & 0.98 & 2.44 & 54.23 & 2.58 & 1.34 & 2.27 & 1.97 & \textbf{0.96} \\
2000 & \textbf{1.09} & 1.17 & 2.80 & 42.72 & 3.81 & 1.89 & 2.59 & 1.34 & 1.49 \\
2500 & \textbf{0.59} & 1.04 & 3.18 & 25.57 & 4.16 & 3.20 & 2.22 & 1.30 & 1.75 \\
3200 & \textbf{0.62} & 1.12 & 3.32 & 2.93 & 5.36 & 3.93 & 1.59 & 0.72 & 0.87 \\
4500 & 0.99 & 1.56 & 2.49 & \textbf{0.84} & 5.69 & 3.76 & 1.17 & 1.16 & 1.18 \\
6500 & 1.27 & 1.04 & 0.66 & \textbf{0.58} & 5.96 & 4.69 & 1.02 & 1.02 & 1.02 \\
\bottomrule
\end{tabular}

\end{table}

% \subsection{Calibration budget: what the pool ceiling hides}
% \label{app:matchedcost}

% Tables~\ref{tab:matchedcostpercell} to~\ref{tab:matchedcostpercellhigh} give every cell at every
% swept budget. 

% \paragraph{Why successive halving is dear at a small budget.} Its schedule spends $k_1$ questions on
% \emph{every} arm in round $1$, then $2k_1$ on the survivors, then $4k_1$, summing to
% $8k_1{+}4(2k_1){+}2(4k_1)=24k_1$ against a one-shot sweep's $8k_1$. The $3\times$ ratio holds on every
% cell, and it is paid again at each of the seven swept windows, since every window reruns the halving
% from scratch. So a small budget does not buy it a coarser search, it buys a search over fewer
% questions per arm, which is why it trails us until $8$--$16$M (Table~\ref{tab:rate}, the crossover
% budget moving with the price ratio) and leads above.

% ----

\begin{figure*}[t]
\centering
\includegraphics[width=\textwidth]{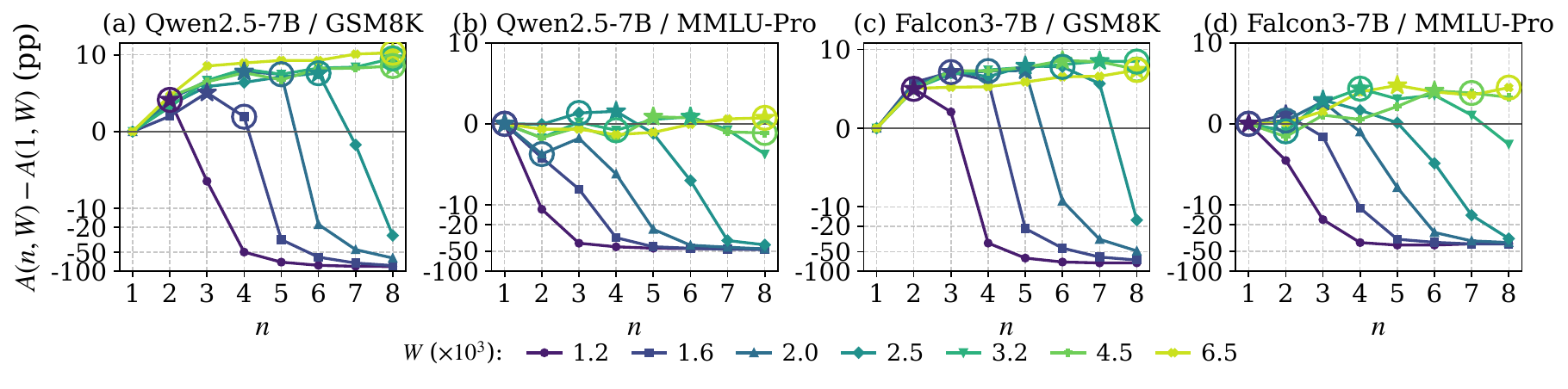}
\caption{Synergy gap $A(n,W)-A(1,W)$ against $n$ under the shared window $W$ of the standard debate protocol. Stars mark the empirical best size $\hat n^\ast(W)$, measured directly on the grid; open circles mark $n^\ast(W)$ predicted from single-agent probes alone, with no debate data.}
\label{fig:surface}
\end{figure*}

\subsection{Is the error concentration $c_w$ constant?}
\label{app:cw}
On a $K$-way decision, $c_w$ is bounded by $\tfrac1{K-1}\le c_w\le1$, while the observed values lie well within these bounds. Uniform errors would give $c_w{=}\tfrac19$ on MMLU-Pro, compared with measured values of $0.42$--$0.67$ (the stable, per-$n$ estimate; Table~\ref{tab:cwsigma}). \S\ref{sec:synergy} holds $c_w$ fixed in $n$ and $g$ so that all budget dependence in $q_{\max}$ runs through $\betaeff$; we test both directions on round-$0$ answers (Table~\ref{tab:cwsigma}).
In $n$, $\hat c_w$ is flat. Eighteen of twenty cells drift by at most $0.05$ over the swept group sizes, and the largest drift is $0.08$ on Qwen2.5-7B/GSM8K. In $g$ it rises: at $g{=}64$ little parses so the surviving wrong mass is thin and spread, at $g{=}4096$ complete answers agree on the same wrong value. A single $c_w$ therefore overstates
concentration at the bottom of the window range, where the collapse happens. Substituting the per-$n$ estimate $c_w(n)$ for the fixed cell-level value moves $n^\ast$ at $26$ of $140$ operating points and cuts mean held-out regret from $1.16$pp to
$0.80$pp, most sharply on Llama3.2-3B/MMLU-Pro ($4.40\!\to\!0.65$) and MATH500 ($2.71\!\to\!0.79$pp).

\begin{table}[h]
\centering
\caption{Range of $c_w$ over grid of $n$ and $g$. The first three columns are the same pooled estimate at $1$, $8$ and $48$ samples per question. The last two columns are held-out regret against the per-cell oracle when $n^\ast$ is chosen with the fixed cell-level $c_w$ against the per-$n$ estimate $c_w(n)$.}
\label{tab:cwsigma}
\resizebox{\textwidth}{!}{\begin{tabular}{ll ccc cc c cc}
\toprule
& & \multicolumn{3}{c}{$\hat c_w$ by samples per question} & \multicolumn{2}{c}{$\hat c_w$ range} & $n^\ast$ & \multicolumn{2}{c}{regret (pp)} \\
\cmidrule(lr){3-5}\cmidrule(lr){6-7}\cmidrule(lr){9-10}
 \textbf{Model} & \textbf{Task} & \textbf{$1$} & \textbf{$8$} & \textbf{$48$} & \textbf{over $g$} & \textbf{over $n$} & \textbf{moved} & \textbf{$c_w$ fixed} & \textbf{$c_w(n)$} \\
\midrule
\multirow{4}{*}{Llama3.2-3B} & GSM8K & 0.00 & 0.00 & 0.14 & 0.00--0.34 & 0.29--0.33 & 0/7 & 0.28 & 0.28 \\
 & MATH500 & 0.41 & 0.27 & 0.20 & 0.24--0.28 & 0.24--0.28 & 4/7 & 2.71 & 0.79 \\
 & MMLU-Pro & 0.03 & 0.02 & 0.21 & 0.02--0.44 & 0.42--0.45 & 6/7 & 4.40 & 0.65 \\
 & CSQA & 0.00 & 0.00 & 0.32 & 0.00--0.77 & 0.72--0.77 & 0/7 & 0.52 & 0.52 \\
\midrule
\multirow{4}{*}{Qwen2.5-7B} & GSM8K & 0.05 & 0.05 & 0.15 & 0.05--0.49 & 0.35--0.43 & 0/7 & 0.50 & 0.50 \\
 & MATH500 & 0.38 & 0.31 & 0.30 & 0.30--0.54 & 0.51--0.56 & 2/7 & 0.78 & 0.79 \\
 & MMLU-Pro & 0.08 & 0.08 & 0.25 & 0.07--0.63 & 0.58--0.61 & 2/7 & 1.09 & 1.50 \\
 & CSQA & 0.00 & 0.00 & 0.24 & 0.00--0.90 & 0.88--0.90 & 0/7 & 0.71 & 0.71 \\
\midrule
\multirow{4}{*}{Falcon3-7B} & GSM8K & 0.05 & 0.04 & 0.11 & 0.04--0.31 & 0.35--0.40 & 0/7 & 0.20 & 0.20 \\
 & MATH500 & 0.37 & 0.29 & 0.28 & 0.29--0.45 & 0.44--0.46 & 1/7 & 1.31 & 1.42 \\
 & MMLU-Pro & 0.10 & 0.08 & 0.24 & 0.07--0.57 & 0.54--0.58 & 1/7 & 1.12 & 0.57 \\
 & CSQA & 0.02 & 0.02 & 0.34 & 0.02--0.85 & 0.81--0.83 & 1/7 & 1.11 & 1.23 \\
\midrule
\multirow{4}{*}{Gemma2-9B} & GSM8K & 0.00 & 0.00 & 0.15 & 0.00--0.43 & 0.40--0.42 & 1/7 & 0.20 & 0.20 \\
 & MATH500 & 0.14 & 0.12 & 0.20 & 0.11--0.31 & 0.30--0.35 & 3/7 & 0.81 & 0.12 \\
 & MMLU-Pro & 0.02 & 0.02 & 0.30 & 0.02--0.65 & 0.63--0.67 & 3/7 & 0.80 & 0.41 \\
 & CSQA & 0.01 & 0.01 & 0.37 & 0.01--0.87 & 0.82--0.84 & 0/7 & 0.51 & 0.51 \\
\midrule
\multirow{4}{*}{Gemma3-12B} & GSM8K & 0.02 & 0.02 & 0.14 & 0.02--0.78 & 0.73--0.77 & 0/7 & 0.30 & 0.30 \\
 & MATH500 & 0.47 & 0.37 & 0.35 & 0.36--0.67 & 0.52--0.57 & 1/7 & 0.33 & 0.22 \\
 & MMLU-Pro & 0.06 & 0.05 & 0.23 & 0.05--0.72 & 0.58--0.64 & 1/7 & 1.70 & 1.31 \\
 & CSQA & 0.00 & 0.00 & 0.30 & 0.00--0.93 & 0.87--0.91 & 0/7 & 3.80 & 3.80 \\
\midrule
\textbf{Mean} & & 0.11 & 0.09 & 0.24 & & & 26/140 & 1.16 & 0.80 \\
\bottomrule
\end{tabular}
}
\end{table}

\subsection{Does $\sigma$ depend on $n$?}
\label{app:sigma}

\S\ref{sec:synergy} treats $\sigma$ as a single constant per cell, fitted once and reused across all $n$. However, $\sigma$ is determined by the three-valued vote score $X_j\in\{+1,0,-1\}$ once $\Pr[X_j{=}{+1}]=\betaeff$ and $\Pr[X_j{=}{-1}]=c_w(1-\betaeff)$ are fixed. Its variance is therefore $\sigma^2(n)=c_w+(1-c_w)\betaeff(n)-\delta(n)^2$, which varies with $n$ through $\betaeff(n)$. We test whether replacing the fitted constant with this derived, $n$-dependent value changes the prediction. Table~\ref{tab:sigman} reruns Algorithm~\ref{alg:predict} under both choices on the same $5$~models~$\times$~$4$~tasks~$\times$~$7$~windows $=140$ operating points as Table~\ref{tab:cwsigma}, and reports every disagreement in $\hat n^\ast$. The two choices disagree at only $7$ of $140$ points, always by one agent. Six shift downward, where the derived $\sigma(n)$ favors a smaller group, while one, Llama3.2-3B/MATH500, shifts upward. Unlike the $c_w(n)$ substitution above, which is uniformly conservative, the derived $\sigma(n)$ is not one-sided. We therefore retain the simpler fitted constant in the main text, while noting that it changes only $7/140$ predictions.

\begin{table}[h]
\centering
\caption{Where the arg max of Algorithm~\ref{alg:predict} moves when the fitted constant $\sigma$ is
replaced by the derived $\sigma(n)$, per model-task cell across the $7$ swept $W$ (columns, in
tokens). \textbf{Moved} counts the windows where $\hat n^\ast$ differs; each numbered column gives
the signed shift at that $W$ ($\hat n^\ast_{\mathrm{derived}}-\hat n^\ast_{\mathrm{const}}$), $0$
where the two choices of $\sigma$ agree.}
\label{tab:sigman}
\begin{tabular}{llcccccccc}
\toprule
\textbf{Model} & \textbf{Task} & \textbf{Moved} & \multicolumn{7}{c}{\textbf{$\hat n^\ast_{\mathrm{derived}}-\hat n^\ast_{\mathrm{const}}$ by $W$}} \\
\cmidrule(l){4-10}
 & & & \textbf{1200} & \textbf{1600} & \textbf{2000} & \textbf{2500} & \textbf{3200} & \textbf{4500} & \textbf{6500} \\
\midrule
\multirow{4}{*}{Llama3.2-3B} & GSM8K & 0/7 & 0 & 0 & 0 & 0 & 0 & 0 & 0 \\
 & MATH500 & 1/7 & 0 & 0 & 0 & +1 & 0 & 0 & 0 \\
 & MMLU-Pro & 0/7 & 0 & 0 & 0 & 0 & 0 & 0 & 0 \\
 & CSQA & 0/7 & 0 & 0 & 0 & 0 & 0 & 0 & 0 \\
\midrule
\multirow{4}{*}{Qwen2.5-7B} & GSM8K & 1/7 & 0 & -1 & 0 & 0 & 0 & 0 & 0 \\
 & MATH500 & 0/7 & 0 & 0 & 0 & 0 & 0 & 0 & 0 \\
 & MMLU-Pro & 0/7 & 0 & 0 & 0 & 0 & 0 & 0 & 0 \\
 & CSQA & 0/7 & 0 & 0 & 0 & 0 & 0 & 0 & 0 \\
\midrule
\multirow{4}{*}{Falcon3-7B} & GSM8K & 1/7 & 0 & 0 & 0 & 0 & -1 & 0 & 0 \\
 & MATH500 & 1/7 & 0 & -1 & 0 & 0 & 0 & 0 & 0 \\
 & MMLU-Pro & 0/7 & 0 & 0 & 0 & 0 & 0 & 0 & 0 \\
 & CSQA & 0/7 & 0 & 0 & 0 & 0 & 0 & 0 & 0 \\
\midrule
\multirow{4}{*}{Gemma2-9B} & GSM8K & 0/7 & 0 & 0 & 0 & 0 & 0 & 0 & 0 \\
 & MATH500 & 0/7 & 0 & 0 & 0 & 0 & 0 & 0 & 0 \\
 & MMLU-Pro & 0/7 & 0 & 0 & 0 & 0 & 0 & 0 & 0 \\
 & CSQA & 0/7 & 0 & 0 & 0 & 0 & 0 & 0 & 0 \\
\midrule
\multirow{4}{*}{Gemma3-12B} & GSM8K & 2/7 & 0 & 0 & 0 & -1 & -1 & 0 & 0 \\
 & MATH500 & 1/7 & 0 & 0 & -1 & 0 & 0 & 0 & 0 \\
 & MMLU-Pro & 0/7 & 0 & 0 & 0 & 0 & 0 & 0 & 0 \\
 & CSQA & 0/7 & 0 & 0 & 0 & 0 & 0 & 0 & 0 \\
\bottomrule
\end{tabular}

\end{table}

\subsection{Sensitivity of the prediction to $m$.} \label{app:m}
Figure~\ref{fig:ablationm} sweeps
$m\in\{0,100,320,700\}$ for the closed form against empirical argmax and the true held-out best
(oracle), all matched to a $1$M-token calibration budget and scored held out. At $m{=}0$ the
prediction is pinned at $n_{\max}$ on every window, as the theory requires with no per-agent reading
cost; $n^\ast$ then falls with $m$ at fixed $W$, with the paper's $m{=}320$ the closest of the three
swept values to the oracle. Argmax, collapse to $K{=}1$ question at this budget, stays flat and low
throughout.

\begin{figure*}[t]
\centering
\includegraphics[width=\textwidth]{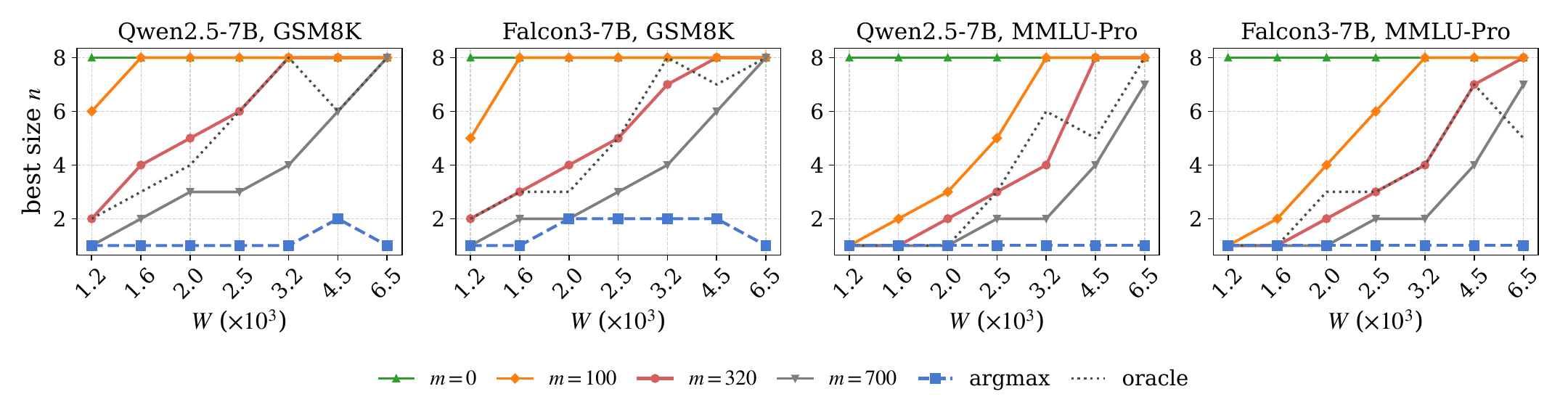}
\caption{Predicted $n^\ast(W)$ from the closed form against empirical argmax and the true held-out best (oracle), all matched to a
$1$M-token calibration budget.}
\label{fig:ablationm}
\end{figure*}

\subsection{Generation-floor recovery.}\label{app:generation-floor}

\begin{figure}[t]
\centering
\includegraphics[width=\linewidth]{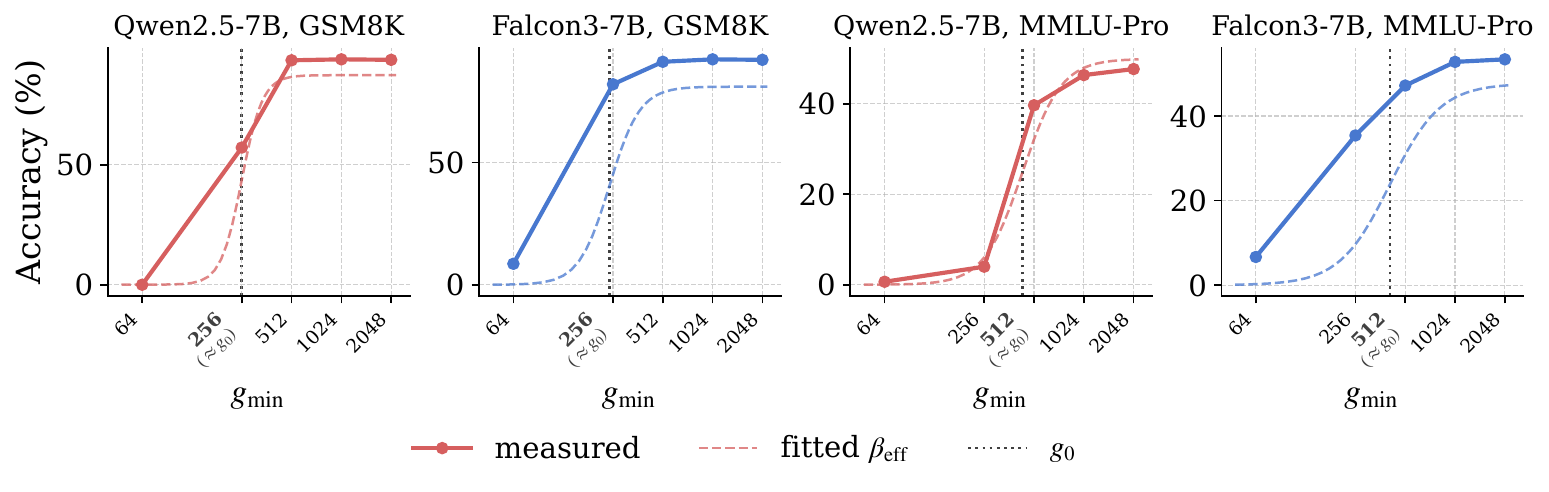}
\caption{Accuracy against the generation floor $g_{\min}$ at the worst cell ($n{=}8,W{=}1200$), Qwen2.5-7B
and Falcon3-7B on GSM8K and MMLU-Pro.}
\label{fig:gmin}
\end{figure}
\label{app:a2}

Figure~\ref{fig:gmin} sweeps the generation floor $g_{\min}$ at the worst cell ($n{=}8$, $W{=}1200$), well below the fitted scale $g_0$. Accuracy recovers as $g_{\min}$ crosses $g_0$, tracking $\betaeff(g_{\min})$ with no free parameter. GSM8K ($g_0{\approx}244$--$256$) is near ceiling by $g_{\min}{=}512$, while the harder MMLU-Pro ($g_0{\approx}414$--$437$) continues climbing past $1024$. This suggests the collapse reflects a task-dependent reasoning requirement captured by $g_0$, rather than a fragile artifact of the $g_{\min}{=}64$ floor reported in \S\ref{sec:main}. The floor gives the sharpest manifestation of the effect, not its only one.

\begin{table}[h]
\centering
\caption{Effect of question position: question first (\texttt{qf}) versus buried after the peer messages (\texttt{qb}). $\Delta{=}A_{\mathrm{qf}}-A_{\mathrm{qb}}$ (\%).}
\label{tab:qburied}
\begin{tabular}{llcccccc}
\toprule
 & & \multicolumn{3}{c}{grid mean} & \multicolumn{3}{c}{collapse cell $(n{=}8,W{=}1200)$}\\
\cmidrule(lr){3-5}\cmidrule(lr){6-8}
Model & Task & $A_{\mathrm{qf}}$ & $A_{\mathrm{qb}}$ & $\Delta$ & $A_{\mathrm{qf}}$ & $A_{\mathrm{qb}}$ & $\Delta$\\
\midrule
\multirow{2}{*}{Falcon3-7B} & GSM8K & 79.2 & 77.9 & +1.24 & 9.0 & 0.2 & +8.80\\
 & MMLU-Pro & 42.1 & 41.6 & +0.48 & 6.9 & 1.3 & +5.67\\
\midrule
\multirow{2}{*}{Qwen2.5-7B} & GSM8K & 78.0 & 78.2 & -0.24 & 0.0 & 0.1 & -0.07 \\
 & MMLU-Pro & 37.4 & 36.9 & +0.50 & 0.7 & 0.5 & +0.20\\
\midrule
Average & & 59.2 & 58.7 & +0.49 & 4.2 & 0.5 & +3.65\\
\bottomrule
\end{tabular}

\end{table}

\section{Limitations and Future work}
\label{app:scope}
\label{sec:scope}
\label{app:limitations}

\textbf{First}, the aggregation model is a tractable approximation rather than an exact account of group behavior. It relies on equicorrelation, a fixed error concentration, and an unweighted, average-based decision threshold, all of which are empirically evaluated rather than assumed to hold exactly; consensus-weighted aggregation~\citep{nguyen2026beyond} is not considered. Moreover, although we investigate heterogeneous model sets, allocation remains uniform across agent counts; adaptive allocation~\citep{nguyen2026uab} is an interesting direction for future work.

\textbf{Second}, the analysis focuses on self-generated computation, including debate and self-refinement, and does not cover settings where additional context is retrieved or supplied externally, such as retrieval-augmented generation.

\textbf{Third}, all five models have at most 12B parameters and short fitted solution lengths $g_0$, ranging from $137$ to $478$ tokens. Models with thousand-token reasoning chains would have much larger $g_0$, shifting the collapse regime to larger windows and potentially beyond deployed context lengths. We expect the qualitative shape to transfer, but not the specific operating points.

\textbf{Fourth}, additional revision rounds increase visibility and may shift the trade-off if their revision benefit does not offset the added context cost; Appendix~\ref{app:selfrefine} explores $R$ up to 3 for self-refinement.

\textbf{Fifth}, the whole account assumes one window pays for reading and generation. Under per-agent contexts, reading does not consume generation room, m = 0 in our accounting, and the interior optimum disappears ($\S$~\ref{sec:ablations}).
\begin{algorithm}[t]
\caption{Predicting group accuracy and the best size at a deployment window}
\label{alg:predict}
\begin{algorithmic}[1]
\Require $(p_{\max},g_0,\kappa)$ from the single-agent sweep, Eq.~\ref{eq:beff-fit}
\Require $\hat\rho=\mathrm{corr}(X_i,X_j)$, $X_i=\ind\left[\text{agent }i\text{ answers correctly at round }0\right]$
\Require tie penalty $c_w$
\Require deployment window $W$, protocol $(q,m,c)$, cap $n_{\max}$
\For{$n \in [1,n_{\max}]$ \textbf{continuous relaxation}}
\State $g \gets \mathrm{clip}\bigl(W-q-(n-1)m-c,;g_{\min},;g_{\mathrm{cap}}\bigr)$
\State $\betaeff \gets p_{\max},\phi(g;g_0,\kappa)$
\State $\delta \gets (1+c_w)\betaeff-c_w$
\State $n_{\mathrm{eff}} \gets n/(1+\hat\rho(n-1))$
\State $S(n) \gets \delta\sqrt{n_{\mathrm{eff}}}$
\EndFor
\State $\hat n^\ast \gets \left\lfloor\arg\max_{n\in[1,n_{\max}]}S(n)\right\rfloor$
\State \Return $\hat n^\ast$
\end{algorithmic}
\end{algorithm}

\begin{algorithm}[t]
\caption{Grid measurement: empirical argmax and the quadratic surface}
\label{alg:baselinegrid}
\begin{algorithmic}[1]
\Require calibration budget $B$ in output-token equivalents; sizes $\mathcal{N}$; windows $\mathcal{W}$; question pool $\mathcal{Q}$
\Require per-cell price $c(n,W)=\mathrm{in}(n,W)+\mathrm{out}(n,W)$, \Comment{symmetric pricing}
\State $u \gets \sum_{(n,W)\in\mathcal{N}\times\mathcal{W}} c(n,W)$;\quad
       $K \gets \min\{|\mathcal{Q}|,\lfloor B/u\rfloor\}$ \Comment{requires $B\ge u$}
\State draw $\mathcal{D}\subseteq\mathcal{Q}$ uniformly without replacement, $|\mathcal{D}|=K$
\For{$(n,W)\in\mathcal{N}\times\mathcal{W}$}
  \State run the $n$-agent debate on $\mathcal{D}$ at $W$ and record $\hat A(n,W)$
\EndFor
\State $\hat n_{\mathrm{emp}}(W) \gets \min\arg\max_{n\in\mathcal{N}}\hat A(n,W)$ for each $W\in\mathcal{W}$
  \Comment{\textbf{empirical argmax}}
\State $\psi(n,W)\gets(1,n,W,n^2,W^2,nW)^{\!\top}$ \Comment{requires $|\mathcal{N}|\,|\mathcal{W}|\ge 6$}
\State $\hat\theta \gets \arg\min_{\theta\in\mathbb{R}^6}
       \sum_{(n,W)\in\mathcal{N}\times\mathcal{W}}\bigl(\hat A(n,W)-\theta^{\!\top}\psi(n,W)\bigr)^2$
\State $\hat n_{\mathrm{quad}}(W) \gets \min\arg\max_{n\in\mathcal{N}}\hat\theta^{\!\top}\psi(n,W)$ for each $W\in\mathcal{W}$
  \Comment{\textbf{quadratic surface}}
\State \Return $\hat n_{\mathrm{emp}}(W),\ \hat n_{\mathrm{quad}}(W)$
\end{algorithmic}
\end{algorithm}

\begin{algorithm}[t]
\caption{Successive halving}
\label{alg:baselinehalving}
\begin{algorithmic}[1]
\Require calibration budget $B$; sizes $\mathcal{N}$; windows $\mathcal{W}$; question pool $\mathcal{Q}$; price $c(n,W)$
\State $R\gets\lceil\log_2|\mathcal{N}|\rceil$;\quad $B_W\gets B/|\mathcal{W}|$
\For{$W\in\mathcal{W}$}
  \State $\mathcal{S}\gets\mathcal{N}$;\quad $b\gets B_W$;\quad $\mathcal{P}\gets\mathcal{Q}$;\quad $\mathrm{ran}\gets\mathrm{false}$
  \State $k\gets\max\!\left\{1,\left\lfloor B_W\big/\bigl(R\textstyle\sum_{n\in\mathcal{S}}c(n,W)\bigr)\right\rfloor\right\}$
  \While{$|\mathcal{S}|>1$}
    \State $C\gets k\sum_{n\in\mathcal{S}}c(n,W)$
    \If{$C>b$ \textbf{or} $k>|\mathcal{P}|$} \State \textbf{break} \EndIf
    \State draw $k$ questions from $\mathcal{P}$ without replacement; remove them from $\mathcal{P}$
    \State run every $n\in\mathcal{S}$ at $W$ on these questions;
           update $\hat A(n,W)$ as the running mean over all questions seen by $n$
    \State $b\gets b-C$;\quad $\mathrm{ran}\gets\mathrm{true}$
    \State $\mathcal{S}\gets$ the $\lceil|\mathcal{S}|/2\rceil$ arms with highest $\hat A(\cdot,W)$, ties to smaller $n$;\quad $k\gets 2k$
  \EndWhile
  \State $\hat n(W)\gets\min\arg\max_{n\in\mathcal{S}}\hat A(n,W)$ if $\mathrm{ran}$, else $\min\mathcal{N}$
    \Comment{budget too small for one round}
\EndFor
\State \Return $\hat n(W)$
\end{algorithmic}
\end{algorithm}

\begin{algorithm}[t]
\caption{Closed-form aggregation baselines: learned $\betaeff$ and star exponent}
\label{alg:baselinesa}
\begin{algorithmic}[1]
\Require calibration budget $B$; probe caps $\mathcal{G}$ ($|\mathcal{G}|{=}7$, $64$--$4096$ log-spaced); question pool $\mathcal{Q}$
\Require price $c(g)=r\cdot\mathrm{in}(g)+\mathrm{out}(g)$, $r{=}1$; measured $\hat\rho$; tie penalty $c_w$
\Require protocol $(q,m,c)$; sizes $\mathcal{N}$; windows $\mathcal{W}$; room limits $g_{\min},g_{\mathrm{cap}}$
\State $u\gets\sum_{g\in\mathcal{G}}c(g)$;\quad
       $K\gets\min\{|\mathcal{Q}|,\lfloor B/u\rfloor\}$ \Comment{requires $B\ge u$}
\State draw $\mathcal{D}\subseteq\mathcal{Q}$ uniformly without replacement, $|\mathcal{D}|=K$
\State run one agent on $\mathcal{D}$ at every $g\in\mathcal{G}$; record accuracy $\hat A(g)$ \Comment{no debate is run}
\State $\hat\beta_{\mathrm{gbt}}\gets\mathrm{GBT}\bigl(\{(g,\hat A(g)):g\in\mathcal{G}\}\bigr)$
  \Comment{$100$ trees, depth $3$, learning rate $0.1$}
\State $(\hat p_{\max},\hat g_0,\hat\kappa)\gets\arg\min_{p_{\max},g_0,\kappa}
       \sum_{g\in\mathcal{G}}\bigl(\hat A(g)-p_{\max}\bigl(1-(1+g/g_0)^{-\kappa}\bigr)\bigr)^2$
  \Comment{\textsc{pow} of Table~\ref{tab:forms}}
\State $\hat\beta_{\mathrm{pow}}(g)\gets\hat p_{\max}\bigl(1-(1+g/\hat g_0)^{-\hat\kappa}\bigr)$
\For{$W\in\mathcal{W}$, $n\in\mathcal{N}$}
  \State $g\gets\mathrm{clip}\bigl(W-q-(n-1)m-c,g_{\min},g_{\mathrm{cap}}\bigr)$;\quad
         $n_{\mathrm{eff}}\gets n/\bigl(1+\hat\rho(n-1)\bigr)$
  \For{$x\in\{\mathrm{gbt},\mathrm{pow}\}$}
    \State $\delta_x\gets(1+c_w)\hat\beta_x(g)-c_w$;\quad
           $S_x(n,W)\gets\delta_x\sqrt{n_{\mathrm{eff}}}$
  \EndFor
\EndFor
\State \Return $\hat n_x(W)=\min\arg\max_{n\in\mathcal{N}}S_x(n,W)$ for each $x\in\{\mathrm{gbt},\mathrm{pow}\}$
\end{algorithmic}
\end{algorithm}

\begin{algorithm}[t]
\caption{Difficulty mixture, with our window accounting}
\label{alg:baselinemix}
\begin{algorithmic}[1]
\Require calibration budget $B$; probe caps $\mathcal{G}$; question pool $\mathcal{Q}$; price $c(g)$
\Require samples per cap $S{=}8$; simulation draws $L{=}400$; clip level $\epsilon{=}10^{-6}$
\Require sizes $\mathcal{N}$; windows $\mathcal{W}$; protocol $(q,m,c)$; room limits $g_{\min},g_{\mathrm{cap}}$; measured $\hat\rho$; tie penalty $c_w$
\State $u\gets S\sum_{g\in\mathcal{G}}c(g)$;\quad
       $K\gets\min\{|\mathcal{Q}|,\lfloor B/u\rfloor\}$ \Comment{requires $B\ge u$}
\State draw $\mathcal{D}\subseteq\mathcal{Q}$ uniformly without replacement, $|\mathcal{D}|=K$
\State run one agent $S$ times on each $i\in\mathcal{D}$ at every $g\in\mathcal{G}$;
       $\hat p_i(g)\gets$ fraction correct
\For{$W\in\mathcal{W}$, $n\in\mathcal{N}$}
  \State $g\gets\mathrm{clip}\bigl(W-q-(n-1)m-c,\;g_{\min},\;g_{\mathrm{cap}}\bigr)$
  \For{$i\in\mathcal{D}$}
    \State $p_i\gets\mathrm{clip}\bigl(\text{linear interpolation of }\hat p_i\text{ in }\log g,\;\epsilon,\;1-\epsilon\bigr)$
    \For{$\ell=1,\ldots,L$}
      \State $\zeta,\varepsilon_1,\ldots,\varepsilon_n\overset{\mathrm{iid}}{\sim}\mathrm{Normal}(0,1)$;\quad
             $z_j\gets\sqrt{\hat\rho}\,\zeta+\sqrt{1-\hat\rho}\,\varepsilon_j$
      \State $h\gets\#\{j:z_j<\Phi^{-1}(p_i)\}$;\quad
             $r\sim\mathrm{Bin}(n-h,\,c_w)$
      \State $\mathrm{win}_{i\ell}\gets\ind[h>r]+\tfrac12\,\ind[h=r]$
    \EndFor
  \EndFor
  \State $\hat A(n,W)\gets\frac{1}{KL}\sum_{i\in\mathcal{D}}\sum_{\ell=1}^{L}\mathrm{win}_{i\ell}$
\EndFor
\State \Return $\hat n(W)=\min\arg\max_{n\in\mathcal{N}}\hat A(n,W)$
\end{algorithmic}
\end{algorithm}

\begin{algorithm}[t]
\caption{Predicting the best size and model subset on a heterogeneous set}
\label{alg:predict-hetero}
\begin{algorithmic}[1]
\Require per-model $(p_{\max,i},g_{0,i},\kappa_i)$ from the single-agent sweep, Eq.~\ref{eq:beff-fit}, $i=1,\dots,K$
\Require within/between round-$0$ correlations $\hat\rho_w,\hat\rho_b$; tie penalty $c_w$
\Require deployment window $W$, protocol $(q,m,c)$, cap $n_{\max}$
\For{$n \in \{K,2K,\dots,n_{\max}\}$}
  \State $g \gets \mathrm{clip}\bigl(W-q-(n-1)m-c,\;g_{\min},\;g_{\mathrm{cap}}\bigr)$
  \State $\delta_i \gets (1+c_w)\,p_{\max,i}\,\phi_i(g;g_{0,i},\kappa_i)-c_w$ \quad\textbf{for } $i=1,\dots,K$
  \State $\mathcal{S} \gets \bigl\{\,\{1,\dots,K\}\,\bigr\}$
  \If{$\min_i \delta_i \le 0 < \max_i \delta_i$} \Comment{edges mixed in sign: a subset may win}
    \State $\mathcal{S} \gets \mathcal{S}\ \cup\ \bigl\{\,\{i\} : \delta_i > 0\,\bigr\}$
  \EndIf
  \For{$S \in \mathcal{S}$}
    \State $\bar\delta_S \gets |S|^{-1}\textstyle\sum_{i\in S}\delta_i$
    \State $n_{\mathrm{eff}} \gets n\big/\bigl(1+\hat\rho_w(n/|S|-1)+\hat\rho_b(n-n/|S|)\bigr)$
    \State $\mathcal{V}(n,S) \gets \operatorname{sign}(\bar\delta_S) n_{\mathrm{eff}}\,\bar\delta_S^{\,2}$
  \EndFor
\EndFor
\State $(\hat n^\ast,\hat S) \gets \arg\max_{(n,S)}\mathcal{V}(n,S)$
\State \Return $\hat n^\ast,\ \hat S$
\end{algorithmic}
\end{algorithm}

\begin{table*}[p]
\centering
\caption{Every cell behind Table~\ref{tab:matchedcost} at the $1.0$M budget: held-out regret
(percentage points) against the per-cell \textbf{Oracle} in the last column. Columns are SH
successive halving, EA empirical argmax, QS quadratic surface, LA learned ability, SE star exponent,
DM difficulty mixture, TB ours.}
\label{tab:matchedcostpercell}
\setlength{\tabcolsep}{4pt}
\resizebox{\textwidth}{!}{%
\begin{tabular}{@{}llccccccccc@{\hspace{9pt}}c@{}}
\toprule
\textbf{Model} & \textbf{Dataset} & QS & EA & SH & $n{=}n_{\max}$ & $n{=}1$ & DM & SE & LA & \textbf{TB} & \textbf{Oracle} \\
\midrule
\multirow{4}{*}{Llama3.2-3B} & GSM8K & 5.20 & 7.04 & 5.84 & 19.22 & 13.08 & 7.81 & 5.70 & 2.26 & 0.41 & 79.34 \\
 & MATH500 & 2.76 & 5.57 & 10.13 & 15.38 & 8.23 & 8.23 & 3.31 & 3.93 & 3.48 & 43.43 \\
 & MMLU-Pro & 4.52 & 3.33 & 6.21 & 14.74 & 4.42 & 4.43 & 3.71 & 3.60 & 4.42 & 37.22 \\
 & CSQA & 4.35 & 3.59 & 6.32 & 22.09 & 5.80 & 2.77 & 2.56 & 1.08 & 0.87 & 69.87 \\
\midrule
\multirow{4}{*}{Qwen2.5-7B} & GSM8K & 4.14 & 6.57 & 11.20 & 39.78 & 7.83 & 0.91 & 1.81 & 1.40 & 2.30 & 92.10 \\
 & MATH500 & 3.97 & 3.45 & 20.37 & 38.52 & 2.89 & 2.89 & 1.28 & 1.38 & 1.12 & 70.46 \\
 & MMLU-Pro & 0.69 & 1.20 & 11.52 & 27.29 & 0.78 & 0.78 & 0.69 & 1.08 & 0.96 & 49.43 \\
 & CSQA & 1.03 & 1.48 & 5.13 & 36.25 & 1.37 & 0.70 & 0.85 & 0.71 & 0.78 & 82.07 \\
\midrule
\multirow{4}{*}{Falcon3-7B} & GSM8K & 2.35 & 6.34 & 8.64 & 33.32 & 7.89 & 0.77 & 1.58 & 0.95 & 0.29 & 90.88 \\
 & MATH500 & 3.62 & 3.12 & 15.99 & 32.95 & 3.35 & 3.71 & 1.19 & 1.31 & 1.03 & 63.92 \\
 & MMLU-Pro & 2.07 & 2.55 & 8.18 & 22.72 & 2.76 & 2.76 & 1.98 & 1.33 & 1.25 & 51.66 \\
 & CSQA & 1.67 & 1.45 & 3.90 & 20.14 & 1.48 & 1.22 & 1.13 & 0.71 & 1.15 & 74.71 \\
\midrule
\multirow{4}{*}{Gemma2-9B} & GSM8K & 2.17 & 4.17 & 1.45 & 15.53 & 5.87 & 3.15 & 0.93 & 0.26 & 0.24 & 86.54 \\
 & MATH500 & 5.99 & 4.32 & 11.77 & 16.00 & 6.55 & 6.55 & 3.83 & 2.30 & 2.70 & 49.92 \\
 & MMLU-Pro & 1.74 & 1.63 & 8.02 & 22.31 & 2.03 & 2.03 & 1.91 & 1.69 & 1.71 & 52.76 \\
 & CSQA & 2.03 & 2.63 & 1.30 & 13.11 & 2.83 & 2.43 & 1.57 & 0.83 & 0.96 & 75.88 \\
\midrule
\multirow{4}{*}{Gemma3-12B} & GSM8K & 1.16 & 1.61 & 8.83 & 36.79 & 1.67 & 1.54 & 0.54 & 2.08 & 0.98 & 93.84 \\
 & MATH500 & 0.76 & 1.28 & 23.51 & 37.77 & 0.99 & 0.95 & 0.43 & 0.50 & 0.60 & 73.67 \\
 & MMLU-Pro & 2.60 & 1.91 & 10.60 & 33.35 & 1.42 & 1.42 & 0.97 & 2.80 & 2.69 & 59.49 \\
 & CSQA & 4.57 & 0.84 & 2.61 & 38.90 & 0.05 & 1.29 & 2.94 & 3.41 & 3.35 & 77.68 \\
\bottomrule
\end{tabular}}

\end{table*}

% ----

\begin{table*}[p]
\centering
\caption{Every cell behind Table~\ref{tab:matchedcost} at the $2.0$M budget. Columns and entries as in
Table~\ref{tab:matchedcostpercell}.}
\label{tab:matchedcostpercell2}
\setlength{\tabcolsep}{4pt}
\resizebox{\textwidth}{!}{%
\begin{tabular}{@{}llccccccccc@{\hspace{9pt}}c@{}}
\toprule
\textbf{Model} & \textbf{Dataset} & QS & EA & SH & $n{=}n_{\max}$ & $n{=}1$ & DM & SE & LA & \textbf{TB} & \textbf{Oracle} \\
\midrule
\multirow{4}{*}{Llama3.2-3B} & GSM8K & 5.04 & 4.26 & 3.40 & 19.22 & 13.08 & 5.15 & 5.70 & 1.80 & 0.41 & 79.34 \\
 & MATH500 & 4.83 & 6.22 & 6.31 & 15.38 & 8.23 & 8.23 & 3.54 & 2.69 & 1.80 & 43.43 \\
 & MMLU-Pro & 2.58 & 3.40 & 9.52 & 14.74 & 4.42 & 4.42 & 4.42 & 4.05 & 4.42 & 37.22 \\
 & CSQA & 4.54 & 4.04 & 2.84 & 22.09 & 5.80 & 3.67 & 2.50 & 1.03 & 0.87 & 69.87 \\
\midrule
\multirow{4}{*}{Qwen2.5-7B} & GSM8K & 4.81 & 4.03 & 1.44 & 39.78 & 7.83 & 1.06 & 1.80 & 1.39 & 0.46 & 92.10 \\
 & MATH500 & 1.95 & 2.34 & 6.05 & 38.52 & 2.89 & 2.89 & 1.28 & 1.31 & 1.10 & 70.46 \\
 & MMLU-Pro & 0.75 & 1.55 & 10.47 & 27.29 & 0.78 & 0.83 & 0.78 & 0.88 & 0.96 & 49.43 \\
 & CSQA & 1.45 & 1.58 & 1.76 & 36.25 & 1.37 & 0.92 & 0.85 & 0.60 & 0.79 & 82.07 \\
\midrule
\multirow{4}{*}{Falcon3-7B} & GSM8K & 5.34 & 4.32 & 1.53 & 33.32 & 7.89 & 1.02 & 1.58 & 0.69 & 0.29 & 90.88 \\
 & MATH500 & 2.66 & 2.43 & 10.63 & 32.95 & 3.35 & 3.35 & 1.19 & 1.37 & 1.04 & 63.92 \\
 & MMLU-Pro & 2.28 & 2.40 & 12.13 & 22.72 & 2.76 & 2.87 & 1.46 & 1.52 & 0.89 & 51.66 \\
 & CSQA & 1.52 & 1.92 & 1.26 & 20.14 & 1.48 & 1.41 & 1.19 & 0.62 & 1.15 & 74.71 \\
\midrule
\multirow{4}{*}{Gemma2-9B} & GSM8K & 2.05 & 3.98 & 2.89 & 15.53 & 5.87 & 4.03 & 0.93 & 0.19 & 0.22 & 86.54 \\
 & MATH500 & 2.98 & 2.90 & 4.05 & 16.00 & 6.55 & 6.55 & 3.89 & 2.30 & 2.66 & 49.92 \\
 & MMLU-Pro & 1.63 & 2.54 & 5.43 & 22.31 & 2.03 & 2.03 & 1.51 & 0.87 & 0.77 & 52.76 \\
 & CSQA & 1.61 & 2.77 & 1.16 & 13.11 & 2.83 & 2.83 & 1.57 & 1.05 & 0.94 & 75.88 \\
\midrule
\multirow{4}{*}{Gemma3-12B} & GSM8K & 1.20 & 1.67 & 0.42 & 36.79 & 1.67 & 1.52 & 0.55 & 1.47 & 0.37 & 93.84 \\
 & MATH500 & 0.40 & 2.57 & 11.89 & 37.77 & 0.99 & 0.99 & 0.43 & 0.50 & 0.60 & 73.67 \\
 & MMLU-Pro & 1.34 & 1.76 & 14.34 & 33.35 & 1.42 & 1.44 & 1.07 & 2.06 & 1.57 & 59.49 \\
 & CSQA & 0.69 & 0.47 & 8.70 & 38.90 & 0.05 & 1.57 & 2.94 & 3.35 & 3.35 & 77.68 \\
\bottomrule
\end{tabular}}

\end{table*}

% ----

\begin{table*}[p]
\centering
\caption{Every cell behind Table~\ref{tab:rate} at the $4.0$M budget. Columns and entries as in
Table~\ref{tab:matchedcostpercell}.}
\label{tab:matchedcostpercell4}
\setlength{\tabcolsep}{4pt}
\resizebox{\textwidth}{!}{%
\begin{tabular}{@{}llccccccccc@{\hspace{9pt}}c@{}}
\toprule
\textbf{Model} & \textbf{Dataset} & QS & EA & SH & $n{=}n_{\max}$ & $n{=}1$ & DM & SE & LA & \textbf{TB} & \textbf{Oracle} \\
\midrule
\multirow{4}{*}{Llama3.2-3B} & GSM8K & 3.22 & 4.34 & 2.16 & 19.22 & 13.08 & 4.30 & 5.70 & 1.59 & 0.41 & 79.34 \\
 & MATH500 & 3.93 & 1.85 & 3.50 & 15.38 & 8.23 & 8.23 & 3.20 & 3.04 & 1.92 & 43.43 \\
 & MMLU-Pro & 2.60 & 2.17 & 2.15 & 14.74 & 4.42 & 4.42 & 4.42 & 4.32 & 4.42 & 37.22 \\
 & CSQA & 3.51 & 3.36 & 1.76 & 22.09 & 5.80 & 2.77 & 2.49 & 0.96 & 0.87 & 69.87 \\
\midrule
\multirow{4}{*}{Qwen2.5-7B} & GSM8K & 4.57 & 2.69 & 2.43 & 39.78 & 7.83 & 0.99 & 1.80 & 1.98 & 0.43 & 92.10 \\
 & MATH500 & 2.10 & 2.51 & 1.77 & 38.52 & 2.89 & 2.89 & 1.28 & 1.51 & 1.09 & 70.46 \\
 & MMLU-Pro & 1.65 & 2.16 & 7.11 & 27.29 & 0.78 & 0.78 & 0.71 & 1.32 & 0.98 & 49.43 \\
 & CSQA & 1.25 & 1.96 & 0.92 & 36.25 & 1.37 & 0.80 & 0.85 & 0.57 & 0.76 & 82.07 \\
\midrule
\multirow{4}{*}{Falcon3-7B} & GSM8K & 5.42 & 2.89 & 1.44 & 33.32 & 7.89 & 1.19 & 1.58 & 0.91 & 0.29 & 90.88 \\
 & MATH500 & 1.85 & 1.75 & 6.34 & 32.95 & 3.35 & 3.35 & 1.19 & 1.31 & 0.94 & 63.92 \\
 & MMLU-Pro & 1.46 & 3.97 & 8.85 & 22.72 & 2.76 & 2.76 & 1.50 & 1.46 & 0.92 & 51.66 \\
 & CSQA & 1.96 & 1.45 & 0.92 & 20.14 & 1.48 & 1.30 & 1.13 & 0.48 & 1.22 & 74.71 \\
\midrule
\multirow{4}{*}{Gemma2-9B} & GSM8K & 1.32 & 1.68 & 0.95 & 15.53 & 5.87 & 4.66 & 0.83 & 0.21 & 0.20 & 86.54 \\
 & MATH500 & 3.11 & 1.65 & 2.27 & 16.00 & 6.55 & 6.55 & 3.06 & 2.02 & 2.53 & 49.92 \\
 & MMLU-Pro & 1.52 & 2.08 & 1.78 & 22.31 & 2.03 & 2.03 & 1.43 & 1.20 & 0.88 & 52.76 \\
 & CSQA & 2.15 & 2.19 & 1.25 & 13.11 & 2.83 & 2.83 & 1.57 & 0.89 & 0.95 & 75.88 \\
\midrule
\multirow{4}{*}{Gemma3-12B} & GSM8K & 1.12 & 0.84 & 0.96 & 36.79 & 1.67 & 1.48 & 0.57 & 1.58 & 0.37 & 93.84 \\
 & MATH500 & 0.44 & 1.14 & 8.33 & 37.77 & 0.99 & 0.99 & 0.43 & 0.50 & 0.60 & 73.67 \\
 & MMLU-Pro & 1.08 & 1.41 & 4.19 & 33.35 & 1.42 & 1.42 & 1.00 & 2.06 & 1.51 & 59.49 \\
 & CSQA & 0.69 & 1.51 & 1.61 & 38.90 & 0.05 & 1.40 & 2.94 & 3.35 & 3.35 & 77.68 \\
\bottomrule
\end{tabular}}

\end{table*}

% ----

\begin{table*}[p]
\centering
\caption{Every cell behind Table~\ref{tab:rate} at the $8.0$M budget. Columns and entries as in
Table~\ref{tab:matchedcostpercell}.}
\label{tab:matchedcostpercell8}
\setlength{\tabcolsep}{4pt}
\resizebox{\textwidth}{!}{%
\begin{tabular}{@{}llccccccccc@{\hspace{9pt}}c@{}}
\toprule
\textbf{Model} & \textbf{Dataset} & QS & EA & SH & $n{=}n_{\max}$ & $n{=}1$ & DM & SE & LA & \textbf{TB} & \textbf{Oracle} \\
\midrule
\multirow{4}{*}{Llama3.2-3B} & GSM8K & 5.11 & 2.08 & 2.49 & 19.22 & 13.08 & 4.30 & 5.70 & 1.59 & 0.41 & 79.34 \\
 & MATH500 & 3.77 & 2.74 & 2.76 & 15.38 & 8.23 & 8.23 & 3.20 & 3.04 & 2.62 & 43.43 \\
 & MMLU-Pro & 2.74 & 1.99 & 3.65 & 14.74 & 4.42 & 4.42 & 4.42 & 4.32 & 4.42 & 37.22 \\
 & CSQA & 4.80 & 2.80 & 2.42 & 22.09 & 5.80 & 2.77 & 2.49 & 0.96 & 0.87 & 69.87 \\
\midrule
\multirow{4}{*}{Qwen2.5-7B} & GSM8K & 4.92 & 1.80 & 1.09 & 39.78 & 7.83 & 0.99 & 1.80 & 1.98 & 0.43 & 92.10 \\
 & MATH500 & 2.08 & 1.76 & 1.03 & 38.52 & 2.89 & 2.89 & 1.28 & 1.51 & 1.10 & 70.46 \\
 & MMLU-Pro & 0.54 & 1.85 & 2.06 & 27.29 & 0.78 & 0.78 & 0.71 & 1.21 & 1.01 & 49.43 \\
 & CSQA & 0.99 & 0.99 & 1.33 & 36.25 & 1.37 & 0.80 & 0.85 & 0.57 & 0.76 & 82.07 \\
\midrule
\multirow{4}{*}{Falcon3-7B} & GSM8K & 4.79 & 1.88 & 1.26 & 33.32 & 7.89 & 1.19 & 1.58 & 0.91 & 0.29 & 90.88 \\
 & MATH500 & 2.15 & 1.70 & 1.23 & 32.95 & 3.35 & 3.35 & 1.19 & 1.31 & 0.94 & 63.92 \\
 & MMLU-Pro & 2.35 & 1.93 & 1.32 & 22.72 & 2.76 & 2.76 & 1.46 & 1.44 & 1.03 & 51.66 \\
 & CSQA & 1.70 & 1.04 & 1.39 & 20.14 & 1.48 & 1.30 & 1.13 & 0.48 & 1.22 & 74.71 \\
\midrule
\multirow{4}{*}{Gemma2-9B} & GSM8K & 1.07 & 1.28 & 1.29 & 15.53 & 5.87 & 4.66 & 0.83 & 0.21 & 0.20 & 86.54 \\
 & MATH500 & 1.70 & 1.22 & 1.39 & 16.00 & 6.55 & 6.55 & 2.72 & 1.57 & 2.66 & 49.92 \\
 & MMLU-Pro & 1.58 & 1.12 & 1.23 & 22.31 & 2.03 & 2.03 & 1.47 & 0.94 & 0.86 & 52.76 \\
 & CSQA & 2.09 & 1.85 & 1.32 & 13.11 & 2.83 & 2.83 & 1.57 & 0.89 & 0.95 & 75.88 \\
\midrule
\multirow{4}{*}{Gemma3-12B} & GSM8K & 1.22 & 0.83 & 0.23 & 36.79 & 1.67 & 1.48 & 0.57 & 1.58 & 0.37 & 93.84 \\
 & MATH500 & 0.47 & 1.61 & 1.05 & 37.77 & 0.99 & 0.99 & 0.43 & 0.50 & 0.60 & 73.67 \\
 & MMLU-Pro & 1.32 & 1.45 & 2.25 & 33.35 & 1.42 & 1.42 & 1.07 & 2.06 & 1.81 & 59.49 \\
 & CSQA & 0.78 & 1.48 & 2.15 & 38.90 & 0.05 & 1.40 & 2.94 & 3.35 & 3.35 & 77.68 \\
\bottomrule
\end{tabular}}

\end{table*}

% ----

\begin{table*}[p]
\centering
\caption{Every cell behind Table~\ref{tab:rate} at the $16.0$M budget. Columns and entries as in
Table~\ref{tab:matchedcostpercell}.}
\label{tab:matchedcostpercellhigh}
\setlength{\tabcolsep}{4pt}
\resizebox{\textwidth}{!}{%
\begin{tabular}{@{}llccccccccc@{\hspace{9pt}}c@{}}
\toprule
\textbf{Model} & \textbf{Dataset} & QS & EA & SH & $n{=}n_{\max}$ & $n{=}1$ & DM & SE & LA & \textbf{TB} & \textbf{Oracle} \\
\midrule
\multirow{4}{*}{Llama3.2-3B} & GSM8K & 3.68 & 1.13 & 0.75 & 19.22 & 13.08 & 4.30 & 5.70 & 1.59 & 0.41 & 79.34 \\
 & MATH500 & 4.34 & 1.84 & 1.41 & 15.38 & 8.23 & 8.23 & 3.20 & 3.04 & 2.62 & 43.43 \\
 & MMLU-Pro & 2.35 & 1.08 & 1.67 & 14.74 & 4.42 & 4.42 & 4.42 & 4.32 & 4.42 & 37.22 \\
 & CSQA & 4.21 & 1.86 & 1.53 & 22.09 & 5.80 & 2.77 & 2.49 & 0.96 & 0.87 & 69.87 \\
\midrule
\multirow{4}{*}{Qwen2.5-7B} & GSM8K & 4.74 & 1.37 & 0.94 & 39.78 & 7.83 & 0.99 & 1.80 & 1.98 & 0.43 & 92.10 \\
 & MATH500 & 2.17 & 1.39 & 0.99 & 38.52 & 2.89 & 2.89 & 1.28 & 1.51 & 1.10 & 70.46 \\
 & MMLU-Pro & 0.54 & 2.29 & 1.19 & 27.29 & 0.78 & 0.78 & 0.71 & 1.21 & 1.01 & 49.43 \\
 & CSQA & 1.01 & 0.94 & 0.98 & 36.25 & 1.37 & 0.80 & 0.85 & 0.57 & 0.76 & 82.07 \\
\midrule
\multirow{4}{*}{Falcon3-7B} & GSM8K & 5.13 & 1.43 & 0.69 & 33.32 & 7.89 & 1.19 & 1.58 & 0.91 & 0.29 & 90.88 \\
 & MATH500 & 2.28 & 1.46 & 1.04 & 32.95 & 3.35 & 3.35 & 1.19 & 1.31 & 0.94 & 63.92 \\
 & MMLU-Pro & 2.05 & 0.69 & 1.33 & 22.72 & 2.76 & 2.76 & 1.46 & 1.44 & 1.03 & 51.66 \\
 & CSQA & 1.71 & 1.20 & 1.00 & 20.14 & 1.48 & 1.30 & 1.13 & 0.48 & 1.22 & 74.71 \\
\midrule
\multirow{4}{*}{Gemma2-9B} & GSM8K & 1.90 & 1.35 & 0.60 & 15.53 & 5.87 & 4.66 & 0.83 & 0.21 & 0.20 & 86.54 \\
 & MATH500 & 1.97 & 1.10 & 0.82 & 16.00 & 6.55 & 6.55 & 2.72 & 1.57 & 2.66 & 49.92 \\
 & MMLU-Pro & 1.71 & 1.63 & 1.67 & 22.31 & 2.03 & 2.03 & 1.47 & 0.94 & 0.86 & 52.76 \\
 & CSQA & 2.17 & 1.51 & 1.23 & 13.11 & 2.83 & 2.83 & 1.57 & 0.89 & 0.95 & 75.88 \\
\midrule
\multirow{4}{*}{Gemma3-12B} & GSM8K & 1.31 & 1.20 & 0.59 & 36.79 & 1.67 & 1.48 & 0.57 & 1.58 & 0.37 & 93.84 \\
 & MATH500 & 0.47 & 1.01 & 0.66 & 37.77 & 0.99 & 0.99 & 0.43 & 0.50 & 0.60 & 73.67 \\
 & MMLU-Pro & 1.46 & 1.35 & 2.23 & 33.35 & 1.42 & 1.42 & 1.07 & 2.06 & 1.81 & 59.49 \\
 & CSQA & 0.67 & 1.22 & 1.11 & 38.90 & 0.05 & 1.40 & 2.94 & 3.35 & 3.35 & 77.68 \\
\bottomrule
\end{tabular}}

\end{table*}
% ----

\begin{table}[t]
\centering
\caption{Every cell behind Table~\ref{tab:heterobk} at the $1.0$M budget: held-out regret
(percentage points) against the widened per-cell oracle (set-best or any single member).
Abbreviations as in Table~\ref{tab:matchedcostpercell}, $n{=}3$--$8$ the fixed-size rules.}
\label{tab:heterobkpercell}
\resizebox{\textwidth}{!}{%
\begin{tabular}{@{}llcccccccccccc@{}}
\toprule
\textbf{Model set} & \textbf{Dataset} & SH & QS & EA & $n{=}8$ & $n{=}6$ & $n{=}3$ & $n{=}4$ & TB w/o gate & DM & SE & LA & \textbf{TB} \\
\midrule
\multirow{4}{*}{Qwen2.5-7B+Gemma3-12B+Llama3.2-3B} & GSM8K & 0.47 & 0.59 & 0.80 & -- & 15.18 & 0.93 & -- & 1.50 & 0.17 & 0.42 & 0.13 & 0.13 \\
 & MATH500 & 3.65 & 2.52 & 3.16 & -- & 19.65 & 3.54 & -- & 1.98 & 3.54 & 2.10 & 1.27 & 1.03 \\
 & MMLU-Pro & 6.61 & 6.31 & 3.48 & -- & 11.78 & 5.59 & -- & 3.39 & 5.59 & 3.42 & 3.02 & 3.39 \\
 & CSQA & 1.15 & 0.91 & 1.13 & -- & 11.99 & 1.47 & -- & 8.19 & 1.10 & 0.62 & 0.50 & 0.50 \\
\midrule
\multirow{4}{*}{Qwen2.5-7B+Falcon3-7B+Gemma2-9B} & GSM8K & 2.02 & 0.79 & 0.62 & -- & 13.06 & 1.17 & -- & 1.63 & 0.00 & 0.40 & 0.02 & 0.02 \\
 & MATH500 & 5.10 & 3.98 & 4.22 & -- & 20.73 & 4.14 & -- & 1.91 & 4.13 & 3.51 & 3.38 & 2.00 \\
 & MMLU-Pro & 0.94 & 0.99 & 0.89 & -- & 13.85 & 1.51 & -- & 2.98 & 1.51 & 0.88 & 0.80 & 2.23 \\
 & CSQA & 4.30 & 3.87 & 1.78 & -- & 8.56 & 1.56 & -- & 2.79 & 1.81 & 1.60 & 1.75 & 1.75 \\
\midrule
\multirow{4}{*}{Qwen2.5-7B+Falcon3-7B+Gemma2-9B+Gemma3-12B} & GSM8K & 1.82 & 1.28 & 1.30 & 22.17 & -- & -- & 1.30 & 1.60 & 1.21 & 1.15 & 1.09 & 1.09 \\
 & MATH500 & 5.40 & 4.23 & 5.05 & 26.67 & -- & -- & 4.40 & 5.07 & 4.23 & 3.80 & 3.80 & 4.31 \\
 & MMLU-Pro & 9.44 & 0.98 & 2.24 & 18.13 & -- & -- & 1.48 & 6.75 & 1.38 & 1.37 & 0.86 & 4.77 \\
 & CSQA & 1.25 & 0.76 & 0.98 & 15.90 & -- & -- & 1.00 & 0.70 & 0.70 & 0.76 & 0.70 & 0.70 \\
\midrule
\multirow{4}{*}{Qwen2.5-7B+Falcon3-7B+Gemma3-12B+Llama3.2-3B} & GSM8K & 0.93 & 0.38 & 0.52 & 22.04 & -- & -- & 0.52 & 0.77 & 0.10 & 0.31 & 0.12 & 0.12 \\
 & MATH500 & 6.65 & 3.67 & 3.98 & 24.86 & -- & -- & 3.98 & 4.06 & 3.27 & 2.98 & 2.98 & 2.77 \\
 & MMLU-Pro & 5.66 & 1.65 & 3.11 & 13.23 & -- & -- & 1.96 & 4.10 & 1.74 & 1.39 & 1.30 & 4.10 \\
 & CSQA & 1.40 & 0.88 & 1.02 & 17.72 & -- & -- & 1.04 & 3.12 & 0.66 & 0.70 & 0.62 & 0.62 \\
\bottomrule
\end{tabular}}

\end{table}

\begin{table}[t]
\centering
\caption{Every cell behind Table~\ref{tab:heterobk} at the $2.0$M budget. Columns and entries as in
Table~\ref{tab:heterobkpercell}.}
\label{tab:heterobkpercell2}
\resizebox{\textwidth}{!}{%
\begin{tabular}{@{}llcccccccccccc@{}}
\toprule
\textbf{Model set} & \textbf{Dataset} & SH & QS & EA & $n{=}8$ & $n{=}6$ & $n{=}3$ & $n{=}4$ & TB w/o gate & DM & SE & LA & \textbf{TB} \\
\midrule
\multirow{4}{*}{Qwen2.5-7B+Gemma3-12B+Llama3.2-3B} & GSM8K & 1.03 & 0.59 & 0.64 & -- & 15.18 & 0.93 & -- & 1.01 & 0.24 & 0.42 & 0.13 & 0.17 \\
 & MATH500 & 2.17 & 2.65 & 2.11 & -- & 19.65 & 3.54 & -- & 1.48 & 3.54 & 2.10 & 1.27 & 0.83 \\
 & MMLU-Pro & 6.03 & 3.42 & 3.66 & -- & 11.78 & 5.59 & -- & 1.38 & 4.91 & 3.42 & 2.99 & 1.38 \\
 & CSQA & 0.39 & 0.96 & 0.58 & -- & 11.99 & 1.47 & -- & 3.79 & 0.74 & 0.62 & 0.50 & 1.60 \\
\midrule
\multirow{4}{*}{Qwen2.5-7B+Falcon3-7B+Gemma2-9B} & GSM8K & 0.56 & 0.83 & 0.87 & -- & 13.06 & 1.17 & -- & 1.89 & 0.22 & 0.40 & 0.02 & 0.02 \\
 & MATH500 & 3.98 & 3.94 & 3.52 & -- & 20.73 & 4.14 & -- & 0.83 & 4.14 & 3.51 & 3.43 & 2.43 \\
 & MMLU-Pro & 0.23 & 1.08 & 0.69 & -- & 13.85 & 1.51 & -- & 5.20 & 1.25 & 0.88 & 0.18 & 1.97 \\
 & CSQA & 1.97 & 1.40 & 2.10 & -- & 8.56 & 1.56 & -- & 1.96 & 1.54 & 1.60 & 1.75 & 1.75 \\
\midrule
\multirow{4}{*}{Qwen2.5-7B+Falcon3-7B+Gemma2-9B+Gemma3-12B} & GSM8K & 1.58 & 1.28 & 1.20 & 22.17 & -- & -- & 1.30 & 1.10 & 1.09 & 1.15 & 1.09 & 1.02 \\
 & MATH500 & 5.18 & 4.23 & 4.12 & 26.67 & -- & -- & 4.40 & 3.26 & 3.94 & 3.80 & 3.80 & 3.95 \\
 & MMLU-Pro & 3.30 & 1.38 & 1.81 & 18.13 & -- & -- & 1.48 & 3.81 & 1.27 & 1.10 & 0.57 & 2.52 \\
 & CSQA & 1.03 & 0.75 & 0.97 & 15.90 & -- & -- & 1.00 & 3.34 & 0.68 & 0.76 & 0.70 & 0.70 \\
\midrule
\multirow{4}{*}{Qwen2.5-7B+Falcon3-7B+Gemma3-12B+Llama3.2-3B} & GSM8K & 0.53 & 0.38 & 0.44 & 22.04 & -- & -- & 0.52 & 0.68 & 0.20 & 0.31 & 0.12 & 0.12 \\
 & MATH500 & 5.68 & 3.98 & 3.98 & 24.86 & -- & -- & 3.98 & 2.24 & 3.25 & 2.98 & 2.98 & 2.24 \\
 & MMLU-Pro & 3.71 & 1.60 & 1.60 & 13.23 & -- & -- & 1.96 & 2.36 & 1.96 & 1.34 & 1.09 & 2.36 \\
 & CSQA & 1.43 & 0.88 & 0.91 & 17.72 & -- & -- & 1.04 & 3.13 & 0.74 & 0.70 & 0.62 & 0.62 \\
\bottomrule
\end{tabular}}

\end{table}

\begin{table}[t]
\centering
\caption{Every cell behind Table~\ref{tab:heterobaselines} at the $1.0$M budget: held-out regret
(percentage points) against the widened per-cell oracle (set-best or any single member).
Abbreviations as in Table~\ref{tab:matchedcostpercell}.}
\label{tab:heterobaselinespercell}
\resizebox{\textwidth}{!}{\small\resizebox{\textwidth}{!}{%
\begin{tabular}{@{}llcccccccccc@{}}
\toprule
\textbf{Model pair} & \textbf{Dataset} & SH & QS & EA & $n{=}n_{\max}$ & $n{=}2$ & DM & SE & LA & TB w/o gate & \textbf{TB} \\
\midrule
\multirow{4}{*}{Qwen2.5-7B+Gemma3-12B} & GSM8K & 6.34 & 1.78 & 2.34 & 32.07 & 2.34 & 0.26 & 0.80 & 0.61 & 0.47 & 0.51 \\
 & MATH500 & 12.75 & 2.80 & 3.46 & 36.14 & 3.99 & 3.17 & 0.84 & 0.80 & 2.28 & 0.84 \\
 & MMLU-Pro & 10.71 & 3.47 & 3.26 & 28.72 & 5.65 & 4.71 & 4.36 & 3.85 & 1.83 & 1.75 \\
 & CSQA & 6.05 & 2.37 & 2.67 & 26.34 & 2.67 & 1.64 & 1.69 & 0.97 & 0.77 & 0.77 \\
\midrule
\multirow{4}{*}{Qwen2.5-7B+Llama3.2-3B} & GSM8K & 7.42 & 4.82 & 5.04 & 31.07 & 6.82 & 1.34 & 2.17 & 1.09 & 0.37 & 1.03 \\
 & MATH500 & 15.36 & 8.67 & 9.42 & 33.24 & 7.74 & 6.67 & 4.96 & 4.59 & 0.27 & 1.77 \\
 & MMLU-Pro & 12.87 & 8.52 & 4.97 & 20.67 & 3.70 & 3.30 & 2.57 & 2.19 & 2.35 & 2.35 \\
 & CSQA & 7.86 & 6.04 & 8.94 & 26.28 & 9.31 & 3.77 & 4.93 & 2.83 & 0.29 & 2.34 \\
\midrule
\multirow{4}{*}{Falcon3-7B+Gemma2-9B} & GSM8K & 6.86 & 2.89 & 3.49 & 16.41 & 3.59 & 0.65 & 1.62 & 0.56 & 1.33 & 0.35 \\
 & MATH500 & 9.44 & 4.84 & 6.05 & 25.46 & 6.90 & 7.25 & 3.84 & 3.67 & 0.27 & 3.87 \\
 & MMLU-Pro & 11.61 & 1.50 & 1.65 & 19.72 & 1.99 & 1.84 & 1.41 & 1.53 & 4.48 & 3.35 \\
 & CSQA & 3.95 & 1.48 & 2.02 & 9.63 & 2.32 & 0.65 & 0.94 & 0.89 & 2.14 & 0.63 \\
\midrule
\multirow{4}{*}{Gemma3-12B+Llama3.2-3B} & GSM8K & 5.94 & 1.66 & 2.28 & 22.43 & 2.28 & 1.59 & 1.15 & 0.67 & 0.71 & 0.67 \\
 & MATH500 & 10.49 & 2.37 & 2.83 & 32.39 & 3.15 & 2.99 & 1.59 & 2.66 & 0.28 & 0.28 \\
 & MMLU-Pro & 12.53 & 1.85 & 2.60 & 19.88 & 2.98 & 2.41 & 1.71 & 1.19 & 4.63 & 3.73 \\
 & CSQA & 0.89 & 4.66 & 4.56 & 16.69 & 5.21 & 1.77 & 1.94 & 0.69 & 4.01 & 0.69 \\
\bottomrule
\end{tabular}}
}
\end{table}

\begin{table}[t]
\centering
\caption{Every cell behind Table~\ref{tab:heterobaselines} at the $2.0$M budget. Columns and entries
as in Table~\ref{tab:heterobaselinespercell}.}
\label{tab:heterobaselinespercell2}
\resizebox{\textwidth}{!}{\small\resizebox{\textwidth}{!}{%
\begin{tabular}{@{}llcccccccccc@{}}
\toprule
\textbf{Model pair} & \textbf{Dataset} & SH & QS & EA & $n{=}n_{\max}$ & $n{=}2$ & DM & SE & LA & TB w/o gate & \textbf{TB} \\
\midrule
\multirow{4}{*}{Qwen2.5-7B+Gemma3-12B} & GSM8K & 2.88 & 1.58 & 2.05 & 32.07 & 2.34 & 0.54 & 0.89 & 0.56 & 1.19 & 0.51 \\
 & MATH500 & 5.57 & 2.87 & 3.23 & 36.14 & 3.99 & 3.99 & 0.84 & 0.95 & 2.89 & 0.61 \\
 & MMLU-Pro & 8.15 & 11.57 & 4.28 & 28.72 & 5.65 & 4.91 & 1.82 & 2.15 & 1.06 & 1.33 \\
 & CSQA & 2.40 & 2.40 & 2.26 & 26.34 & 2.67 & 1.78 & 1.69 & 0.86 & 0.77 & 0.77 \\
\midrule
\multirow{4}{*}{Qwen2.5-7B+Llama3.2-3B} & GSM8K & 3.14 & 3.77 & 3.80 & 31.07 & 6.82 & 1.61 & 1.61 & 1.03 & 0.60 & 1.16 \\
 & MATH500 & 12.61 & 6.69 & 6.83 & 33.24 & 7.74 & 7.51 & 4.96 & 4.51 & 0.22 & 1.42 \\
 & MMLU-Pro & 7.40 & 10.67 & 5.57 & 20.67 & 3.70 & 3.45 & 1.70 & 1.51 & 3.37 & 3.37 \\
 & CSQA & 4.34 & 5.37 & 4.53 & 26.28 & 9.31 & 4.80 & 4.93 & 2.34 & 0.29 & 2.34 \\
\midrule
\multirow{4}{*}{Falcon3-7B+Gemma2-9B} & GSM8K & 1.81 & 2.87 & 3.33 & 16.41 & 3.59 & 1.06 & 1.62 & 0.44 & 0.41 & 0.41 \\
 & MATH500 & 12.32 & 5.05 & 5.54 & 25.46 & 6.90 & 6.90 & 3.84 & 3.56 & 0.57 & 3.73 \\
 & MMLU-Pro & 3.85 & 1.66 & 3.99 & 19.72 & 1.99 & 1.44 & 1.42 & 1.84 & 3.14 & 3.02 \\
 & CSQA & 1.29 & 1.59 & 1.70 & 9.63 & 2.32 & 1.27 & 0.94 & 0.49 & 2.93 & 0.63 \\
\midrule
\multirow{4}{*}{Gemma3-12B+Llama3.2-3B} & GSM8K & 3.47 & 1.77 & 2.18 & 22.43 & 2.28 & 1.12 & 1.12 & 0.67 & 0.77 & 0.67 \\
 & MATH500 & 8.69 & 2.45 & 3.15 & 32.39 & 3.15 & 3.15 & 1.59 & 2.59 & 0.28 & 0.87 \\
 & MMLU-Pro & 9.42 & 10.49 & 5.33 & 19.88 & 2.98 & 2.67 & 1.36 & 1.22 & 1.17 & 1.17 \\
 & CSQA & 0.43 & 1.51 & 2.08 & 16.69 & 5.21 & 1.43 & 1.94 & 0.77 & 3.23 & 0.69 \\
\bottomrule
\end{tabular}}
}
\end{table}

\end{document}